\documentclass[twoside,11pt]{article}

\usepackage{jmlr2e}

\usepackage{graphicx}
\usepackage{subfigure}
\usepackage{enumitem}

\usepackage{amsmath,amssymb}
\usepackage{bbm}

\usepackage{natbib}

\usepackage{color}

\newcommand{\Real}[1]{ { {\mathbb R}^{#1} } }

\newcommand{\field}[1]{\mathbb{#1}}

\newcommand{\R}{\field{R}}

\newcommand{\E}{\field{E}}

\newcommand{\One}[1]{ \mathbbm{1}_{#1} }

\renewcommand{\epsilon}{\varepsilon}
\newcommand{\eps}{\epsilon}

\newcommand{\bigO}[1]{{ \scalebox{1.3}{$O$} \left( {#1} \right) }}

\newcommand{\calX}{\mathcal{X}}

\newcommand{\calH}{\mathcal{H}}
\newcommand{\calP}{\mathcal{P}}
\newcommand{\calR}{\mathcal{R}}

\newcommand{\proj}{\mathrm{proj}}

\newcommand{\scZ}{\mathcal{Z}}

\newcommand{\bS}{{\boldsymbol{S}}}
\newcommand{\bz}{{\boldsymbol{z}}}
\newcommand{\bx}{{\boldsymbol{x}}}
\newcommand{\by}{{\boldsymbol{y}}}
\newcommand{\bk}{{\boldsymbol{k}}}
\newcommand{\bphi}{{\boldsymbol{\phi}}}
\newcommand{\btheta}{{\boldsymbol{\theta}}}

\renewcommand{\Pr}{{\mathbb{P}}}
\newcommand{\loss}{\ell}
\newcommand{\algo}{\mathcal{A}}
\newcommand{\risk}{\mathsf{R}}
\newcommand{\ms}{\mathsf{ms}}
\newcommand{\co}{\mathsf{c}}
\newcommand{\dd}{\mathrm{d}}
\newcommand{\ee}{\mathrm{e}}
\newcommand{\qed}{\hfill $\star$}
\newcommand{\fmap}{\varphi}

\newcommand{\rotgeq}{\mathbin{\rotatebox[origin=c]{90}{$\leq$}}}

\newcommand{\SVC}{\scriptscriptstyle{\mathrm{SVM}}}
\newcommand{\SVR}{\scriptscriptstyle{\mathrm{SVR}}}
\newcommand{\GEM}{\scriptscriptstyle{\mathrm{GEM}}}

\newtheorem{prop}[theorem]{Proposition}
\newtheorem{exa}[theorem]{Example}

\newtheorem{assumption}[theorem]{Assumption}
\newtheorem{property}[theorem]{Property}

\firstpageno{1}

\usepackage{lastpage}
\jmlrheading{24}{2023}{1-\pageref{LastPage}}{5/22; Revised 9/23}{10/23}{22-0605}{Marco C. Campi and Simone Garatti}
\ShortHeadings{Compression, Generalization and Learning}{M.C. Campi and S. Garatti}

\begin{document}
	
\title{Compression, Generalization and Learning}

\author{\name Marco C. Campi \email marco.campi@unibs.it \\
	\addr Department of Information Engineering \\
	University of Brescia \\
	via Branze 38, 25123 Brescia, Italy
	\AND
	\name Simone Garatti \email simone.garatti@polimi.it \\
	\addr Dipartimento di Elettronica, Informazione e Bioingegneria \\
	Politecnico di Milano\\
	piazza L. da Vinci 32, 20133 Milano, Italy}

\editor{Manfred Warmuth}

\maketitle

\begin{abstract}%
	A \emph{compression function} is a map that slims down an observational set into a subset of reduced size, while preserving its informational content. In multiple applications, the condition that one new observation makes the compressed set change is interpreted that this observation brings in extra information and, in learning theory, this corresponds to misclassification, or misprediction. In this paper, we lay the foundations of a new theory that allows one to keep control on the probability of change of compression (which maps into the statistical ``risk'' in learning applications). Under suitable conditions, the cardinality of the compressed set is shown to be a consistent estimator of the probability of change of compression (without any upper limit on the size of the compressed set); moreover, unprecedentedly tight finite-sample bounds to evaluate the probability of change of compression are obtained under a generally applicable condition of \emph{preference}. All results are usable in a fully \emph{agnostic} setup, i.e., without requiring any \emph{a priori} knowledge on the probability distribution of the observations. Not only these results offer a valid support to develop trust in observation-driven methodologies, they also play a fundamental role in learning techniques as a tool for hyper-parameter tuning.  
\end{abstract}

\begin{keywords}
	Compression Schemes, Statistical Risk, Statistical Learning Theory, Scenario Approach 
\end{keywords}

\section{Introduction} \label{section-intro}

\emph{Compression} is an established topic in theoretical learning, and various generalization bounds have been proven for compression schemes. \\

According to a definition introduced in \cite{LittlestoneWarmuth1986}, a compression scheme consists of i. a \emph{compression function} $\co$, which maps any list of observed examples $S = ((x_1,y_1), \ldots, (x_N,y_N))$ ($x_i$ is called an ``instance'' and $y_i$ a ``label'') into a sub-list $\co(S)$, and ii. a \emph{reconstruction function} $\rho$, which maps any list of examples $S$ into a classifier $\rho(S)$. An important feature of a classifier is its \emph{risk} and, in the context of compression schemes, one is interested in the risk associated to the classifier $\rho(\co(S))$. The concept of risk finds a natural definition in \emph{statistical learning} where one assumes that examples $(x,y)$ are generated according to a random mechanism: the risk of a generic classifier $f$ is then defined as $\risk(f) = \Pr\{f(\bx) \neq \by\}$ (throughout, boldface indicates random quantities). In loose terms, most results in compression schemes establish that low cardinality of the compression implies low risk for the ensuing classifier. A bit more precisely, let $\bS = ((\bx_1,\by_1), \ldots, (\bx_N,\by_N))$ be a list of independent random examples all sharing the same distribution (which coincides with the distribution of $(\bx,\by)$ in the definition of risk). In the example-consistent framework (i.e., the bound on the risk is only given for lists of examples $S$ for which $\rho(\co(S))$ returns the corresponding label $y_i$ for any $x_i$ in $S$ -- in this case $\rho(\co(S))$ is said to be ``consistent'' with $S$) and under the assumption that the maximum cardinality of $\co(\bS)$ is bounded by an integer $d$ ($d$ is called the ``size'' of the compression scheme), \cite{LittlestoneWarmuth1986} and \cite{FloydWarmuth1995} establish results of the type: with high probability $1 \! - \! \delta$ with respect to the generation of the list of examples $\bS$, if $\rho(\co(\bS))$ is consistent with $\bS$, then the risk of $\co(\bS)$ is below a known bound that depends on $N$, $d$ and $\delta$ only (in particular, the bound does not depend on the distribution by which examples are generated). Results in this vein have been subsequently extended to the non-consistent framework and to compression schemes with unbounded size, in which context the best known results are given in \cite{GraepelHerbrichShawe-Taylor2005}.\\

In a series of recent papers, compression schemes have been studied under a \emph{stability} condition, a notion that is natural in many contexts and that has its roots in \cite{VapnikChervonenkis1974}. For the example-consistent framework and compression schemes with bounded size, \cite{pmlr-v125-bousquet20a} succeeded in removing a $\log(N)$ term in the expression of the bound for the risk as compared to the formulation given in the above referenced papers; when applied to Support Vector Machines, this result resolves a long-standing issue that was posed in -- and remained open since -- \cite{VapnikChervonenkis1974}. Later, the scope of \cite{pmlr-v125-bousquet20a} has been significantly broadened by \cite{HannekeKontorovich_2021}, where the non-consistent framework with no upper bounds on the size of the compression scheme has been considered. Since the stable case is most relevant to the present paper, we shall come back to these latter contributions with a more detailed discussion and comparison at the end of Section \ref{section-learning}. \\

In the present paper, we make a paradigm shift: since we are interested in compression as a general tool applicable across various domains, we are well-advised to adopt a ``purist'' approach in which compression functions are studied in isolation (without a reconstruction function). Our essential goal is to study how the probability of change of compression relates to the size of the compressed set. The corresponding results can be applied to supervised learning, unsupervised learning and, in addition, to any other contexts where compression functions are in use. Our findings are summarized at the end of this section, we start with introducing the formal elements of the problem.

\subsection{Mathematical setup and notation}
\label{section-mathsetup}

Examples $z$ are elements from a set $\scZ$ (for instance, in supervised learning $z$ are pairs $(x,y)$). The compression functions we study are permutation invariant. Correspondingly, given any $n=0,1,2,\ldots$ and a list of examples $(z_1,\dots,z_n)$,\footnote{Note that here symbol $n$ indicates the size of a generic list, while $N$ is in use throughout to indicate the actual number of observed examples. The distinction between the two is necessary to accommodate various needs in theoretical developments.} we introduce the associated multiset written as $\ms(z_1,\dots,z_n)$, where the operator ``$\ms$'' removes the ordering in the list, while maintaining repetitions. The set operations $\cup,\cap,\setminus$ (union, intersection, and set difference) are easily extended to multisets using the notion of multiplicity function $\mu_U$ for a multiset $U$, which counts how many times each element of $\scZ$ occurs in $U$. Then, $\mu_{U \cup U'}(z) = \mu_U(z)+\mu_{U'}(z)$, $\mu_{U \cap U'}(z) = \min\big\{\mu_U(z),\mu_{U'}(z)\big\}$, and $\mu_{U \setminus U'}(z) = \max\big\{0,\mu_U(z)-\mu_{U'}(z)\big\}$. Moreover, $U \subseteq U'$ means that $\mu_U(z) \leq \mu_{U'}(z)$ for all $z$, and $|U|$ stands for the cardinality of a multiset where each example is counted as many times as is its multiplicity. Throughout this paper, multisets have always finite cardinality and, any time a multiset is introduced, it is tacitly assumed that it has finitely many elements. A compression function $\co$ is a map from any multiset of examples $U$ to a sub-multiset: $\co(U) \subseteq U$. We write $\co(z_1,\ldots,z_n)$ as a shortcut for $\co(\ms(z_1,\ldots,z_n))$. Also, given a multiset $U$ and one more example $z$, $\co(U,z)$ stands for $\co(U \cup \ms(z))$. Similar notations apply to other maps having a multiset as argument. Throughout, an example is modeled as a realization of a random element defined over a probability space $(\Omega,\mathcal{F},\Pr)$; moreover, a list of $n$ examples is the realization of the first $n$ elements of an independent and identically distributed (i.i.d.) sequence $\bz_1,\bz_2,\ldots$. A training set is a multiset generated from a list of observed examples. When dealing with problems in machine learning, a learning algorithm $\algo$ is a map from training sets to a hypothesis $h$ in a set $\calH$ (in supervised binary classification, $h$ is a concept; in supervised learning with continuous label, $h$ is a predictor; in unsupervised learning, $h$ can, e.g., be a collection of clusters; etc.). According to the above notation, we write $\algo(z_1,\dots,z_N)$ to denote the hypothesis generated by $\algo$ when the input is the training set $\ms(z_1,\dots,z_N)$. We use a $\{0,1\}$-valued function $\loss(h,z)$ to indicate whether or not a hypothesis $h$ is \emph{appropriate} for an example $z$: $\loss(h,z) = 0$ signifies that $h$ is \emph{appropriate} for $z$, while $\loss(h,z) = 1$ corresponds to \emph{inappropriateness} (for instance, in supervised classification, $\loss(h,z) = \One{h(x) \neq y}$, where $\One{}$ is the indicator function). The statistical risk of $h$ is $\risk(h) = \Pr\{\loss(h,\bz) = 1\}$, where $\bz$ is a random element distributed as each $\bz_i$.

\subsection{Main contributions}

The contributions of this paper are summarized in the following three points. \\

{\bf (i)} Under the property of \emph{preference},\footnote{While stated differently, the \emph{preference} property is equivalent to the property of \emph{stability}, see Section \ref{section-compression} for an explanation of our terminology.} Theorem \ref{th:compression_1} establishes a new bound to the probability of change of compression as a function of the cardinality of the compressed multiset. For a finite size of the training set, the bound is informative and useful in applications (see Figure \ref{function-upperepsk}). When the size of the training set $N$ tends to infinity, the bound tends to the ratio $k/N$ (where $k$ is the cardinality of the compressed multiset), uniformly in $k \in \{0,1,\ldots,N\}$ (the fact that the range for $k$ arrives at $N$ means that the compressed multiset has no upper limits other than the size of the training set itself), see Proposition \ref{th:bounds4asympt}. No lower bounds to the probability of change of compression are possible under the sole \emph{preference} property. Under an additional \emph{non-associativity} property and a condition of \emph{non-concentrated} probabilistic mass, Theorem \ref{th:compression_2} establishes a lower bound (see Figure \ref{function-upperepsk-lowerepsk}). This lower bound also converges to $k/N$ uniformly in $k \in \{0,1,\ldots,N\}$. Hence, under the assumptions of Theorem \ref{th:compression_2}, the probability of change of compression is in sandwich between two bounds that merge one on top of the other as $N \to \infty$ (see Proposition \ref{th:bounds4asympt}). This entails that the cardinality of the compressed multiset is a highly informative statistics to evaluate the probability of change of compression.  
\\ 

{\bf (ii)} In Section \ref{section-learning}, the results in (i) are put at work to study classical compression schemes in the presence of a reconstruction function. It is shown that a \emph{preferent} compression scheme augmented with the examples that are misclassified preserves the \emph{preference} property. From this, one finds that the risk can be evaluated without resorting to an incremental approach (as it was customary in previous contributions) in which the empirical risk is incremented with an estimate of the mismatch between empirical and actual risk. The resultant evaluations of the risk, established in Theorem \ref{theorem above-learning-2}, are unprecedentedly sharp. Under the additional conditions of \emph{non-associativity} and of \emph{non-concentrated} mass, if certain \emph{coherence} properties hold, then one obtains bounds on the risk that are valid both from above and from below, which provides a statistically consistent evaluation of the risk. Empirical demonstrations complement the theoretical study and show that the bounds cover tightly the actual stochastic dispersion of the risk. \\ 

{\bf (iii)} As examples of application, the achievements in point (ii) are applied in Section \ref{section-application} to various support vector methods, including the Support Vector Machine (SVM) and the Support Vector Regression (SVR), and to the Guaranteed Error Machine (GEM). This study shows that, in various learning contexts, one can identify statistics of the data from which consistent estimates of the risk can be obtained without resorting to validation or testing.

\subsection{Relation with previous contributions and a more general perspective on this work}

The scientific background in which this work has matured lies in some fifteen years of work by its authors in the field of data-driven optimization. In a group of papers, whose forefathers are \cite{CalCam:05,CalCam:06,CamGa:08} and that include \cite{CamGa:11,Campi_Care:13,CarGarCam2015,CampiGarattiRamponi2018,GarCamCare2019,GarCamCare2023,GarCam2023}, they laid with co-authors the foundations of the so-called ``scenario approach'', a vast body of methods and algorithms to obtain data-driven, theoretically-certified, solutions to uncertain optimization problems. The scenario approach has spurred quite a bit of work also done by others, as witnessed by a large number of theoretical contributions, of which we here only mention the most significant ones: \cite{Welsh_Rojas:09,Welsh_Kong:11,PagnReichCam:12,SchiFaMo:13,SchiFaFrMo2014,MarGouLy:14,MarPraLyl:14,ZhangEtal:2015,EsfSutLy:15,CreKenGie2015,GrammaticoEtal:16,crespo2016interval,LacerdaCrescpoACC2017,MaFaGaPr2018,CreColKenGie2019,FALSONE2019108537}. Recently, the studies on scenario optimization have culminated in the works \cite{CamGa2018,GarCa2019}, which are conceptually linked to the present contribution by the fact that the generalization properties of the solution are evaluated from an observable called ``complexity'' (complexity parallels the size of the compressed multiset of this paper). As compared with all this previous literature, the present contribution introduces two major elements of novelty: \\ 

{\bf (a)} compression takes center stage, beyond any contextualization. By this purist approach, we aim to lay the groundwork for a new theory of wide applicability, to machine learning \emph{in primis}, but also across the other multiple data science fields in which compression finds application; \\

{\bf (b)} by a novel, powerful, theoretical apparatus, this paper establishes bounds on the risk that fare beyond the domain of previous contributions; in particular, they allow one to drop any condition of \emph{non-degeneracy}, which was a standing and limiting assumption in previous works, e.g., \cite{CamGa2018,GarCa2019}. \\

We hope that the findings presented in this paper will open a new era of exploration and discovery in an important subarea of data-driven methods that is centered around the notion of compression. As previously mentioned, we here already consider support vector methods and improve the results in \cite{CampiGaratti2021} by eliminating all assumptions on the distribution of the examples for the problem of obtaining upper bounds on the risk. We also study a generalized version of the so-called Guaranteed Error Machine, which was introduced in \cite{Campi2010} under a limiting condition on the complexity of the classifier. Beyond the applications discussed in this paper, we expect that our results will prove useful in various fields where the scenario approach is applicable (including robust optimization, with its multiform applications to diverse contexts). We feel like to also mention that the authors of this paper are at present actively exploring a wide range of example-driven computer science algorithms in which the application of the compression theory of this paper is made possible through an \emph{importance procedure} for example selection, even in cases where the original algorithm lacks any compression (see \cite{PacCamGar:23} for a study in the context of machine learning). For a broader discussion on the increasing importance of establishing well-founded risk theories for data-driven decision processes, particularly in today's time in which the use of data is becoming pervasive, the reader is referred to the recent position paper \cite{CampiCareGaratti2021}. 

\subsection{Structure of the paper}

The main results on compression schemes are presented in a unified treatment in the next Section \ref{section-compression}, which also includes a discussion on the asymptotic behavior of the bounds. Section \ref{section-learning} presents a rapprochement with classical compression schemes in statistical learning that incorporate a reconstruction function, along with some other more general results useful for machine learning problems. Specific machine learning schemes are considered in Section \ref{section-application}. The proofs of the main results are deferred till Section \ref{section-proofs}.

\section{New generalization results for compression schemes}
\label{section-compression}

Our interest lies in quantifying the probability with which a change of compression occurs. As it was mentioned in the introduction, and it will be further explored in Section \ref{section-application}, this probability has important implications in relation to learning schemes. We start with a formal definition of probability of change of compression. 
\begin{definition}[probability of change of compression] 
The \emph{probability of change of compression} is defined as 
\label{probability-of-change}
$$
\phi(\bz_1,\dots,\bz_n) = \Pr \{ \co(\co(\bz_1,\ldots,\bz_n),\bz_{n+1}) \neq \co(\bz_1,\ldots,\bz_n) | \bz_1,\ldots,\bz_n \}. \footnote{This means that $\phi(\bz_1,\dots,\bz_n)$ is any version of the conditional probability on the right-hand side.}
$$
\qed
\end{definition}
\noindent 
On the right-hand side a new element $\bz_{n+1}$ is added to the compression of $\ms(\bz_1,\dots,\bz_n)$~\footnote{It is not unimportant that $\bz_{n+1}$ is added to the compression, not to the initial multiset.} and it is tested whether this makes the compression change. This gives an event, and our interest lies in the probability of this event. However, in view of its use in applications, what matters is not the probability \emph{tout court}, rather, we take a more fine-grained standpoint by conditioning on $\bz_1,\ldots,\bz_n$, so as to capture the variability of the probability of change of compression as determined by the examples. This makes $\phi(\bz_1,\dots,\bz_n)$ into a random variable. In what follows, we shall often use the symbol $\bphi_n$ as a shorthand for the random variable $\phi(\bz_1,\dots,\bz_n)$. \\

Before delving into the mathematical developments, we are well-advised to digress a moment to discuss the nature of the results we mean to reach. Let $N$ be the size of the multiset at hand, the one of which we want to study the probability of change of compression. $\bphi_N$ has a probability distribution of its own. Arguably, this distribution may vary significantly with the distribution by which the $z_i$'s are generated. This fact has an important implication: any result that describes the distribution of $\bphi_N$ without referring to some prior knowledge on the distribution of the $\bz_i$'s is bound to stay on the conservative side and is therefore poorly informative. While this may seem to set fundamental limitations to obtaining \emph{distribution-free} results on $\bphi_N$ (i.e., results valid without any \emph{a priori} knowledge on the distribution of the $\bz_i$'s), nevertheless it turns out that this conclusion is hasty and incorrect: indeed, one can instead move along a different path and take a bi-variate standpoint, as next explained. Let $|\co(\bz_1,\ldots,\bz_N)|$ be the cardinality of the compressed multiset $\co(\bz_1,\ldots,\bz_N)$. We consider the pair $(|\co(\bz_1,\ldots,\bz_N)|,\bphi_N)$ and identify conditions of general interest under which its bi-variate distribution concentrates in a slender, lenticular-shaped, region (see Figure \ref{function-upperepsk-lowerepsk}). The implications are quite notable: within the lenticular-shaped region, the distribution of $(|\co(\bz_1,\ldots,\bz_N)|,\bphi_N)$ does exhibit a strong variability depending on the problem (which also translates into the variability of the marginal distribution of $\bphi_N$). However, given the realization of $\bz_1,\ldots,\bz_N$ at hand, one can compute the value of $|\co(\bz_1,\ldots,\bz_N)|$ and intersect the vertical line corresponding to this value with the lenticular-shaped region to obtain an interval for the probability of the change of compression. The so-formulated evaluation is tight and informative even for small values of $N$ and offers an useful assessment tool for applications. Importantly, the corresponding theory retains the characteristic of being distribution-free. This finding is stated below as Theorem \ref{th:compression_2} and it holds under two properties called \emph{preference} (Property \ref{preference}) and \emph{non-associativity} (Property \ref{non-associativity}), besides a condition that rules out \emph{concentrated masses} (Property \ref{no-concentrated-mass}). Interestingly, under the sole \emph{preference} property, only the lower bound of the lenticular-shaped region is lost while the upper bound maintains its validity (Theorem \ref{th:compression_1}), which provides a result broadly applicable to evaluate an upper limit on the probability of change of compression as a function of the cardinality of the compressed multiset. \\

Moving towards the mathematical results, we first formalize the concept of \emph{preference}. 
\begin{property}[preference]
\label{preference}
For any multisets $U$ and $V$ such that $V \subseteq U$, if $V \neq \co(U)$, then $V \neq \co(U,z)$ for all $z \in \scZ$. \qed
\end{property}
\noindent 
Hence, if a sub-multiset is not chosen as the compressed multiset, then it cannot become the compressed multiset at a later stage after augmenting the multiset with a new example.\footnote{This property is not new and is called ``stability'' in the literature, see, e.g., \cite{pmlr-v125-bousquet20a} where the formulation is slightly different but, provably, equivalent. Our introducing a change of terminology is in the interest of clarity as we believe that ``stability'' may convey the erroneous idea of \emph{absence of change} or, what is germane to the field of systems theory, the idea that a small input variation can only cause a small output variation. We chose ``preference'' because we feel that this term rightly conveys the idea that a multiset cannot be selected -- and hence preferred -- at a later stage if it had not been preferred earlier when it was already available.} \\

The following lemma provides a useful reformulation of the preference property.
\begin{lemma}
\label{lemma_fund}
A compression function $\co$ satisfies the \emph{preference} property if and only if $\co(V) = \co(U)$ for all multisets $U,V$ such that $\co(U) \subseteq V \subseteq U$. \qed
\end{lemma}
\begin{proof}
Assume $\co$ satisfies the \emph{preference} property and let $z_1,\ldots,z_n$ be the elements in $U \setminus V$ where $U$ and $V$ are multisets such that $\co(U) \subseteq V \subseteq U$. Let
$S_0 = V$ and $S_i = S_{i-1} \cup \ms(z_i)$ for $i=1,\ldots,n$ so that $S_n = U$. Now suppose that $\co(V) \neq \co(U)$. Since $\co(S_0) = \co(V)$ and $\co(S_n) = \co(U)$, then it must be that $\co(S_{i-1}) \neq \co(U)$ and $\co(S_{i-1},z_i) = \co(U)$ for some $i \in \{1,\ldots,n\}$. However, since $\co(U) \subseteq S_{i-1}$, this contradicts the assumption that $\co$ satisfies the \emph{preference} property.\\
For the other direction, assume that the \emph{preference} property does not hold. Then, we can find $U,V,z$ such that $V \subseteq U$ and $\co(U) \neq V = \co(U,z)$. This implies $\co(U,z) \subseteq U \subseteq U \cup \ms(z)$ and $\co(U) \neq \co(U,z)$, contradicting the statement that $\co(V') = \co(U')$ for all multisets $U',V'$ such that $\co(U') \subseteq V' \subseteq U'$.
\end{proof}
An immediate consequence of Lemma~\ref{lemma_fund} is that $\co\big(\co(U)\big) = \co(U)$ whenever $c$ satisfies the \emph{preference} property.\\

The statement of our first theorem is better enunciated by introducing the following functions $\Psi_{k,\delta}: (0,1) \to \R$, which are indexed by $k=0,1,\ldots,N-1$ and by the confidence parameter $\delta\in (0,1)$: 
$$
\Psi_{k,\delta}(\alpha) = \frac{\delta}{N}\sum_{m=k}^{N-1} \frac{\binom{m}{k}}{\binom{N}{k}} (1-\alpha)^{-(N-m)}~.
$$
For any $k$ and any $\delta$, the equation $\Psi_{k,\delta}(\alpha) = 1$ admits one and only one solution in $(0,1)$. Indeed, $\Psi_{k,\delta}(\alpha)$ is strictly increasing, continuous, and $\Psi_{k,\delta}(\alpha) \leq \delta < 1$ when $\alpha \to 0$, while it grows to $+\infty$ when $\alpha \to 1$ (see Figure~\ref{fig:f-special}(a)). 
\begin{figure}[t]
\centering
\includegraphics[width=1\columnwidth]{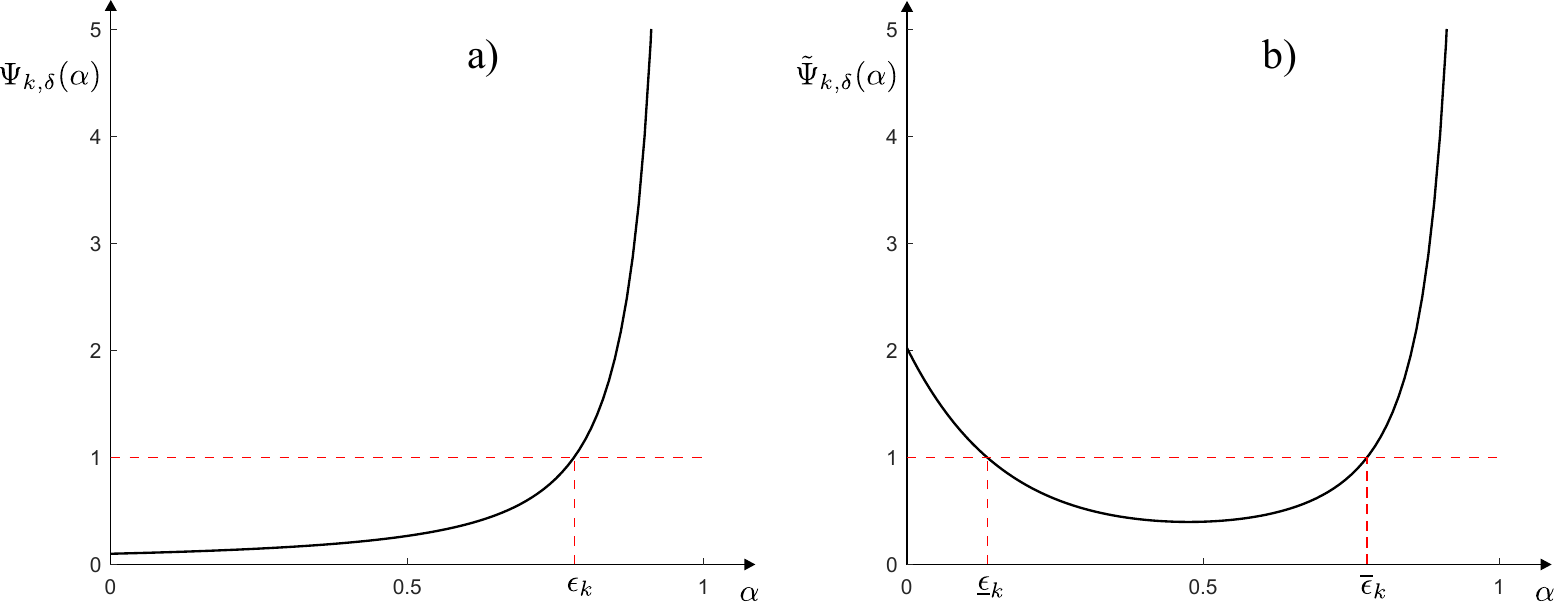}
\caption{(a) Function $\Psi_{k,\delta}(\alpha)$: it starts below $\delta$ when $\alpha \to 0$ and tends to $+ \infty$ when $\alpha \to 1$; (b) Function $\tilde{\Psi}_{k,\delta}(\alpha)$: it tends to $+ \infty$ as $\alpha \to 1$ or $\alpha \to - \infty$ and takes a value below $1$ in a point in $(-\infty, 1)$.}
\label{fig:f-special}
\end{figure}
Define\footnote{$\eps_k$ can be computed via a bisection algorithm. An efficient and ready-to-use \textsf{MATLAB} code is provided in Appendix \ref{appendix-bisection-algo1}.} 
\begin{equation}
\label{epsilonk}
\eps_k =
\begin{cases}
\mbox{solution to } \Psi_{k,\delta}(\alpha) = 1, & k=0,1,\ldots,N-1; \\
1, & k = N.
\end{cases}
\end{equation}
\begin{theorem}
\label{th:compression_1}
Assume the \emph{preference} Property \ref{preference}. For any $\delta \in (0,1)$, it holds that
\begin{equation}
\label{result-th-1}
	\Pr \{ \bphi_N > \eps_\bk \} \leq \delta,
\end{equation}
where $\eps_\bk$ is the random variable obtained by the composition of $\bk := |\co(\bz_1,\ldots,\bz_N)|$ (the cardinality of the multiset $\co(\bz_1,\ldots,\bz_N)$) with the function $\eps_k$ given in \eqref{epsilonk} (in other words, it is $\eps_k$ evaluated at the random value $\bk$). \qed
\end{theorem}
\begin{proof}
The proof of Theorem \ref{th:compression_1} is given in Section \ref{proof-Theorem-1}.
\end{proof}
In the theorem, parameter $\delta$ is called the ``confidence parameter'' and it is normally selected to a very small value, say $10^{-5}$ or $10^{-6}$. The theorem claims that $\bphi_N$, the probability of change of compression, is upper-bounded, with high confidence $1 - \delta$, by $\eps_\bk$, which is a known, deterministic, function $\eps_k$ evaluated in correspondence of the cardinality $\bk$ of the compressed multiset. Figure \ref{function-upperepsk} visualizes function $\epsilon_k$ for $N = 2000$ and various values of $\delta$. \\
\begin{figure}[t]
	\centering
	\includegraphics[width=0.6\columnwidth]{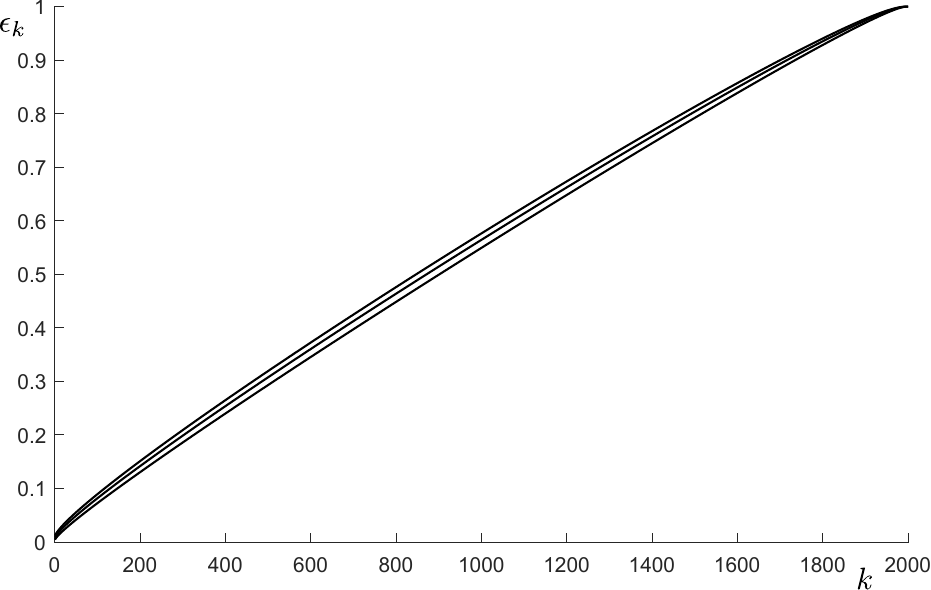}
	\caption{Curve $\eps_k$ against the value of $k$ for $N = 2000$ and various values of $\delta$ ($10^{-3}$, $10^{-6}$, and $10^{-9}$). As established in Theorem \ref{th:compression_1}, for \emph{preferent} compression functions this curve sets an upper bound (valid with confidence $1-\delta$) on the probability of change of compression as a function of the cardinality $k$ of the compressed multiset.}
	\label{function-upperepsk}
\end{figure}

In machine learning applications, the interest of Theorem \ref{th:compression_1}  stems from the fact that a change of compression occurs when an example is misclassified, or mispredicted. Hence, the theorem allows one to upper-bound the probability of misclassification, or misprediction, by using an observable, the cardinality of the compressed multiset. Importantly, the evaluation holds independently of the distribution of the $\bz_i$'s, hence the user can apply the result without positing any, possibly hazardous, conjecture on how data are generated. An ample space to the use of Theorem \ref{th:compression_1} in machine learning problems is given in Sections \ref{section-learning} and \ref{section-application}. \\ 

Lower and upper bounds on $\bphi_N$ are established under additional conditions, the \emph{non-associativity} Property \ref{non-associativity} and the Property \ref{no-concentrated-mass} of \emph{non-concentrated} mass, as described in the following. 

\begin{property}[non-associativity]
\label{non-associativity}
For any $n \geq 0$ and $p \geq 1$, 
$$
\Pr \big( E_1 \setminus  E_2 \big) = 0,
$$
where
$$
E_1 = \{ \co(\bz_1,\ldots,\bz_n,\bz_{n+i}) = \co(\bz_1,\ldots,\bz_n), \; i=1,\dots,p\},
$$
$$
E_2 = \{ \co(\bz_1,\ldots,\bz_n,\bz_{n+1},\ldots,\bz_{n+p}) = \co(\bz_1,\ldots,\bz_n) \}.
$$
\qed
\end{property}
In words, the \emph{non-associativity} property can be phrased as follows: if the compression does not change adding elements one at a time, then it does non change when they are added altogether, with the possible exception of an event whose probability is zero.\footnote{\emph{Non-associativity} is naturally satisfied in many contexts, including all cases in which the compression singles out the relevant observations in a robust optimization process (this is because adding multiple constraints that do not change the solution when considered in isolation -- viz., the current solution is feasible for the new constraints -- does not change the solution when all the constraints are introduced simultaneously). See Section \ref{section-application} for examples in the machine learning context.} The reader may have noticed that Property \ref{non-associativity} is given in probability unlike the \emph{preference} Property \ref{preference}, which was required to hold for any choice of the examples. The reason is that requiring the validity of the non-associativity property for all examples can be restrictive in some applications.  
\begin{property}[non-concentrated mass]
\label{no-concentrated-mass}
$$
\Pr\{ \bz_i = z \} = 0, \; \forall z \in \scZ.
$$
\qed
\end{property}
\noindent 
The property of \emph{non-concentrated} mass simply requires that any $z$ can only be drawn with probability zero and, hence, it excludes with probability $1$ that the same $z$ occurs twice or more times in a training set. \\ 

Theorem \ref{th:compression_2} is stated by means of the following functions $\tilde{\Psi}_{k,\delta}: (-\infty,1) \to \R$ indexed by $k=0,1,\ldots,N$ and by the confidence parameter $\delta\in (0,1)$: \\

\noindent
for $k = 0,\ldots,N-1$, let
\begin{equation}
\label{Psi-tilde}
\tilde{\Psi}_{k,\delta}(\alpha) = \frac{\delta}{2N} \sum_{m=k}^{N-1} \frac{{m \choose k}}{{N \choose k}} (1-\alpha)^{-(N-m)} + \frac{\delta}{6N} \sum_{m=N+1}^{4N} \frac{{m \choose k}}{{N \choose k}} (1-\alpha)^{m-N}, 
\end{equation}
\noindent
while, for $k = N$, let 
$$
\tilde{\Psi}_{N,\delta}(\alpha) = \frac{\delta}{6N} \sum_{m=N+1}^{4N}
{m \choose N} (1-\alpha)^{m-N}.
$$
In Appendix \ref{Appendix_Psi_tilde} it is shown that, for $k=0,1,\ldots,N-1$, equation $\tilde{\Psi}_{k,\delta}(\alpha) = 1$ admits two and only two solutions in $(-\infty,1)$, say $\underline{\alpha}_k$ and $\overline{\alpha}_k$, with $\underline{\alpha}_k < \frac{k}{N} < \overline{\alpha}_k$ (see Figure~\ref{fig:f-special}(b) for a graphical visualization of $\tilde{\Psi}_{k,\delta}(\alpha)$, $k < N$). Instead, equation $\tilde{\Psi}_{N,\delta}(\alpha) = 1$ admits only one solution in $(-\infty,1)$, which is denoted by $\underline{\alpha}_N$ (this is easy to verify because $\tilde{\Psi}_{N,\delta}(\alpha)$ is strictly decreasing and it tends to $0$ as $\alpha \to 1$ while it grows to $+\infty$ as $\alpha \to -\infty$). Define\footnote{See Appendix \ref{appendix-bisection-algo2} for a \textsf{MATLAB} code that efficiently computes $\underline{\eps}_k$ and $\overline{\eps}_k$.}
\begin{equation}
\label{underline_epsilonk}
\underline{\eps}_k = \max \{0,\underline{\alpha}_k \}, \quad k=0,1,\ldots,N, 
\end{equation}
and
\begin{equation}
\label{overline_epsilonk}
\overline{\eps}_k =
\begin{cases}
\overline{\alpha}_k, & k=0,1,\ldots,N-1; \\
1, & k = N.
\end{cases}
\end{equation}
\begin{theorem}
\label{th:compression_2}
Assume the \emph{preference} Property \ref{preference}, the \emph{non-associativity} Property \ref{non-associativity} and the \emph{non-concentrated mass} Property \ref{no-concentrated-mass}. For any $\delta \in (0,1)$, it holds that
\begin{equation} \label{result-th-2}
\Pr \{ \underline{\eps}_\bk \leq \bphi_N \leq \overline{\eps}_\bk \} \geq 1-\delta,
\end{equation}
where $\underline{\eps}_\bk$ is the random variable obtained by the composition of $\bk := |\co(\bz_1,\ldots,\bz_N)|$ with the function $\underline{\eps}_k$ given in \eqref{underline_epsilonk} and $\overline{\eps}_\bk$ is the random variable obtained by the composition of $\bk$ with the function $\overline{\eps}_k$ given in \eqref{overline_epsilonk}. \qed
\end{theorem}
\begin{proof}
The proof of Theorem \ref{th:compression_2} is given in Section \ref{proof-theorem-2}.
\end{proof}

With the additional properties of \emph{non-associativity} and \emph{non-concentrated mass}, Theorem~\ref{th:compression_2} assigns upper and lower bounds for the change of compression, as visualized in Figure \ref{function-upperepsk-lowerepsk}. 
\begin{figure}[t]
	\centering
	\includegraphics[width=0.6\columnwidth]{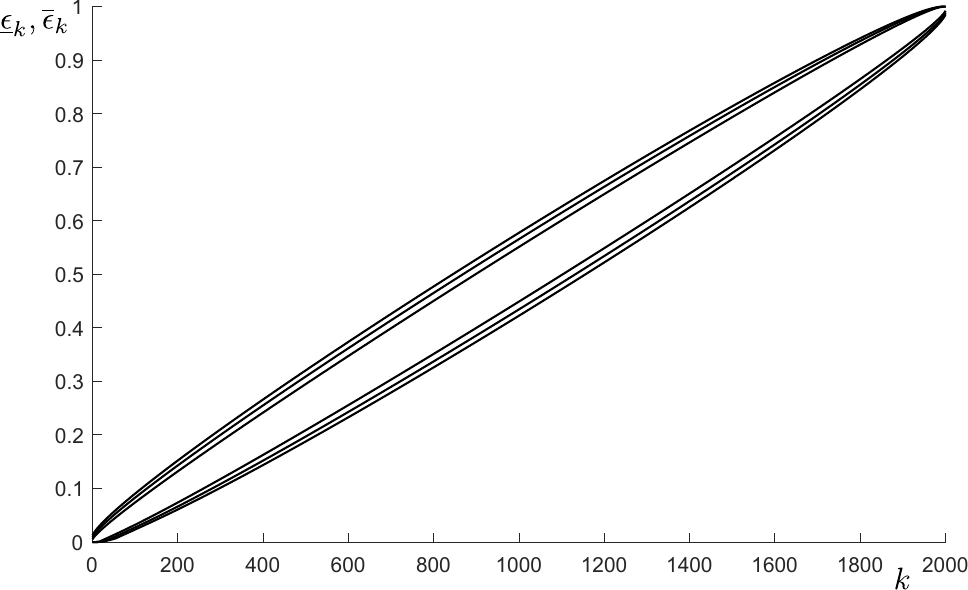}
	\caption{Region delimited by $\underline{\eps}_k$ and $\overline{\eps}_k$ for $N = 2000$ and various values of $\delta$ ($10^{-3}$, $10^{-6}$, and $10^{-9}$). Under the assumptions of Theorem \ref{th:compression_2}, this region contains with confidence $1-\delta$ the probability of change of compression as a function of the cardinality of the compressed multiset.}
	\label{function-upperepsk-lowerepsk}
\end{figure}
Strictly speaking, these additional requirements do depend on the underlying probability by which examples are generated and, hence, they cannot be labeled as being ``distribution-free''. Nonetheless, in various situations, Theorem \ref{th:compression_2} becomes applicable under a very limited knowledge on the distribution of the $\bz_i$'s, and we shall see examples of this in Section \ref{section-application}. We also note that dropping the assumption of \emph{non-concentrated mass} makes Theorem \ref{th:compression_2} false, as shown in the following counterexample.\footnote{While the \emph{non-concentrated mass} property is easy to state, which led us to prefer this formulation, Theorem \ref{th:compression_2} can still be proven under a slightly weaker condition, as briefly discussed in Remark \ref{rmk:weaker_conditions_for_th2} after the proof of the theorem.} Suppose that one single element $\bar{z}$ has probability $1$ to be selected (so that the \emph{non-concentrated mass} property is violated), fix any integer $M$ (for example $M = 100$) and consider the compression function that returns the initial multiset any time this multiset includes the element $\bar{z}$ less than $M$ times, while it trims the number of elements $\bar{z}$ to $M$ when $\bar{z}$ appears more than $M$ times in the initial multiset. It is easily seen that the \emph{preference} and \emph{non-associativity} properties hold. On the other hand, for $N \geq M$ the probability of change of compression is zero with probability $1$, so that no meaningful lower bounds can be assigned in this example. 

\subsection{Asymptotic behavior of $\eps_k$ and $\underline{\eps}_k$, $\overline{\eps}_k$} \label{section-asymptotic}

The purpose of this section is to establish explicit lower and upper bounds on $\eps_k$ and $\overline{\eps}_k$, $\underline{\eps}_k$ able to reveal the dependencies of these quantities on $k$, $N$, and $\delta$, and also to pinpoint convergence properties as $N$ tends to infinity. The main result is in Proposition \ref{th:bounds4asympt}, followed by some comments. We advise the reader that the explicit bounds in Proposition \ref{th:bounds4asympt} are in use to clarify various dependencies, but they are not meant for practical computation since they lead to conservative results if used in place of the numerical procedures given in Appendix~\ref{appendix:MATLAB_code}. 
\begin{prop} \label{th:bounds4asympt}
\begin{align}
	\overline{\eps}_k & \; \leq \;  \frac{k}{N} + 2 \frac{\sqrt{k+1}}{N}\left( \sqrt{\ln(k+1)} + 4 \right) + 2 \frac{ \sqrt{k+1}\sqrt{\ln \frac{1}{\delta}} }{N}  + \frac{\ln \frac{1}{\delta}}{N} \label{eq:bounds4asympt_up} \\ 
	\underline{\eps}_k & \; \geq \; \frac{k}{N} -3 \frac{\sqrt{k+1}}{N}\left( \sqrt{\ln(k+1)} + 2 \right) - 3 \frac{\sqrt{k+1}\sqrt{\ln \frac{1}{\delta}}}{N}. \label{eq:bounds4asympt_low}
\end{align}
Moreover, it holds that 
\begin{align}
	\frac{k}{N} \leq \eps_k & \; \leq \; \overline{\eps}_k \nonumber 
\end{align}
\end{prop}
\begin{proof}
The proof of Proposition \ref{th:bounds4asympt} is given in Section \ref{proof_bounds4asympt}. 
\end{proof}
In both \eqref{eq:bounds4asympt_up} and \eqref{eq:bounds4asympt_low}, the dependence on $\delta$ is inversely logarithmic, which shows that ``confidence is cheap'': very small values of $\delta$ can be enforced without significantly affecting the results and, thereby, the width of the interval $[\underline{\eps}_k,\overline{\eps}_k]$ (see again Figure \ref{function-upperepsk-lowerepsk}). For any fixed $k$, we see that $\eps_k$ and $\overline{\eps}_k$, $\underline{\eps}_k$ tend to $k/N$ as $O(1/N)$, while for $k$ that grows at the same rate as $N$ (say $k/N$ = constant) $\eps_k$ and $\overline{\eps}_k$, $\underline{\eps}_k$ converge towards $k/N$ as $O(\sqrt{\ln (N)}/\sqrt{N})$. This is just marginally slower than the convergence rate for the law of large numbers, as given by the central limit theorem. 
\begin{figure}[t]
	\centering
	\includegraphics[width=\columnwidth]{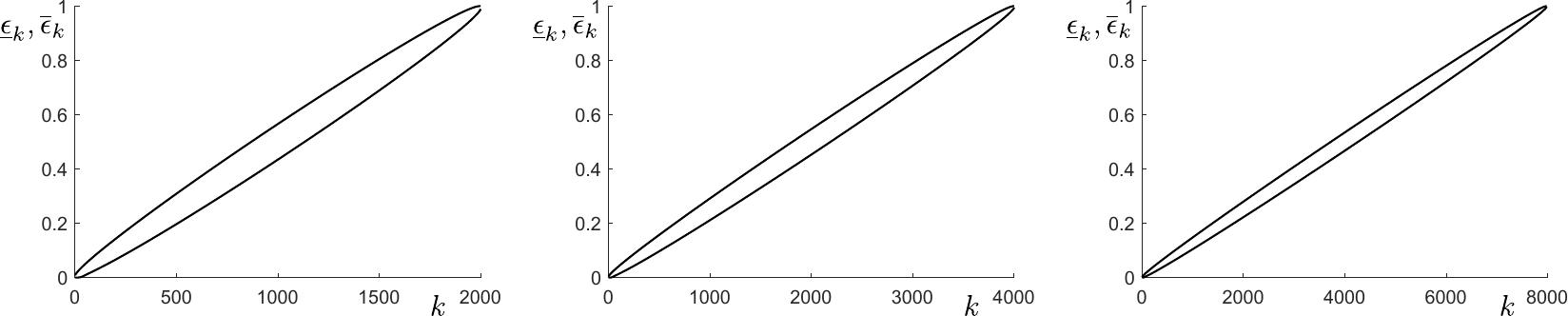}
	\caption{Graph of $\overline{\eps}_k$, and $\underline{\eps}_k$ as functions of $k$ for $\delta = 10^{-6}$ and $N= 2000$, $4000$, and $8000$. 
	}
	\label{fig:epsLU_N=2_4_8_x1000_bet=1e-6}
\end{figure}
The evolution of  $\overline{\eps}_k$, $\underline{\eps}_k$ as $N$ grows can be seen in Figure \ref{fig:epsLU_N=2_4_8_x1000_bet=1e-6}. \\

Reading the results of this section in the light of Theorem \ref{th:compression_1}, one concludes that, for all compression functions satisfying the \emph{preference} Property \ref{preference}, the bi-variate distribution of $\bk = 
|\co(\bz_1,\ldots,\bz_N)|$ and $\bphi_N$ all lies below the line $k/N$ plus an offset whose size goes to zero as $O(\sqrt{\ln (N)}/\sqrt{N})$ with the exception of a slim tail whose probabilistic mass is no more than $\delta$. If, additionally, Properties \ref{non-associativity} and \ref{no-concentrated-mass} hold, Theorem \ref{th:compression_2} shows that the bi-variate distribution of $\bk = |\co(\bz_1,\ldots,\bz_N)|$ and $\bphi_N$ all lies in a strip around $k/N$ whose size goes to zero as $O(\sqrt{\ln (N)}/\sqrt{N})$ but a slim tail. In this latter case, Theorem \ref{th:compression_2} also carries the very important implication that the ratio $\bk/N$ is a strongly consistent estimator of $\bphi_N$ irrespective of the problem at hand (it can be proven that Theorem \ref{th:compression_2} and Proposition \ref{th:bounds4asympt} together imply that $|\bk/N-\bphi_N|$ converges to zero both in the mean square sense and almost surely). 

\section{Compression schemes for machine learning}
\label{section-learning}

The aim of this section is to connect the theory of Section \ref{section-compression} to that of statistical risk in learning algorithms. Our findings will be compared with existing results at the end of this section. Refer to Section \ref{section-mathsetup} for the mathematical setup and notation. \\ 

Given a learning algorithm $\algo$, suppose that there exists a compression function $\co$ that ties in with the loss function $\loss$ according to the following property.
\begin{property}[coherence -- part I]
\label{mis_implies_change}
For any $n \geq 0$ and any choice of $z_1,\ldots,z_n,z_{n+1} \in \scZ$, if $\loss\big(\algo(z_1,\ldots,z_n),z_{n+1}\big) = 1$, then $\co\big(\co(z_1,\ldots,z_n),z_{n+1}\big) \neq \co(z_1,\ldots,z_n)$. \qed
\end{property}
Under Property \ref{mis_implies_change}, $\risk(\algo(\bz_1,\ldots,\bz_N)) \leq \bphi_N$ holds with probability $1$,\footnote{The reason why the inequality holds with probability $1$ and not always is that $\bphi_N$ is just a version of the conditional probability in Definition \ref{probability-of-change} and various versions can differ over events having probability zero.} and Theorem \ref{th:compression_1} can be used to bound the risk of the hypothesis returned by the learning algorithm, as specified in the following theorem.
\begin{theorem}
\label{theorem above-learning}
Given a learning algorithm $\algo$, suppose that there exists a compression function $\co$ that satisfies the \emph{coherence -- part I} Property \ref{mis_implies_change}. Assume the \emph{preference} Property~\ref{preference}. For any $\delta \in (0,1)$, it holds that
\[
\Pr\Big\{ \risk\big(\algo(\bz_1,\ldots,\bz_N)\big) > \eps_\bk \Big\} \le \delta,
\]
where $\bk = |\co(\bz_1,\ldots,\bz_N)|$ and $\eps_k$ is given in \eqref{epsilonk}. \qed
\end{theorem}
We next provide a sufficient condition for Property \ref{mis_implies_change} to hold for the case in which the learning algorithm can be reconstructed from the compressed multiset.
\begin{definition}[reconstruction function]
\label{def-reconstruction}
Given a learning algorithm $\algo$ and a compression function $\co$, a reconstruction function $\rho$ is a map from multisets to hypotheses such that $\rho\big(\co(z_1,\ldots,z_n)\big) = \algo(z_1,\ldots,z_n)$ for any multiset $\ms(z_1,\dots,z_n)$. \qed
\end{definition}
We also need the following property, which requires that the examples in the training set for which the hypothesis chosen by the learning algorithm $\algo$ is \emph{inappropriate} are included in the compressed multiset.
\begin{property}[inclusion]
\label{inclusion}
For any multiset $\ms(z_1,\ldots,z_n)$ and for any $i = 1,2,\ldots,n$, if $\loss\big( \algo(z_1,\ldots,z_n),z_i \big) = 1$, then $z_i$ appears in $\co(z_1,\ldots,z_n)$ the same number of times as it appears in $\ms(z_1,\ldots,z_n)$. \qed
\end{property}
The following lemma shows that \emph{inclusion} implies \emph{coherence -- part I} whenever $\algo$ admits a reconstruction function for the given $\co$.
\begin{lemma}
\label{lemma-violation-gives-change}
Given a learning algorithm $\algo$ and a compression function $\co$ satisfying the \emph{inclusion} Property \ref{inclusion}, if there exists a reconstruction function for $(\algo,\co)$, then the \emph{coherence -- part I} Property \ref{mis_implies_change} holds. \qed
\end{lemma}
\begin{proof}
We prove that, under \emph{inclusion} and existence of a reconstruction function, absence of change of compression is necessarily associated to \emph{appropriateness} (contrapositive of \emph{coherence -- part I} Property \ref{mis_implies_change}).

Let $U := \ms(z_1,\ldots,z_n)$ and suppose that for a new $z_{n+1}$ the compression does not change, i.e.,
\begin{equation} \label{k_not_change}
\co\big( \co(U),z_{n+1} \big) = \co(U).
\end{equation}
Applying $\rho$ to both sides of \eqref{k_not_change} and using the definition of reconstruction function gives
\begin{equation}
\label{fact1}
\algo\big( \co(U),z_{n+1} \big) = \algo(U).
\end{equation}
On the other hand, \eqref{k_not_change} implies that $z_{n+1}$ appears in $\co\big( \co(U),z_{n+1} \big)$ as many times as it does
in $\co(U)$. Thus, $z_{n+1}$ appears in $\ms\big( \co(U),z_{n+1} \big)$ one more time than it does in $\co\big( \co(U),z_{n+1} \big)$. Since $\algo$ and $\co$ satisfy the \emph{inclusion} Property \ref{inclusion}, it follows that
\begin{equation}
\label{fact2}
\loss\Big( \algo\big( \co(U),z_{n+1} \big), z_{n+1} \Big) = 0,
\end{equation}
and, substituting \eqref{fact1} in \eqref{fact2} gives
\[
\loss\big( \algo(U),z_{n+1} \big) = 0.
\]
Hence, it remains proven that
\[
\co\big( \co(U),z_{n+1} \big) = \co(U) \; \Longrightarrow \; \loss\big( \algo(U),z_{n+1} \big) = 0,
\]
which is the contrapositive of Property \ref{mis_implies_change}.
\end{proof}
\begin{remark}
\label{inclusion and reconstruction}
If, for any multiset, an algorithm generates a hypothesis that is \emph{appropriate} for all the examples in the multiset (i.e., the hypothesis is consistent with the multiset), then the \emph{inclusion} Property \ref{inclusion} is automatically satisfied and, hence, the existence of a reconstruction function implies the \emph{coherence -- part I} Property \ref{mis_implies_change}. The \emph{inclusion} Property \ref{inclusion} provides a condition for the \emph{coherence -- part I} property to hold when the algorithm is allowed to generate hypotheses without \emph{appropriateness} requirements on the training set. \\
Notice also that the \emph{inclusion} property alone (without a reconstruction function) does not imply the \emph{coherence -- part I} property. For example, consider points $z_i \in \Real{}$ and let\footnote{If two or more $z_i$ attain the maximum, then the second largest equals the largest.} $\algo(z_1,\ldots,z_n) = (-\infty, \mbox{second largest } z_i]$ and $\co(z_1,\ldots,z_n) = \max\{z_1,\ldots,z_n\}$,\footnote{For a training set that has only one element or it is empty, let, e.g., the algorithm return the whole real line and the compression coincide with the training set.} and say that $\algo(z_1,\ldots,z_n)$ is \emph{appropriate} for $z$ if $z \in \algo(z_1,\ldots,z_n)$. Here, one can verify that the \emph{inclusion} property holds, while no reconstruction function exists. If a new point $z_{n+1}$ falls in between the second largest $z_i$ and $\max\{z_1,\ldots,z_n\}$, then $\algo(z_1,\ldots,z_n)$ is not \emph{appropriate} for this $z_{n+1}$, but the compression does not change, that is, the \emph{coherence -- part I} property does not hold. \qed
\end{remark}
The statement of Lemma \ref{lemma-violation-gives-change} does not admit a converse: under the existence of a reconstruction function, \emph{coherence -- part I} does not imply \emph{inclusion}. To see this, let examples $z_i$ be points of $\Real{}$ and consider $\algo(z_1,\ldots,z_n) = [\mbox{second largest } z_i,+\infty)$, while $\co(z_1,\ldots,z_n) = \mbox{second largest } z_i$.\footnote{When the training set has only one element $z_1$, let the algorithm return $[z_1,+\infty)$ and the compression be $\ms(z_1)$ while, with an empty training set, the algorithm returns the empty subset of $\Real{}$ and the compression is obviously empty.} Then, $\algo(z_1,\ldots,z_n)$ can be reconstructed from $\co(z_1,\ldots,z_n)$ and, when a point for which $\algo(z_1,\ldots,z_n)$ is \emph{inappropriate} (i.e., the point does not belong to $\algo(z_1,\ldots,z_n)$) is added to $\co(z_1,\ldots,z_n)$, the compression becomes the newly added point (which is now the second largest\footnote{If the compression is empty, then the newly added point is alone and it becomes the compression as well.}) so that the \emph{coherence -- part I} property holds. On the other hand, if there are among  $z_1,\ldots,z_n$ some examples strictly smaller than  the second largest $z_i$, then $\algo(z_1,\ldots,z_n)$ is \emph{inappropriate} for all of these examples while these examples are not in the compression (and, hence, the \emph{inclusion} property does not hold). \\ 

Interestingly, the properties of \emph{inclusion} and \emph{coherence -- part I} become equivalent under \emph{preference}, a fact that is stated in the next lemma.
\begin{lemma}
Consider a learning algorithm $\algo$ and a compression function $\co$ that satisfies the \emph{preference} Property~\ref{preference}. Assume that there exists a reconstruction function $\rho$ for $(\algo,\co)$. Then, the \emph{inclusion} Property \ref{inclusion} holds iff the \emph{coherence -- part I} Property \ref{mis_implies_change} holds.
\end{lemma}
\begin{proof}
In view of Lemma \ref{lemma-violation-gives-change}, we only need to show the implication \emph{coherence -- part I} $\Rightarrow$ \emph{inclusion}.

Let $U := \ms(z_1,\ldots,z_n)$. Assume the \emph{coherence -- part I} property and, by contradiction, that the \emph{inclusion} property fails so that there is a $z_i \in U$ such that $\loss\big( \algo(U),z_i \big) = 1$ and $z_i$ does not appear in $\co(U)$ as many times as it does in $U$. Now, $\algo(\co(U)) = \rho(\co(\co(U))) = \rho(\co(U)) = \algo(U)$ (where the second last equality is true because $\co(\co(U)) = \co(U)$ under \emph{preference} -- see the comment immediately after Lemma \ref{lemma_fund}); hence, $\loss\big( \algo(\co(U)),z_i \big) = \loss\big( \algo(U),z_i \big) = 1$. By the \emph{coherence -- part I} property, we then have: $\co(\co(U),z_i) \neq \co(U)$, which contradicts Lemma \ref{lemma_fund} by the choice $V = \ms(\co(U),z_i)$ for which $\co(U) \subseteq V \subseteq U$.
\end{proof}
Under \emph{preference} and the existence of a reconstruction function, the previous lemma shows that \emph{inclusion} is strictly necessary to have the \emph{coherence -- part I} property. The next lemma shows a way to secure \emph{inclusion} (and thereby \emph{coherence -- part I}) by augmenting the compression function so as to include examples for which the hypothesis is \emph{inappropriate}.
\begin{lemma}
\label{inclusion by augmentation}
Consider a learning algorithm $\algo$ and a compression function $\co$ that satisfies the \emph{preference} Property~\ref{preference} ($(\algo,\co)$ are not required to satisfy the \emph{inclusion} Property \ref{inclusion}). Assume that there exists a reconstruction function $\rho$  for $(\algo,\co)$. Define a new couple $(\tilde{c},\tilde{\rho})$ as follows:
\begin{itemize}
\item[$\bullet$] for any multiset $U$, let $\tilde{\co}(U) = \co(U) \cup \big( \ms(z_i \in U: \loss(\algo(U)),z_i) = 1) \; \setminus \; \co(U) \big)$;\\
(i.e., $\tilde{\co}(U)$ is $\co(U)$ augmented with the examples that are not already in $\co(U)$ for which $\algo(U)$ is \emph{inappropriate});
\item[$\bullet$] for any multiset $U$, let $\tilde{\rho}(U) = \algo(U)$.
\end{itemize}
Then,
\begin{itemize}
\item[(i)] $\tilde{\co}$ satisfies the \emph{preference} Property \ref{preference};
\item[(ii)] $\tilde{\rho}$ is a reconstruction function for $(\algo,\tilde{\co})$;
\item[(iii)] $(\algo,\tilde{\co})$ satisfies the \emph{inclusion} Property~\ref{inclusion} (and, thereby, the \emph{coherence -- part I} Property~\ref{mis_implies_change}). \qed
    \end{itemize}
\end{lemma}
\begin{proof} $\phantom A$ \\
\\
\emph{(i)} Consider any two multisets $U$ and $V$ such that $\tilde{\co}(U) \subseteq V \subseteq U$. We want to show that $\tilde{\co}(V) = \tilde{\co}(U)$, which, by Lemma \ref{lemma_fund}, implies that $\tilde{\co}$ satisfies the \emph{preference} Property \ref{preference}. By definition of $\tilde{\co}$, it holds that $\co(U) \subseteq  \tilde{\co}(U)$, yielding $\co(U) \subseteq V \subseteq U$. Since $\co$ satisfies the \emph{preference} Property \ref{preference}, Lemma \ref{lemma_fund} gives $\co(V) = \co(U)$, which, together with the fact that $\rho$ is a reconstruction function for $(\algo,\co)$, also implies that $\algo(V) = \rho(\co(V)) = \rho(\co(U)) = \algo(U)$. Using $\co(V) = \co(U)$ and  $\algo(V) = \algo(U)$ in the definition of $\tilde{\co}(V)$ gives
\begin{eqnarray}
\label{eq:aux_extended}
\tilde{\co}(V) & = & \co(V) \cup \big( \ms(z_i \in V: \loss(\algo(V),z_i) = 1) \; \setminus \; \co(V) \big) \nonumber \\
& = & \co(U) \cup \big( \ms(z_i \in V: \loss(\algo(U),z_i) = 1) \; \setminus \; \co(U) \big) \nonumber \\
& = & \co(U) \cup \big( \ms(z_i \in U: \loss(\algo(U),z_i) = 1) \; \setminus \; \co(U) \big),
\end{eqnarray}
where the last equality follows by the observation that $\tilde{\co}(U) \subseteq V$ and that $\tilde{\co}(U)$ contains by definition all the $z_i \in U$ such that $\loss(\algo(U),z_i) = 1$. The thesis follows by observing that the right-hand side of \eqref{eq:aux_extended} is $\tilde{\co}(U)$. \\
\\
\emph{(ii)} For any multiset $U$ it holds that $\co(U) \subseteq \tilde{\co}(U) \subseteq U$ and, since $\co$ satisfies the \emph{preference} Property \ref{preference}, Lemma \ref{lemma_fund} gives that $\co(\tilde{\co}(U)) = \co(U)$. Recalling now the definition of $\tilde{\rho}$ and the fact $\rho$ is a reconstruction function for $(\algo,\co)$, we have that
$$
\tilde{\rho}(\tilde{\co}(U)) = \algo(\tilde{\co}(U)) = \rho(\co(\tilde{\co}(U))) = \rho(\co(U)) = \algo(U).
$$
\emph{(iii)} This is obvious in view of the definition of $\tilde{\co}$.
\end{proof}

Using Lemma \ref{inclusion by augmentation}, the following theorem follows immediately from Theorem \ref{theorem above-learning}.
\begin{theorem}
\label{theorem above-learning-2}
Consider a learning algorithm $\algo$ and a compression function $\co$ that satisfies the \emph{preference} Property~\ref{preference}. Assume that there exists a reconstruction function $\rho$  for $(\algo,\co)$. Define $\tilde{\co}$ as in Lemma \ref{inclusion by augmentation}. For any $\delta \in (0,1)$, it holds that
\[
\Pr\Big\{ \risk\big(\algo(\bz_1,\ldots,\bz_N)\big) > \eps_\bk \Big\} \le \delta,
\]
where $\bk = |\tilde{\co}(\bz_1,\ldots,\bz_N)|$ and $\eps_k$ is given in \eqref{epsilonk}. \qed
\end{theorem}
Upper and lower bounds for $\risk(\algo(\bz_1,\ldots,\bz_N))$ can be established under additional conditions.
\begin{property}[coherence -- part II]
\label{chance implies mis}
For any $n \geq 0$ and $p \geq 1$,
$$
\Pr \big( E_1 \setminus  E_2 \big) = 0,
$$
where 
$$
E_1 = \{ \co\big(\co(\bz_1,\ldots,\bz_n),\bz_{n+1}\big) \neq \co(\bz_1,\ldots,\bz_n) \},
$$
$$
E_2 = \{ \loss\big(\algo(\bz_1,\ldots,\bz_n),\bz_{n+1}\big) = 1 \}.
$$
\qed
\end{property}
The \emph{coherence -- part II} property requires that $E_2$ covers $E_1$ up to an event of probability zero. The reason for not requiring that $E_1 \setminus  E_2 = \emptyset$ is that non-pathological examples can be exhibited where this latter condition fails (while the one in probability does hold), showing that this requirement would be unduly restrictive. \\ 

We now have the following theorem, which can be proven from Theorem \ref{th:compression_2} in the light of the two \emph{coherence} properties. 
\begin{theorem}
\label{theorem above below-learning}
Given a learning algorithm $\algo$, suppose that there exists a compression function $\co$ that satisfies the \emph{coherence -- part I} Property \ref{mis_implies_change} and the \emph{coherence -- part II} Property \ref{chance implies mis}. Assume the \emph{preference} Property \ref{preference}, the \emph{non-associativity} Property \ref{non-associativity} and the \emph{non-concentrated mass} Property \ref{no-concentrated-mass}. For any $\delta \in (0,1)$, it holds that
$$
\Pr\Big\{ \underline{\eps}_\bk \leq \risk\big(\algo(\bz_1,\ldots,\bz_N)\big) \leq \overline{\eps}_\bk  \Big\} \geq 1-\delta,
$$
where $\bk = |\co(\bz_1,\ldots,\bz_N)|$ and $\underline{\eps}_k$, $\overline{\eps}_k$ are given respectively in \eqref{underline_epsilonk}, \eqref{overline_epsilonk}. \qed 
\end{theorem}
\begin{proof}
Consider events $E_1$ and $E_2$ as in the statement of the \emph{coherence - part II} Property \ref{chance implies mis} and notice that $\bphi_N = \Pr \{E_1 | \bz_1,\ldots,\bz_N \}$ while $\risk\big(\algo(\bz_1,\ldots,\bz_N)\big) = \Pr \{E_2 | \bz_1,\ldots,\bz_N \}$. Properties \ref{mis_implies_change} and  \ref{chance implies mis} then imply that $\bphi_N \neq \risk\big(\algo(\bz_1,\ldots,\bz_N)\big)$ over a zero probability set only, and the conclusion of Theorem \ref{theorem above below-learning} follows in view of \eqref{result-th-2}.
\end{proof} 

We close this section with some comparison of the results given here with previous results established under the \emph{preference} property (or the \emph{stability} property, as it is phrased in some contributions) for compression schemes that consist of a compression function $\co$ and a reconstruction function $\rho$. In an example-consistent framework (i.e., the bound on the risk is only given for multisets $S$ for which $\rho(\co(S))$ is \emph{appropriate} for all $z_i \in S$), the best available result for the case when a threshold on the maximum cardinality of $\co(S)$ is known is given by Theorem 15 in \cite{pmlr-v125-bousquet20a}. When $N \to \infty$, the bound on the risk in \cite{pmlr-v125-bousquet20a} exhibits a $O(1/N)$ convergence rate to zero, in line with the asymptotic results of this paper given in Section \ref{section-asymptotic}. The findings of \cite{pmlr-v125-bousquet20a} have been extended to compression schemes that have no upper limit for the maximum cardinality of $\co(S)$ in \cite{HannekeKontorovich_2021}, Theorem 10. As compared to \cite{HannekeKontorovich_2021}, our upper bound shows a uniform (in $|\co(S)| \in \{0,1,\ldots,N\}$) convergence towards $|\co(S)|/N$, which is unattainable within the approach of \cite{HannekeKontorovich_2021} (where $|\co(S)|/N$ is multiplied by a non-unitary constant). Moreover, our bound is unprecedentedly sharp for finite values of $N$ and gets rapidly close to $|\co(S)|/N$ as $N$ grows, see Figure \ref{fig:epsLU_N=2_4_8_x1000_bet=1e-6}. Moving to the non-consistent framework, the available literature aims at bounding the gap between the empirical probability of \emph{inappropriateness} (ratio between the number of examples in the training set for which the selected hypothesis is \emph{inappropriate} divided by the size of the training set) and the actual probability of \emph{inappropriateness} (i.e., the actual risk). Our Theorem \ref{theorem above-learning-2} departs from this approach by allowing for an evaluation of the risk that uses directly the size of an augmented compression that automatically incorporates the empirical cases of \emph{inappropriateness}. This allows us to use the same bound for the risk, without distinguishing between the consistent and non-consistent frameworks. The ensuing theory reveals all its sharpness when change of compression and \emph{inappropriateness} are equivalent as specified in Theorem \ref{theorem above below-learning} (this is, e.g., the case for Support Vector Regression in Section \ref{section-SVR} or the Guaranteed Error Machine in Section \ref{section-GEM} under mild conditions), in which case one can establish lower and upper bounds that converge one to the other for increasing $N$ as shown in Section \ref{section-asymptotic} (which is unprecedented in statistical learning).\footnote{This also shows the sharpness of the upper bound in Theorem \ref{theorem above-learning-2} because $\eps_k \leq \overline{\eps}_k$ (see Proposition \ref{th:bounds4asympt}) and the cases dealt with in Theorem \ref{theorem above below-learning} that admit lower bound $\underline{\eps}_k$ and upper bound $\overline{\eps}_k$ have to be accommodated in Theorem \ref{theorem above-learning-2} as well.} See also the next Example \ref{example:cvx_hull} for a numerical simulation that shows that our bounds for finite $N$ well capture the intrinsic stochastic variability of the risk. 

\begin{exa} \label{example:cvx_hull}
We consider a sample of $1000$ points drawn in an independent fashion in $\Real{3}$ and an algorithm $\algo$ that constructs the corresponding convex hull (see Figure \ref{convex-hull}). 
\begin{figure}[t]
	\centering
	\includegraphics[width=0.4\columnwidth]{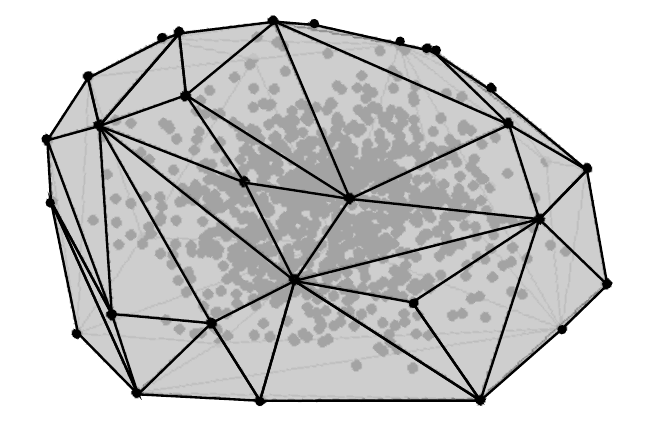}
	\caption{Convex hull of points in $\Real{3}$.}
	\label{convex-hull}
\end{figure}
The compression function $\co$ returns the vertexes of the convex hull (in case of multiple points corresponding to the same vertex, only one point is put in the compression) and a new point is \emph{appropriate} if it belongs to the convex hull. It is easy to check that $\co$ satisfies the \emph{preference} Property \ref{preference} and the \emph{non-associativity} Property \ref{non-associativity} and that \emph{coherence -- part I} Property \ref{mis_implies_change} and \emph{coherence -- part II} Property \ref{chance implies mis} also hold. Hence, if the probability by which the points are drawn has no concentrated mass (for instance, if it admits density), then the \emph{non-concentrated mass} Property \ref{no-concentrated-mass} is also verified and Theorem \ref{theorem above below-learning} can be used to assess the risk.
	\begin{figure}[t]
	\centering
	\includegraphics[width=\columnwidth]{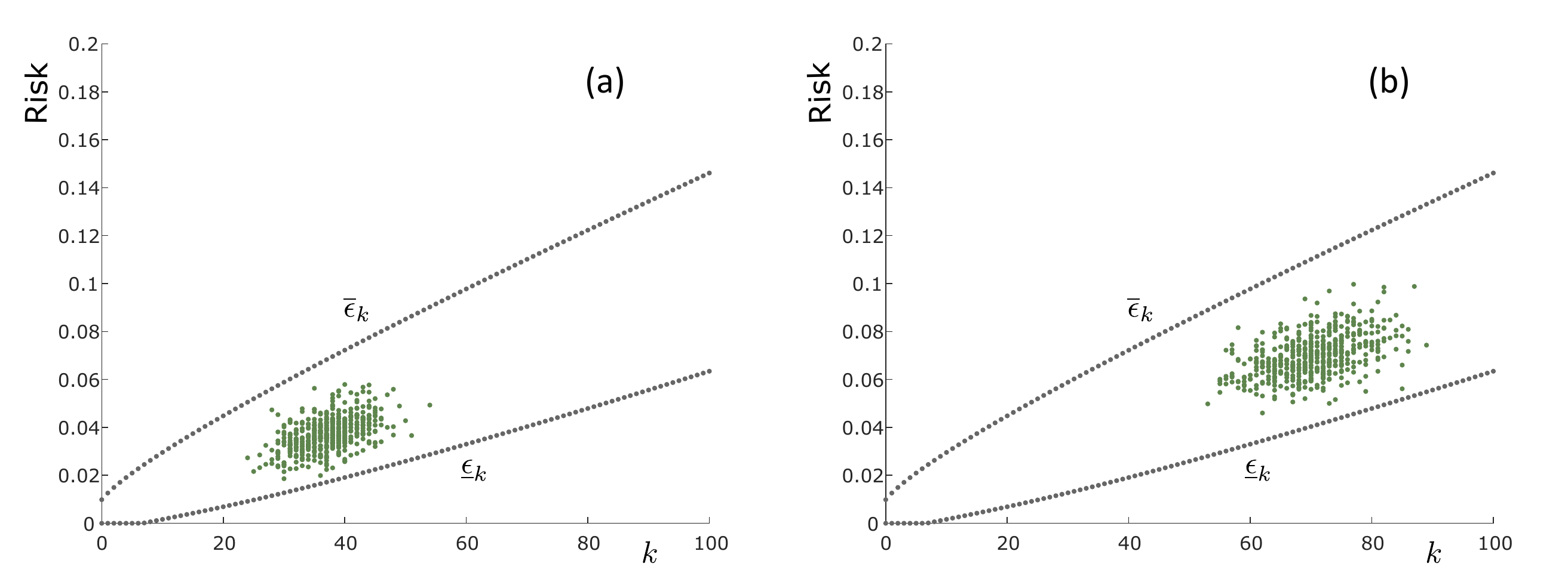}
	\caption{Region delimited by $\underline{\eps}_k$ and $\overline{\eps}_k$ for $N = 1000$ and $\delta = 10^{-3}$. The green dots are generated by a Monte-Carlo testing with (a) a Gaussian distribution and (b) a uniform distribution.}
	\label{convex-hull-risk}
\end{figure}
Panels (a) and (b) in Figure \ref{convex-hull-risk} profile the region delimited by $\underline{\eps}_k$ and $\overline{\eps}_k$ for $N = 1000$ and $\delta = 10^{-3}$. The green dots have coordinates equal to the cardinality of the compressed multiset ($x$ axis) and the risk ($y$ axis) in a Monte-Carlo testing in which points in $\Real{3}$ have a Gaussian distribution -- panel (a) -- and a uniform distribution in a hyper-cube -- panel (b). One sees that the two clouds of green dots in (a) and in (b) are quite different, while, in both cases, they belong to the region (compare with the discussion in Section \ref{section-compression}). Moreover, the stochastic fluctuation of the two clouds well covers the gap between the lower and the upper bound, signifying that the bounds are tight. \qed
\end{exa}
\begin{remark}{\bf (On the role of observations)}
	We feel advisable to just touch upon here an aspect, the full study of which goes beyond the intended goal of this paper. In data-driven applications, it is common practice that observations are split in two sets, used respectively for training and testing. This paper shows that, in a compression setup, data can well stand a double role, in which they are all used for training, while preserving their usability in the process of assessing the risk. Indeed, under the assumptions of Theorem \ref{theorem above below-learning} one can show that the quality of risk assessment by means of $\underline{\eps}_k$ and $\overline{\eps}_k$ (which is based on the sample of data points that has been used for training) only marginally degrades as compared to testing the solution with a new, untouched, sample of data points of equal cardinality. Hence, in this context, saving data for testing seems inappropriate, particularly when data are a scarce or costly resource. \qed
\end{remark}
For the non-consistent framework, it is interesting to further contrast asymptotic bounds that stem from the theory of this paper with previous results. To facilitate the comparison, we first reformulate our upper bound in terms of the empirical probability of inappropriateness $\hat{\risk}\big(\algo(\bS)\big)$ (empirical risk). Start by noticing that $|\tilde{\co}(\bS)|$ in Theorem \ref{theorem above-learning-2} can be bounded as follows: $|\tilde{\co}(\bS)| \leq \hat{\risk}\big(\algo(\bS)\big) N + |\co(\bS)|$ (strict inequality occurs when one or more examples $\bz_i$ are simultaneously inappropriate and also contained in $\co(\bS)$ and, therefore, counted twice in the right-hand side). Then, Theorem \ref{theorem above-learning-2} in conjunction with Proposition \ref{th:bounds4asympt} give, with probability $1-\delta$ with respect to the draw of the training set, that
\begin{eqnarray}
	\lefteqn{ \risk\big(\algo(\bS)\big) } \nonumber \\
	& \leq & \hat{\risk}\big(\algo(\bS)\big) + \frac{|\co(\bS)|}{N} \nonumber \\
	& & + 2 \frac{\sqrt{\hat{\risk}\big(\algo(\bS)\big) N + |\co(\bS)|+1}}{N}\left( \sqrt{\ln\left(\hat{\risk}\big(\algo(\bS)\big) N + |\co(\bS)|+1\right)} + 4 \right) \nonumber \\
	& & + 2 \frac{ \sqrt{\hat{\risk}\big(\algo(\bS)\big) N + |\co(\bS)|+1}\sqrt{\ln \frac{1}{\delta}} }{N}  + \frac{\ln \frac{1}{\delta}}{N} \nonumber \\
	& = & \hat{\risk}\big(\algo(\bS)\big) + \frac{|\co(\bS)|}{N} + \nonumber \\
	& & \bigO{ \left(\frac{\sqrt{\hat{\risk}\big(\algo(\bS)\big)}}{\sqrt{N}} + \frac{\sqrt{|\co(\bS)|+1}}{N}\right) \left( \sqrt{\ln\left(\hat{\risk}\big(\algo(\bS)\big) N + |\co(\bS)|+1\right)} + 1 \right)} \quad \label{bigO_scenario}
\end{eqnarray}
(since our use of the notation $O(\cdot)$ is not standard, we clarify that here and in \eqref{bigO_others} $g(N,\co(\bS),\hat{\risk}(\algo(\bS))) = O(f(N,\co(\bS),\hat{\risk}(\algo(\bS))))$ means that there exist a constant $C$ and an $\bar{N}$ such that $g(N,\co(\bS),\hat{\risk}(\algo(\bS))) \leq C f(N,\co(\bS),\hat{\risk}(\algo(\bS)))$ for all $\co(\bS) \in \{0,\ldots,N\}$ and $\hat{\risk}(\algo(\bS)) \in [0,1]$ when $N \geq \bar{N}$). This last expression can be compared with the best available result in the literature given by Theorem 17 in \cite{HannekeKontorovich_2021}, which yields
\begin{equation} \label{bigO_others}
\risk\big(\algo(\bS)\big) =  \hat{\risk}\big(\algo(\bS)\big) + \bigO{ \frac{|\co(\bS)|+1}{N} +  \frac{\sqrt{\hat{\risk}\big(\algo(\bS)\big)}\sqrt{|\co(\bS)|+1}}{\sqrt{N}} }.
\end{equation}
It stands out that in \eqref{bigO_scenario} term $\sqrt{\hat{\risk}\big(\algo(\bS)\big)}$ does not multiply $\sqrt{|\co(\bS)|+1}$, as it instead does in \eqref{bigO_others}; moreover, $\sqrt{|\co(\bS)|+1}$ is divided by $N$ instead of $\sqrt{N}$. As a consequence, it is easy to show that, if, e.g., $\hat{\risk}\big(\algo(\bS)\big)$ is replaced by a constant (so as to accommodate typical non-consistent frameworks), then the rate provided by \eqref{bigO_scenario} outdoes that of \eqref{bigO_others} whenever $|\co(\bS)|$ grows sub-linearly and faster than $\ln(N)$ (when $|\co(\bS)|$ is slower than $\ln(N)$, the dominant term in \eqref{bigO_scenario} is of the type $O(\sqrt{\ln(N)} / \sqrt{N})$, whereas in \eqref{bigO_others} it is of the type $O(\sqrt{|\co(\bS)|} / \sqrt{N})$). For instance, replacing $|\co(\bS)|$ with $\sqrt{N}$ gives the rate $\sqrt{\ln(N)} / \sqrt{N}$ in \eqref{bigO_scenario} and the rate $1 / N^{\frac{1}{4}}$ in \eqref{bigO_others}.\footnote{Interestingly enough, both these rates violate the lower bound established in \cite{HannekeKontorovich:2019b} for generic \emph{non-preferent} compression schemes, which shows that the property of preference is strictly necessary to establish accelerated convergence rates for the excess risk as discussed here.} \\

To close, we finally mention an interesting implication of our result that has been kindly suggested to us by an anonymous reviewer. Suppose that in a binary classification problem a hypothesis is selected from a class $\calH$ via a compression scheme that minimizes the empirical risk (such compression scheme is named ``\emph{agnostic sample compression scheme for $\calH$}'' in \citealp{DavidMoran_etal:2016b}, Section 2; see also \citealp{DavidMoran_etal:2016a}), and that this compression scheme is also \emph{preferent}. Then, \eqref{bigO_scenario} and Lemma 3.2 in \cite{DavidMoran_etal:2016b}  give (with high probability with respect to the draw of the training set) that 
\begin{equation} \label{eq:implication_scen}
	\risk\big(\algo(\bS)\big) =  \inf_{h \in \calH} \risk\big(h\big) + \frac{|\co(\bS)|}{N} + \bigO{ \frac{\sqrt{\ln(N)}}{\sqrt{N}} }.
\end{equation}
On the other hand, Theorem 5.2 in \cite{AnthonyBartlett} establishes a bound to the rate at which any hypotheses class $\calH$ of Vapnik-Chernovenkis dimension $d$ can be learned:
\begin{equation} \label{eq:necessary_learning_rate}
	\risk\big(\algo(\bS)\big) -  \inf_{h \in \calH} \risk\big(h\big) \geq \sqrt{\frac{d}{320 \cdot N}}.
\end{equation}
Considering classes $\calH$ whose Vapnik-Chernovenkis dimension increases with $N$ more than $\ln(N)$, results \eqref{eq:implication_scen} and \eqref{eq:necessary_learning_rate} imply that these classes cannot admit \emph{preferent} ``\emph{agnostic sample compression schemes for $\calH$}'' of a size that increases at a rate less than $\sqrt{d}\sqrt{N}$. This result is in contrast with the case of \emph{non-preferent} ``\emph{agnostic sample compression scheme for $\calH$}'', in which context \cite{DavidMoran_etal:2016b} shows that schemes of smaller cardinality can be found.

\section{Application to known learning schemes}
\label{section-application}

The theory developed in Section~\ref{section-learning} is here applied to well-known learning techniques: some algorithms within the family of support vector methods and then a more recent classification technique called Guaranteed Error Machine (GEM). In all these cases, Theorem~\ref{theorem above-learning} applies without restrictions, while the use of Theorem~\ref{theorem above below-learning} requires some conditions on the probability distribution of the examples. The content of this section is also meant to illustrate the flexibility and usefulness of the theory of this paper and, in this light, additional schemes that are amenable to be analyzed within the framework of Section~\ref{section-learning} are hinted upon at the end of the section.

\subsection{Support Vector methods} \label{section-SV}

\subsubsection{Support Vector Machines} \label{section-SVM}

In supervised binary classification, an example $z$ is a pair $z := (x,y)$, where $x \in \calX$ (think of $\calX$ as a generic space without any specific structure) is an ``instance'' and $y \in \{-1,1\}$ is a ``label''. A hypothesis, called a ``binary classifier'', is a map $h: \calX \to \{-1,1\}$. We let $\loss(h,(x,y)) = \One{y \neq h(x)}$, so that the loss function equals $0$ when $y=h(x)$ (correct classification) and $1$ when $y \neq h(x)$ (misclassification). \\

To add flexibility, support vector methods are often got to operate in a feature space. Let $\fmap:\calX \to \calH$ be a ``feature map'' from $\calX$ into a Hilbert space $\calH$ equipped with an inner product $\langle \cdot,\cdot \rangle$. Support Vector Machine (SVM, see \citealp{CortesVapnik1995,SchSmola_BOOK}) is a learning algorithm $\algo_{\SVC}$ that maps a training set $S = \ms((x_1,y_1),\ldots,(x_n,y_n))$ into a binary classifier according to the formula 
$$
\algo_{\SVC}(S)(x) = 
\begin{cases}
	1, & \text{if } \langle w^\ast(S),\fmap(x) \rangle + b^\ast(S) \geq 0 \\
	-1, & \text{if } \langle w^\ast(S),\fmap(x) \rangle + b^\ast(S) < 0, 
\end{cases} 
$$ 
where $x$ is a generic instance and $(w^\ast(S),b^\ast(S))$ (along with the auxiliary variables $\xi_i^\ast(S)$ that are used to relax the constraint of exact classification of all the examples in the training set) is the solution to the optimization program 
\begin{eqnarray} 
\label{eq:svc_pb} 	
\calP_{\SVC}(S): & \displaystyle \min_{w \in \calH, b \in \R \atop \xi_i \geq 0, i=1,\ldots,n} & \quad  \displaystyle \| w \|^2 + \rho \sum_{i=1}^n \xi_i \\
& \displaystyle \textrm{\rm subject to:} & \displaystyle\quad 1 - y_i (\langle w,\fmap(x_i) \rangle + b) \leq \xi_i, \ \ i = 1, \ldots,n. \nonumber 
\end{eqnarray}
As shown in Theorem 2 in \cite{BurgesCrisp:99}, \eqref{eq:svc_pb} always admits a minimizer and, moreover, $w^\ast(S)$ is unique; however, $b^\ast(S)$ (and $\xi^\ast_i(S)$) need not be unique. When $b^\ast(S)$ is not unique, we assume that the tie is broken by selecting the value of $b$ that minimizes $|b|$.\footnote{This certainly breaks the tie because, if the smallest absolute value were achieved by two values for $b^\ast$, say $b^\ast = \pm \bar{b}$, corresponding to the solutions $(w^\ast,\bar{b},\xi_{i,1}^\ast)$ and $(w^\ast,-\bar{b},\xi_{i,2}^\ast)$ (recall that $w^\ast$ must be the same at optimum), then the optimality of these two solutions would imply that $\sum_{i=1}^n \xi_{i,1}^\ast = \sum_{i=1}^n \xi_{i,2}^\ast$ and therefore the solution half way between $(w^\ast,\bar{b},\xi_{i,1}^\ast)$ and $(w^\ast,-\bar{b},\xi_{i,2}^\ast)$, i.e., $(w^\ast,0,0.5 \cdot \xi_{i,1}^\ast + 0.5 \cdot \xi_{i,2}^\ast)$, would be feasible thanks to convexity, it would achieve the same cost as the other two solutions, but it would be preferred because it carries a smaller value of $|b|$.} Note that once $w^\ast(S)$ and $b^\ast(S)$ are uniquely determined, then also the optimal values $\xi^\ast_i(S)$ remain univocally identified. \\ 

As is well known, see e.g. \cite{SchSmola_BOOK}, the feature map $\fmap(\cdot)$ and the inner product $\langle \cdot,\cdot \rangle$ need not be explicitly assigned to solve \eqref{eq:svc_pb}. The reason is that the determination of the solution to \eqref{eq:svc_pb} (as well as the evaluation of $\algo_{\SVC}(S)(x)$) only involves the computation of inner products of the type $\langle \fmap(x_i),\fmap(x_j) \rangle$ and $\langle \fmap(x_i),\fmap(x) \rangle$. These inner products can indeed be directly obtained from a kernel $k(\cdot,\cdot)$ ($k(\cdot,\cdot)$ is a function $\calX \times \calX \to \R$ that satisfies suitable conditions of positive definiteness, see \citealp{SchSmola_BOOK}) and the theory of Reproducing Kernel Hilbert Spaces ensures that any choice of $k(\cdot,\cdot)$  always corresponds to allocate a suitable pair $\fmap(\cdot)$, $\langle \cdot,\cdot \rangle$ so that $k(\cdot,\cdot) = \langle \fmap(\cdot),\fmap(\cdot) \rangle$ (this is the so-called ``kernel trick''). We also note that the solution to \eqref{eq:svc_pb} is typically obtained by solving a program that is the dual of \eqref{eq:svc_pb}. However, although important for the practice of SVM, all these remarks are immaterial for the discussion that follows. \\

To cast SVM within the framework of the present paper, introduce the following compression function $\co_{\SVC}$. First, endow $\calX$ with an arbitrary total ordering (to be used below in point (ii)). For any multiset $S$, define $\co_{\SVC}(S) = S_+ \cup S_0$ where:
\begin{itemize}
	\item[(i)] $S_+$ is the sub-multiset of all the examples (repeated as many times as they appear in $S$) for which $1 - y_i (\langle w^\ast(S),\fmap(x_i) \rangle + b^\ast(S)) > 0$ (note that this is equivalent to the condition $\xi^\ast_i(S) > 0$);
	\item[(ii)] consider the sub-multisets $\tilde{S}_0$ of smallest cardinality for which: 
	\begin{itemize}
		\item[(a)] $1 - y_i (\langle w^\ast(S),\fmap(x_i) \rangle + b^\ast(S)) = 0$; and,
		\item[(b)] the couple $(w^\ast(S_+ \cup \tilde{S}_0),b^\ast(S_+ \cup \tilde{S}_0))$ given by \eqref{eq:svc_pb} with $S_+ \cup \tilde{S}_0$ in place of the original $S$ is the same as $(w^\ast(S),b^\ast(S))$.\footnote{The sub-multiset of $S$ formed by all examples that satisfy relation $1 - y_i (\langle w^\ast(S),\fmap(x_i) \rangle + b^\ast(S)) \geq 0$ certainly gives the same couple $(w^\ast(S),b^\ast(S))$ obtained when \eqref{eq:svc_pb} is applied to the whole multiset $S$, hence a multiset of smallest cardinality satisfying (a) and (b) certainly exists.}
	\end{itemize}
	Among the multisets $\tilde{S}_0$, single out $S_0$ by using the ordering on $\calX$: for each candidate multiset $\tilde{S}_0$, identify the element with smallest instance and pick the multiset that exhibits the smallest among all; if a tie remains, move on to compare the second smallest and so on until $S_0$ is uniquely determined. 
\end{itemize}

We want to apply Theorem \ref{theorem above-learning} to SVM. To this purpose, we start by verifying that $\co_{\SVC}$ satisfies the \emph{preference} Property \ref{preference}. \\
\\
$\diamond$ {\it Preference.} We apply Lemma \ref{lemma_fund}, which requires to show that, for every multisets $S$ and $S'$ such that $\co_{\SVC}(S) \subseteq S' \subseteq S$, it holds that $\co_{\SVC}(S') = \co_{\SVC}(S)$. \\
First, we show that 
\begin{equation}
\label{S-S'}
(w^\ast(S'),b^\ast(S')) = (w^\ast(S),b^\ast(S)). 
\end{equation}
For the sake of contradiction, suppose that \eqref{S-S'} does not hold: 
\begin{equation}
\label{S-S'-contradiction}
(w^\ast(S'),b^\ast(S')) \neq (w^\ast(S),b^\ast(S)). 
\end{equation}
Note that the value achieved by $(w^\ast(S'),b^\ast(S'))$ for the problem that only contains the constraints corresponding to the examples in $\co_{\SVC}(S)$\footnote{This amounts to substitute $(w^\ast(S'),b^\ast(S'))$ in the problem of the type \eqref{eq:svc_pb} where only the constraints corresponding to the examples in $\co_{\SVC}(S)$ are enforced, and then optimize with respect to the variables~$\xi_i$.} cannot be worse than the value achieved by $(w^\ast(S'),b^\ast(S'))$ for the problem containing the constraints associated to $S'$ (because the latter is more constrained than the former); moreover, the value achieved by $(w^\ast(S),b^\ast(S))$ for the problem containing only the constraints associated to $\co_{\SVC}(S)$ is equal to the value achieved by $(w^\ast(S),b^\ast(S))$ for the problem containing the constraints associated to $S'$ (because the examples in $S'$ that are not in $\co_{\SVC}(S)$ corresponds to $\xi_i^\ast(S)$ whose  value is zero). From \eqref{S-S'-contradiction}, then, one concludes that $(w^\ast(S'),b^\ast(S'))$ must be preferred to $(w^\ast(S),b^\ast(S))$ for the problem containing only the constraints associated to $\co_{\SVC}(S)$. This, however, leads to a contradiction because, by construction $(w^\ast(\co_{\SVC}(S)),b^\ast(\co_{\SVC}(S))) = (w^\ast(S),b^\ast(S))$. Hence, \eqref{S-S'} remains proven. \\
Consider now $\co_{\SVC}(S') = S'_+ \cup S'_0$, where $S'_+$ and $S'_0$ are obtained from (i) and (ii) applied to $S'$. Since $S' \supseteq \co_{\SVC}(S)$,  $S'$ contains all the examples in $S$ for which $1 - y_i (\langle w^\ast(S),\fmap(x_i) \rangle + b^\ast(S)) > 0$ and, in view of \eqref{S-S'}, these examples are also all those in $S'$ for which $1 - y_i (\langle w^\ast(S'),\fmap(x_i) \rangle + b^\ast(S')) > 0$. Thus, $S'_+ = S_+$. The fact that $S'_0 = S_0$ follows instead from observing that $S_0$ is the preferred selection according to the ordering procedure in (ii) to recover $(w^\ast(S),b^\ast(S))$ and, since $(w^\ast(S'),b^\ast(S')) = (w^\ast(S),b^\ast(S))$ and $S_0$ is available as a sub-multiset of $S'$, this same $S_0$ is selected when (ii) is applied to $S'$, leading to $S'_0 = S_0$. This establishes the validity of Lemma \ref{lemma_fund} and closes the argument. \qed \\

Next, we verify the \emph{coherence -- part I} Property \ref{mis_implies_change}. \\
\\
$\diamond$ {\it Coherence -- part I.}
Notice first that the very definition of $\co_{\SVC}$ implies that $\algo_{\SVC}(\co_{\SVC}(S)) \break = \algo_{\SVC}(S)$, i.e., $\algo_{\SVC}$ itself acts as a reconstruction function. Moreover, if we have that $\loss(\algo_{\SVC}(S),(x_i, y_i)) = 1$ for an example $(x_i,y_i)$ in $S$, then it must be that $y_i (\langle w^\ast(S),\fmap(x_i) \rangle + b^\ast(S)) \leq 0$ (because $y_i$ has sign opposite to that of $\langle w^\ast(S),\fmap(x_i) \rangle + b^\ast(S)$). This implies that $1 - y_i (\langle w^\ast(S),\fmap(x_i) \rangle + b^\ast(S)) > 0$, from which the \emph{inclusion} Property \ref{inclusion} follows because $\co_{\SVC}(S)$ includes all examples for which this latter inequality holds true. Based on these results, the \emph{coherence -- part I} Property \ref{mis_implies_change} follows from Lemma \ref{lemma-violation-gives-change}. \qed \\

Having established the \emph{preference} and the \emph{coherence - part I} properties, the following theorem follows as a corollary of Theorem \ref{theorem above-learning}. 
\begin{theorem}{\bf (Risk of SVM)}
\label{theorem-risk-SVM}
For any $\delta \in (0,1)$, it holds that
\[
\Pr\Big\{ \risk\big(\algo_{\SVC}(\bS)\big) > \eps_\bk \Big\} \le \delta,
\]
where $\bk = |\co_{\SVC}(\bS)|$ and $\eps_k$ is given in \eqref{epsilonk}. \qed
\end{theorem}
We are instead not in a position to establish lower bounds for $\risk\big(\algo_{\SVC}(\bS)\big)$. The reason is that adding a new example $(x_{n+1},y_{n+1})$ for which $1 - y_{n+1} (\langle w^\ast(S),\fmap(x_{n+1}) \rangle + b^\ast(S)) > 0$ changes the compression (refer to (i) in the definition of $\co_{\SVC}$), but this does not exclude that $y_{n+1} (\langle w^\ast(S),\fmap(x_{n+1}) \rangle + b^\ast(S)) > 0$, in which case the new example is not misclassified. This fact prevents the \emph{coherence -- part II} Property \ref{chance implies mis} from being satisfied. 

\begin{remark}{\bf (A computational aspect)} \label{rmk:SVC_computational_issues}
To evaluate an upper bound to the risk according to Theorem \ref{theorem-risk-SVM}, one needs to compute $|\co_{\SVC}(S)|$. Computing the cardinality of $S_+$ is easy; determining the cardinality of $S_0$, however, is more computationally demanding. On the other hand, $\eps_k$ is increasing with $k$ so that a valid result can be easily found by overestimating $|\co_{\SVC}(S)|$ with the cardinality of the set of examples for which $1 - y_i (\langle w^\ast(S),\fmap(x_i) \rangle + b^\ast(S)) \geq 0$. Heuristically, this evaluation often turns out to be sharp. \qed

\end{remark} 

\subsubsection{Support Vector Regression} \label{section-SVR}

In regression problems, an example $z$ is a pair $z:=(x,y)$ with $x \in \calX$, a generic space, and $y \in \R$. A hypothesis $h$ is called a ``predictor'' and it is a map from $\calX$ to $\R$. As loss function, we take 
$$
\loss(h,(x,y)) = \begin{cases}
	1, & \text{if } | y - h(x)| > t \\
	0, & \text{if } | y - h(x)| \leq t,
\end{cases}
$$
where $t$ is the so-called ``prediction tolerance''. Let $\fmap:\calX \to \calH$ be a feature map from $\calX$ to a Hilbert space $\calH$ endowed with inner product $\langle \cdot,\cdot \rangle$. Support Vector Regression (SVR, see \citealp{SmolaScholkopf2004}) is a learning algorithm $\algo_{\SVR}(S)$ that maps any training set $S = \ms((x_1,y_1),\ldots,(x_n,y_n))$ into the predictor 
$$
\algo_{\SVR}(S)(x) = \langle w^\ast(S),\fmap(x) \rangle + b^\ast(S)
$$
where $w^\ast(S)$ and $b^\ast(S)$ (along with $\xi_i^\ast(S)$) are the solution to the program 
\begin{eqnarray} 
\label{eq:svr_pb} 	
\calP_{\SVR}(S): & \displaystyle \min_{w \in \calH, b \in \R \atop \xi_i \geq 0, i=1,\ldots,n} & \displaystyle \quad \| w \|^2 + \rho \sum_{i=1}^n \xi_i \nonumber \\
& \displaystyle \textrm{\rm subject to:} & \displaystyle \quad | y_i - \langle w,\fmap(x_i) \rangle - b | \leq t + \xi_i, \ \ i = 1, \ldots,n. \nonumber 
\end{eqnarray}
Similarly to SVM, a minimizer certainly exists and $w^\ast(S)$ is also unique, but $b^\ast(S)$ may not be, see \cite{BurgesCrisp:99}. In the latter case, the tie is broken by choosing the minimizer with the smallest value of $|b|$. After that $w^\ast(S)$ and $b^\ast(S)$ have been made unique, also the values of $\xi_i^\ast(S)$ remain univocally determined. The kernel trick described for SVM applies here as well, so that $\fmap(\cdot)$ and $\langle \cdot,\cdot \rangle$ need not be specified explicitly and can be assigned via a kernel function. \\

Our definition of a compression function $\co_{\SVR}$ closely resembles that for SVM. For any multiset of examples $S$, define $\co_{\SVR}(S) = S_+ \cup S_0$ where:
\begin{itemize}
	\item[(i)] $S_+$ is the multiset of all the examples (repeated as many times as they appear in $S$) for which $| y_i - \langle w^\ast(S),\fmap(x_i) \rangle - b^\ast(S) | > t$ (note that this is equivalent to the condition $\xi^\ast_i(S) > 0$);
	\item[(ii)] $S_0$ is the smallest sub-multiset only containing examples that satisfy condition $| y_i - \langle w^\ast(S),\fmap(x_i) \rangle - b^\ast(S) | = t$ for which $w^\ast(S_+ \cup S_0) = w^\ast(S)$ and $b^\ast(S_+ \cup S_0) = b^\ast(S)$. In complete analogy with SVM, $S_0$ as defined before may not be unique, in which case an $S_0$ is singled out by a total ordering on $\calX$. 
\end{itemize}

The proof that $\co_{\SVR}$ satisfies the \emph{preference} Property \ref{preference} and that $(\algo_{\SVR},\co_{\SVR})$ satisfies the \emph{coherence -- part I} Property \ref{mis_implies_change} follows the same path, \emph{mutatis mutandis}, as for SVM and is therefore omitted. This gives the following theorem, obtained as a direct consequence of Theorem \ref{theorem above-learning}. 
\begin{theorem}{\bf (Risk of SVR)}
	\label{theorem-risk-SVR}
	For any $\delta \in (0,1)$, it holds that
	\[
	\Pr\Big\{ \risk\big(\algo_{\SVR}(\bS)\big) > \eps_\bk \Big\} \le \delta,
	\]
	where $\bk = |\co_{\SVR}(\bS)|$ and $\eps_k$ is given in \eqref{epsilonk}. \qed
\end{theorem}
Unlike SVM, for SVR lower and upper bounds for the risk are established under an additional, mild, distributional assumption. 
\begin{assumption} 
\label{asmpt:SVR_density}
The regular conditional distribution of $\by$ given $\bx$ has no concentrated mass almost surely. \qed
\end{assumption}
To establish the lower and upper bounds, we resort to Theorem \ref{theorem above below-learning}. Preliminarily, we show the validity of the assumptions of this theorem. \\ 
\\
$\diamond$ {\it Non-associativity.} Consider any training set $S = \ms((x_1,y_1),\ldots,(x_n,y_n))$ and an additional multiset of examples $S' = \ms((x_{n+1},y_{n+1}),\ldots,(x_{n+p},y_{n+p}))$ such that $\co(S,(x_{n+i}, \break y_{n+i})) = \co(S)$ for all $i \in \{1,\ldots,p\}$. Further, assume that $|y_{n+i} - \algo_{\SVR}(S)(x_{n+i})| \neq t$ for all $i \in \{1,\ldots,p\}$. Then, it must be that 
$| y_{n+i} - \algo_{\SVR}(S)(x_{n+i})| < t$ for all $i \in \{1,\ldots,p\}$. Indeed, suppose by contradiction the opposite: $| y_{n+i} - \algo_{\SVR}(S)(x_{n+i})| > t$ for some $i$. Then, if  $\algo_{\SVR}(S,(x_{n+i},y_{n+i})) = \algo_{\SVR}(S)$, then $| y_{n+i} - \algo_{\SVR}(S,(x_{n+i},y_{n+i}))(x_{n+i})| > t$ and $(x_{n+i},y_{n+i})$ is counted in $\co_{\SVR}(S,(x_{n+i},y_{n+i}))$ leading to $\co_{\SVR}(S,(x_{n+i},y_{n+i})) \neq \co_{\SVR}(S)$; if instead $\algo_{\SVR}(S,(x_{n+i},y_{n+i})) \neq \algo_{\SVR}(S)$, then $\algo_{\SVR}(\co_{\SVR}(S,(x_{n+i},y_{n+i}))) \neq \algo_{\SVR}(\co_{\SVR}(S))$, which means that $\co_{\SVR}(S,(x_{n+i},y_{n+i}))$ cannot be the same as $\co_{\SVR}(S)$. Thus, it remains proven that $| y_{n+i} - \algo_{\SVR}(S)(x_{n+i})| < t$ for all $i \in \{1,\ldots,p\}$ and this yields immediately that $\algo_{\SVR}(S \cup S') = \algo_{\SVR}(S)$ and that $\co_{\SVR}(S \cup S') = \co_{\SVR}(S)$. This, along with the fact that $|\by_{n+i} - \algo_{\SVR}(\bS)(\bx_{n+i})| \neq t$ for all $i \in \{1,\ldots,p\}$ holds with probability $1$ in view of Assumption \ref{asmpt:SVR_density}, gives the \emph{non-associativity} property. \qed \\
\\
$\diamond$ {\it Non-concentrated mass.} This is obvious in view of Assumption \ref{asmpt:SVR_density}. \qed \\
\\
$\diamond$ {\it Coherence -- part II.} The proof is by contrapositive: letting $\bS = \ms((\bx_1,\by_1),\ldots,(\bx_n,\by_n))$ we show that
\begin{equation} \label{coherence2-contrapositive}
\Pr \Big( \{ \loss(\algo_{\SVR}(\bS),(\bx_{n+1},\by_{n+1})) = 0 \} \setminus \{ \co_{\SVR}(\co_{\SVR}(\bS),(\bx_{n+1},\by_{n+1})) = \co_{\SVR}(\bS) \} \Big) = 0.
\end{equation}
Assumption \ref{asmpt:SVR_density} implies that the case $|y_{n+1} - \algo_{\SVR}(S)(x_{n+1})| = t$ can be disregarded because it correpsonds to an event that has probability zero. Moreover, when $| y_{n+1} - \algo_{\SVR}( S)( x_{n+1})| > t$, we have that $\loss(\algo_{\SVR}( S),( x_{n+1}, y_{n+1})) = 1$, so that this case can be disregarded too. When instead $| y_{n+1} - \algo_{\SVR}( S)( x_{n+1})| < t$, it holds that $\loss\big(\algo_{\SVR}( S),(x_{n+1}, \break y_{n+1})\big) = 0$. In this case, 
$$
| y_{n+1} - \algo_{\SVR}(\co_{\SVR}( S))( x_{n+1})| = | y_{n+1} - \algo_{\SVR}( S)( x_{n+1})| < t
$$
yields $\algo_{\SVR}(\co_{\SVR}( S),( x_{n+1}, y_{n+1})) = \algo_{\SVR}( S)$, from which 
$$
| y_{n+1} - \algo_{\SVR}(\co_{\SVR}( S),( x_{n+1}, y_{n+1}))( x_{n+1})| < t.
$$
Hence, $( x_{n+1}, y_{n+1})$ is not in $\co_{\SVR}(\co_{\SVR}( S),( x_{n+1}, y_{n+1}))$, so that, owing to the \emph{preference} Property \ref{preference}, it must be that  $\co_{\SVR}(\co_{\SVR}( S),( x_{n+1}, y_{n+1})) = \co_{\SVR}(S)$. This shows the validity of \eqref{coherence2-contrapositive}.  \qed \\

Using Theorem \ref{theorem above below-learning}, we now have the following result. 
\begin{theorem}{\bf (Risk of SVR - bounds from below and from above)}
	\label{theorem-risk-SVM-ba}
	Under Assumption \ref{asmpt:SVR_density}, for any $\delta \in (0,1)$, it holds that
	$$
	\Pr\Big\{ \underline{\eps}_\bk \leq \risk\big(\algo_{\SVR}(\bS)\big) \leq \overline{\eps}_\bk  \Big\} \geq 1-\delta,
	$$
	where $\bk = |\co_{\SVR}(\bS)|$ and $\underline{\eps}_k$, $\overline{\eps}_k$ are given respectively in \eqref{underline_epsilonk}, \eqref{overline_epsilonk}. \qed 
\end{theorem}

\subsubsection{Other Support Vector methods}

The applicability of Theorems \ref{theorem above-learning} and \ref{theorem above below-learning} can be carried over to other Support Vector methods. SVR with Adjustable Size, \cite{ScBaSmWi:98}, requires minor modifications. Suitable but conceptually straightforward modifications of the arguments used in this section can also be applied to one-class SVM, \cite{ScWiSmSTPl:99}, and Support Vector Data Description (SVDD), \cite{TaxDuin2004} and \cite{WangChungWang2011}. In these latter two cases, the setup changes slightly since examples are unlabeled and hypotheses are regions in the set hosting the examples. To apply Theorem \ref{theorem above below-learning}, one needs here to modify Assumption \ref{asmpt:SVR_density} so as to enforce specific non-accumulation conditions for the method at hand. 

\subsection{Guaranteed Error Machine} \label{section-GEM}

The Guaranteed Error Machine (GEM) is a learning algorithm for classification that was first introduced in \cite{Campi2010} and then further developed in \cite{CareRamponiCampi2018}.\footnote{The algorithm described here is a variant of those proposed in the referenced papers.} GEM returns a ternary-valued classifier, which is also allowed to abstain from classifying in case of doubt. To be specific, letting $z = (x,y)$ with $x \in \calX$, a generic set, and $y \in \{-1,1\}$ (this is the same setup as in SVM), a hypothesis $h$ is here a map from $\calX$ to $\{-1,1,0\}$, where the value $0$ is interpreted as admission of being unable to classify. Issuing an incorrect label ($-1$ in place of $1$ or \emph{vice versa}) leads to a mistake, and the theory aims at bounding the probability for this to happen. Correspondingly, the loss function is defined as follows: 
$$
\loss(h,(x,y)) = 
\begin{cases}
	1, & \text{if } |y - h(x)| = 2 \\
	0, & \text{if } |y - h(x)| = 0 \text{ or } 1.
\end{cases}
$$
To describe the operation of GEM, start by introducing a feature map $\fmap:\calX \to \calH$, where $\calH$ is a Hilbert space with inner product $\langle \cdot,\cdot \rangle$ (as for support vector methods, $\fmap$, $\calH$, and $\langle \cdot,\cdot \rangle$ need not be explicitly given and can be implicitly defined by means of a kernel) and also assume the existence of an ordering on $\calX$ (used later to introduce a tie-break rule). GEM requires that the user chooses an integer $d \geq 1$, which specifies the maximal cardinality for the compression.\footnote{Selecting a large value for $d$ reduces the chance of abstention from classifying. When $d$ is larger than the cardinality of the training set, the set of abstention becomes empty. In other cases, the user tries to achieve a good compromise between the probability of abstention and the probability of making an error. This is not specific to GEM and applies to any technique for the construction of ternary-valued classifiers. See \cite{Campi2010} for more discussion on this point.} In loose terms, GEM operates as follows. It is assumed that one has an additional observation $(\bar{x},\bar{y})$ (besides the training set $S = \ms((x_1,y_1),\ldots,(x_n,y_n))$) that acts as initial ``center''. GEM constructs the hyper-sphere in $\calH$ around $\fmap(\bar{x})$ which is the largest possible under the condition that the hyper--sphere does not include any $\fmap(x_i)$ with label $y_i$ different from $\bar{y}$. All points inside this hyper-sphere are classified as the label $\bar{y}$, and all examples $(x_i,y_i)$ for which $\fmap(x_i)$ is inside the hyper-sphere are removed from the training set. The example that lies on the boundary of the hyper-sphere (and that has therefore prevented the hyper--sphere from further enlarging) is then appointed as the new center (in case of ties, the tie is broken by using the ordering on $\calX$) and the procedure is repeated by constructing another hyper-sphere around the new center. This time, only the region given by the difference between the newly constructed hyper-sphere and the first hyper--sphere (which has been already classified) is classified as the label of the second center. This procedure continues the same way and comes to a stop when either the whole space has been classified or the total number of centers is equal to $d$, in which case the portion of $\calX$ that has not been covered is classified as $0$. This leads to the algorithm formally described below.  
\begin{itemize}
	\item[STEP 0.] SET $q:=0$, $P:= S\cup \ms((\bar{x},\bar{y}))$, $C = \emptyset$ and $x_C = \bar{x}$, $y_C = \bar{y}$; 
	\item[STEP 1.] SET $q := q + 1$ and SOLVE problem
	\begin{eqnarray*} 
		\displaystyle \max_{r \geq 0} & & r \\
		\textrm{\rm subject to:} & & \| \fmap(x_i) - \fmap(x_C) \| \geq r, \text{ for all } (x_i,y_i) \in P \text{ such that } y_i \neq y_C. \nonumber
	\end{eqnarray*}
	Let $r^\ast$ be the optimal solution (note that $r^\ast$ can possibly be $+\infty$);
	\item[STEP 2.] FORM the region $\calR_q := \{ x \in \calX : \; \| \fmap(x) - \fmap(x_C) \| < r^\ast \}$ and LET $\ell_q := y_C$; UPDATE $P$ as follows: if $r^\ast > 0$, then remove from $P$ all the examples with $x_i \in \calR_q$; if instead $r^\ast = 0$,\footnote{$r^\ast = 0$ only happens if there are examples with different labels whose instance is $x_C$.} then remove from $P$ the example $(x_C,y_C)$; 
	\item[STEP 3.] IF $r^\ast < +\infty$, THEN 
	\begin{itemize}
		\item[3.a] SET $C := C \cup \ms((x_{i^\ast},y_{i^\ast}))$, where $(x_{i^\ast},y_{i^\ast})$ is an example in $P$ such that: a.  $\| \fmap(x_{i^\ast}) - \fmap(x_C) \| = r^\ast$; b. $y_{i^\ast} \neq y_C$; c. $x_{i^\ast}$ is smallest in the ordering of $\calX$ among all the examples satisfying a. and b.;
		\item[3.b] SET $(x_C,y_C) := (x_{i^\ast},y_{i^\ast})$;
	\end{itemize}	  
	\item[STEP  4.] IF either $|C| = d$ or $P = \emptyset$ THEN STOP and RETURN $\ell_j$, $\calR_j$, $j=1,\ldots,q$ and $C$; \\
	ELSE, GO TO 1.
\end{itemize}

The GEM predictor is defined as 
$$
\algo_{\GEM}(S)(x) = \begin{cases}
	0, & \quad \text{if } x \notin \calR_j \; \forall j=1,\ldots,q; \\
	\ell_{j^\ast} & \quad \text{otherwise, with } j^\ast = \min \big\{j \in \{1,\ldots,q\}: \;  x \in \calR_j \big\}.
\end{cases}
$$
The compression function for GEM is $\co_{\GEM}(S) = C$. \\

We next establish the \emph{preference} and \emph{coherence -- part I} properties, required to apply Theorem \ref{theorem above-learning}. \\
\\
$\diamond$  {\it Preference.} Given any multisets $S$ and $S'$ such that $\co_{\GEM}(S) \subseteq S' \subseteq S$, it is easy to verify that running STEPS 0-4 with $S'$ as input returns the same output as when these steps are run with input $S$. In particular, $\co_{\GEM}(S') = C = \co_{\GEM}(S)$ and, therefore, the \emph{preference} property follows by an application of Lemma \ref{lemma_fund}. \qed \\
\\
$\diamond$  {\it Coherence -- part I.} Since applying STEPS 0-4 to $\co_{\GEM}(S)$ returns the same output as when they are applied to $S$, $\algo_{\GEM}$ itself acts as a reconstruction function. Also, the \emph{inclusion} Property \ref{inclusion} is immediately verified (just pay a bit of care to the case in which more examples  with different labels corresponds to the same $x_C$). The \emph{coherence -- part I} Property \ref{mis_implies_change} then follows from Lemma \ref{lemma-violation-gives-change}. \qed \\
\\
Applying Theorem \ref{theorem above-learning} we now have the following result. 
\begin{theorem}{\bf (Risk of GEM)}
	\label{theorem-risk-GEM}
	For any $\delta \in (0,1)$, it holds that
	\[
	\Pr\Big\{ \risk\big(\algo_{\GEM}(\bS)\big) > \eps_\bk \Big\} \le \delta,
	\]
	where $\bk = |\co_{\GEM}(\bS)|$ and $\eps_k$ is given in \eqref{epsilonk}. \qed
\end{theorem}
Notice also that $|\co_{\GEM}(\bS)| \leq d$ holds by construction, which implies that the bound $\risk\big(\algo_{\GEM}(\bS)\big) \leq  \eps_d$ is always correct with high confidence $1 - \delta$.\footnote{A similar result would not be possible without resorting to ternary classifiers that allow for abstention from classifying.}\\

We now turn to lower bounds to the risk, which are established by an application of Theorem \ref{theorem above below-learning}. We start by showing the validity of the \emph{non-associativity} property. \\
\\
$\diamond$ {\it Non-associativity.} Consider any training set $S = \ms((x_1,y_1),\ldots,(x_n,y_n))$ and an additional multiset of examples $S' = \ms((x_{n+1},y_{n+1}),\ldots,(x_{n+p},y_{n+p}))$. Suppose that $\co_{\GEM}(S \cup S') \neq \co_{\GEM}(S)$. For this to be, it is required that at least one of these conditions applies: (i) $\ell(\algo_{\GEM}(S),(x_{n+i},y_{n+i})) = 1$ for some $i \in \{1,\ldots,p\}$; or, (ii) one of the $(x_{n+i},y_{n+i})$, $i \in \{1,\ldots,p\}$, for which $\ell(\algo_{\GEM}(S),(x_{n+i},y_{n+i})) = 0$ lies on the boundary of a $\calR_j$ and is lower in order than the example that is chosen as center by the algorithm applied to $S$. However, take in isolation an example $(x_{n+i},y_{n+i})$ that satisfies either (i) or (ii); then, that example alone makes the compression change. This proves the \emph{non-associativity} property. \qed \\

To move on and prove the \emph{non-concentrated mass} and \emph{coherence -- part II} properties, we need a mild assumption on the distribution of examples. 
\begin{assumption}
	\label{asmpt:GEM_density}
	For any $c \in \calH$ and $\gamma \in \Real{}$, it holds that
	\begin{equation}
		\Pr \{\| \fmap({\bx}) - c \|^2 = \gamma \} = 0. \nonumber 
	\end{equation} \qed
\end{assumption}
\noindent 
$\diamond$ {\it Non-concentrated mass.} This immediately follows from Assumption \ref{asmpt:GEM_density}: if $\Pr\{ \bz = \bar{z} \} \neq 0$ for some $\bar{z} = (\bar{x},\bar{y})$, then Assumption \ref{asmpt:GEM_density} is violated by the choices $c = \fmap(\bar{x})$ and $\gamma = 0$. \qed \\
\\
$\diamond$ {\it Coherence -- part II.} In view of Assumption \ref{asmpt:GEM_density}, $\bx_{n+1}$ lies on the boundary of a region $\calR_j$ with probability zero. On the other hand, when $x_{n+1}$ is not on the boundary, a change of compression only occurs if $(x_{n+1},y_{n+1})$ is misclassified, that is, $\ell(\algo_{\GEM}(S),(x_{n+1},y_{n+1})) = 1$. This proves the \emph{coherence -- part II} property. \qed \\

The following theorem now follows from Theorem \ref{theorem above below-learning}. 
\begin{theorem}{\bf (Risk of GEM - bounds from below and from above)}
	\label{theorem-risk-GEM-ba}
	Under Assumption \ref{asmpt:GEM_density}, for any $\delta \in (0,1)$, it holds that
	$$
	\Pr\Big\{ \underline{\eps}_\bk \leq \risk\big(\algo_{\GEM}(\bS)\big) \leq \overline{\eps}_\bk  \Big\} \geq 1-\delta,
	$$
	where $\bk = |\co_{\GEM}(\bS)|$ and $\underline{\eps}_k$, $\overline{\eps}_k$ are given respectively in \eqref{underline_epsilonk}, \eqref{overline_epsilonk}. \qed 
\end{theorem}

\subsection{Other learning schemes with a \emph{preferent} compression}
	
Besides Support Vector methods and GEM, other learning schemes can be studied within the framework of the present paper. We mention here just two additional examples, without working out all the details as it was done in Sections~\ref{section-SV} and~\ref{section-GEM}. \\
	
One first example is the class of methods for classification based on the nearest-neighbor (NN) algorithm. Consider for simplicity $1$-NN in a	finite Euclidean space $\calX$ and with labels generated by a target concept, see \cite{Shalev-Shwartz_Ben-David_2014}. For every training set $S = \ms((x_1,y_1),\ldots,(x_n,y_n))$ of instance/label pairs, $1$-NN relies on the Voronoi partition of the instance	domain $\calX$ induced by $S$, where cells are $C_i = \{ x \in \calX: \|x-x_i\| \leq \|x-x_j\| \; \forall j=1,\ldots,n \}$, $i=1,\ldots,n$. If one eliminates from $S$ all the examples whose associated cell is in the interior of the region labeled as the cell, then the remaining examples form a compressed multiset for which $1$-NN itself acts as a reconstruction function. It is then a simple enough task to show that such a compression function satisfies the \emph{preference} Property \ref{preference}; moreover, since the $1$-NN classifier is always consistent with $S$, Lemma \ref{lemma-violation-gives-change} allows us to conclude that the \emph{coherence -- part I} Property \ref{mis_implies_change} holds. Thereby, Theorem \ref{theorem above-learning}
can be applied to evaluate the probability of misclassification of the $1$-NN classifier. Similar arguments are expected to be applicable to more general $k$-NN schemes, even in generic (infinite dimensional) metric spaces. Further studies can possibly cover more general NN-based methods along the lines of \cite{KonSabWei2017,KonSabUrn2018,HanKonSabWei2021}. \\
	
As a second example, again in the context of classification, we would like to mention  the \emph{Total Recall} algorithm of \cite{HelSloWar_1990}, which is in use to learn nested differences of concepts from intersection-closed classes. In the setup of \cite{HelSloWar_1990}, no matter whether the depth\footnote{See \cite{HelSloWar_1990}, page	166.} of the hypothesis returned by the algorithm is arbitrary or a-priori fixed, it is fairly easy to prove that the union of the	\emph{spanning sets}\footnote{See \cite{HelSloWar_1990}, page 170.} with minimal cardinality for the multisets of examples used in the various calls to the closure learner by the Total Recall algorithm\footnote{If multiple 	spanning sets with minimal cardinality exist, then a choice is singled-out by means of any total ordering of the finite subests of $\calX$.} defines a compression function that satisfies the \emph{preference} Property \ref{preference}. By the very definition of spanning sets, we also have that the Total Recall algorithm applied to the union of the minimal spanning sets (i.e., applied to the compressed multiset) returns the same hypothesis as  when the algorithm is run on the whole training set. This means that the Total Recall algorithm itself acts as a reconstruction function and Theorem \ref{theorem above-learning-2} can be used to obtain evaluations of the probability of misclassification. More research can extend this analysis to alternative algorithms, e.g., along the lines discussed in Section 5 in \cite{HelSloWar_1990}.

\section{Proofs}
\label{section-proofs}

\subsection{A brief overview of the proofs}
\label{proof-overview}

To help readability, we first trace a roadmap of the fundamental steps in which the rather long proofs of Theorems \ref{th:compression_1} and \ref{th:compression_2} are articulated. To prove Theorem \ref{th:compression_1}, we first establish some properties that have necessarily to be satisfied by any compression scheme that is \emph{preferent}. These are (i) and (ii) in the second page of the proof. Next, the probability $\Pr \{ \bphi_N > \eps_\bk \}$ that appears  in the left-hand side of \eqref{result-th-1} in the statement of Theorem \ref{th:compression_1} is re-written in integral form with respect to suitable measures that are  introduced in \eqref{def-m}, and the resulting expression is minimized under conditions (i) and (ii). The ensuing variational problem \eqref{eq:primal_M} returns an upper bound to $\Pr \{ \bphi_N > \eps_\bk \}$. The next step consists in the evaluation of the optimal value of problem \eqref{eq:primal_M}. This step is accomplished by dualization, leading to the reformulation \eqref{eq:dual_simple}. Interestingly, dualization does not introduce any conservatism since strong duality holds, as stated in equation \eqref{equality-duality}. To close the proof, we show that the value $\delta$ that appears in the statement of Theorem \ref{th:compression_1} is achieved by a feasible solution of the dual problem and, thereby, it upper bounds the optimal value of \eqref{eq:dual_simple} and, by this, that of $\Pr \{ \bphi_N > \eps_\bk \}$. This derivation is covered in the last part of the proof that starts after equation \eqref{proof-summary}. \\

The proof of Theorem \ref{th:compression_2} follows the same path as that of Theorem \ref{th:compression_1} with the non-trivial difference that the property (ii) holds with equality in this case (it has an inequality in the proof of Theorem \ref{th:compression_1}). This results in primal and dual problems that have substantial differences from those in the proof of Theorem \ref{th:compression_1}, while the conceptual structure of the proof remains the same.

\subsection{Proof of Theorem \ref{th:compression_1}}
\label{proof-Theorem-1}

Result \eqref{result-th-1} is first proven under the following additional assumption of no concentrated mass
\begin{equation}
\label{eq:no_conc_mass_proof}
\Pr \Big\{ \bz_i = z \Big\} = 0, \; \forall z \in \scZ;
\end{equation}
the extension to the general case is dealt with at the end of this proof. \\
\\
The quantity of interest $\Pr \{ \bphi_N > \eps_\bk \}$ can be expressed as follows
\begin{align*}
	\Pr \Big\{ \bphi_N > \eps_\bk \Big\}
&=
	\Pr \Big\{ \bphi_N > \eps_{|\co(\bz_1,\ldots,\bz_N)|} \Big\}
\\&=
	\sum_{k=0}^N \Pr \Big\{ |\co(\bz_1,\ldots,\bz_N)| = k \text{ and } \bphi_N > \eps_k \Big\}
\\&=
	\sum_{k=0}^N \Pr \left( \bigcup_{ \substack{\{ i_1,\ldots,i_k \} \\ \subseteq \{1,\ldots,N\}} } \Big\{ \co(\bz_1,\ldots,\bz_N) = \ms(\bz_{i_1},\ldots,\bz_{i_k}) \text{ and } \bphi_N > \eps_k \Big\} \right)
\\&=
	\sum_{k=0}^N \sum_{ \substack{\{ i_1,\ldots,i_k \} \\ \subseteq \{1,\ldots,N\}} } \Pr \Big\{ \co(\bz_1,\ldots,\bz_N) = \ms(\bz_{i_1},\ldots,\bz_{i_k}) \text{ and } \bphi_N > \eps_k \Big\},
\end{align*}
where the last equality holds because: due to \eqref{eq:no_conc_mass_proof}, $\bz_1 \neq \cdots \neq \bz_N$ occurs with probability $1$, and so the multisets $\ms(\bz_{i_1},\ldots,\bz_{i_k})$ are all different from each other with probability $1$; whence, $\co(\bz_1,\ldots,\bz_N) = \ms(\bz_{i_1},\ldots,\bz_{i_k})$ holds for one and only one choice of the indexes with probability $1$, implying that the events under the sign of union are disjoint up to overlaps of probability zero. \\
\\
Now, for any fixed $k$, all the probabilities in the inner summation are equal because the $\bz_i$'s are i.i.d. and so we can write
\begin{align}
\nonumber
	\sum_{k=0}^N \sum_{ \substack{\{ i_1,\ldots,i_k \} \\ \subseteq \{1,\ldots,N\}} } &\Pr \Big\{ \co(\bz_1,\ldots,\bz_N) = \ms(\bz_{i_1},\ldots,\bz_{i_k}) \text{ and } \bphi_N > \eps_k \Big\}
\\&=
\nonumber
	\sum_{k=0}^N {N \choose k} \Pr \Big\{ \co(\bz_1,\ldots,\bz_N) = \ms(\bz_1,\ldots,\bz_k) \text{ and } \bphi_N > \eps_k \Big\}
\\&=
	\sum_{k=0}^N {N \choose k} \int_{(\eps_k,1]} \dd \mathfrak{m}^+_{k,N},
\label{eq:objective_primal}
\end{align}
where $\mathfrak{m}^+_{k,N}$ is a (positive) measure on $[0,1]$ defined as follows (for future use we introduce a definition that holds for a generic integer $m$, and not just for $m=N$): for all $m = 0,1,\ldots$ and $k=0,\ldots,m$, let
\begin{equation}
\label{def-m}
\mathfrak{m}^+_{k,m}(B) = \Pr \Big\{ \co(\bz_1,\ldots,\bz_m) = \ms(\bz_1,\ldots,\bz_k) \text{ and } \bphi_m \in B \Big\},
\end{equation}
with $B$ any Borel set in $[0,1]$. \\
\\
Next we derive two relations (i) and (ii) that are satisfied by measures $\mathfrak{m}^+_{k,m}$ for all compression schemes that satisfy the \emph{preference} property; relations (i) and (ii)  will be in use when evaluating $\Pr \{ \bphi_N > \eps_\bk \}$.
\begin{itemize}
\item[\textbf{(i)}] For $m=0,1,\ldots$, it holds that (we use $\alpha$ as variable of integration) 
\begin{equation}
\label{condition(i)}
\sum_{k=0}^m {m \choose k} \int_{[0,1]} \dd \mathfrak{m}^+_{k,m}(\alpha) = 1;
\end{equation}
\item[\textbf{(ii)}] For $m=0,1,\ldots$ and $k=0,\ldots,m$, it holds that
\begin{equation}
\label{condition(ii)}
\int_B \dd  \mathfrak{m}^+_{k,m+1}(\alpha) - \int_B (1-\alpha) \; \dd \mathfrak{m}^+_{k,m}(\alpha) \leq 0,
\end{equation}
for any Borel set $B \subseteq [0,1]$.
\end{itemize}
For any given $B$, the left-hand side of \eqref{condition(ii)} returns a numerical value and, when $B$ ranges over the Borel sets in $[0,1]$, the left-hand side of \eqref{condition(ii)} defines a signed measure. Condition \eqref{condition(ii)} means that this measure is in fact negative. In the following, this measure will be denoted as $\mathfrak{m}^+_{k,m+1} - (1-\alpha) \mathfrak{m}^+_{k,m}$,\footnote{Note that $(1-\alpha) \mathfrak{m}^+_{k,m}$ cannot be interpreted as a product since $(1-\alpha)$ is not a number because it depends on $\alpha$; hence, ``$\mathfrak{m}^+_{k,m+1} - (1-\alpha) \mathfrak{m}^+_{k,m}$'' has
to be interpreted just as a symbol that indicates the measure defined via
the left-hand side of equation \eqref{condition(ii)}.} and condition (ii)
can also be written as $$
\mathfrak{m}^+_{k,m+1} - (1-\alpha) \mathfrak{m}^+_{k,m} \in \mathcal{M}^-,
$$
where $\mathcal{M}^-$ is the cone of negative finite measures on $[0,1]$.

\begin{itemize}
\item[] {\bf Proof of (i):} Along the same lines as the proof of \eqref{eq:objective_primal}, we obtain
\begin{eqnarray*}
1 & = & \sum_{k=0}^m \Pr \Big\{ |\co(\bz_1,\ldots,\bz_m)| = k \Big\} \nonumber \\
& = & \sum_{k=0}^m \Pr \left( \bigcup_{ \substack{\{ i_1,\ldots,i_k \} \\ \subseteq \{1,\ldots,m\}} } \Big\{ \co(\bz_1,\ldots,\bz_m) = \ms(\bz_{i_1},\ldots,\bz_{i_k}) \Big\} \right) \nonumber \\
& = & \sum_{k=0}^m \sum_{ \substack{\{ i_1,\ldots,i_k \} \\ \subseteq
\{1,\ldots,m\}} } \Pr \Big\{ \co(\bz_1,\ldots,\bz_m) = \ms(\bz_{i_1},\ldots,\bz_{i_k}) \Big\} \nonumber \\
& = & \sum_{k=0}^m {m \choose k} \Pr \Big\{ \co(\bz_1,\ldots,\bz_m) = \ms(\bz_1,\ldots,\bz_k) \Big\} \nonumber \\
& = & \sum_{k=0}^m {m \choose k} \int_{[0,1]} \dd \mathfrak{m}^+_{k,m}.
\end{eqnarray*} \qed
\item[] {\bf Proof of (ii):} For any given Borel set $B$ in $[0,1]$, we have that
\begin{equation}
\label{intBm}
\int_B \dd  \mathfrak{m}^+_{k,m+1} = \Pr \Big\{ \co(\bz_1,\ldots,\bz_{m+1})
= \ms(\bz_1,\ldots,\bz_k) \text{ and } \bphi_{m+1} \in B \Big\}.
\end{equation}
By Lemma \ref{lemma_fund}, relation $\co(z_1,\ldots,z_{m+1}) = \ms(z_1,\ldots,z_k)$ implies the following two facts:
\begin{itemize}
\item[(a)] $\co(z_1,\ldots,z_m) = \ms(z_1,\ldots,z_k)$;
\item[(b)] $\co(\co(z_1,\ldots,z_m),z_{m+1}) = \co(z_1,\ldots,z_m)$.
\end{itemize}
Equation (a) is an immediate consequence of Lemma \ref{lemma_fund}, while
(b) is proven by the following chain of equalities: $\co(\co(z_1,\ldots,z_m),z_{m+1}) = [\text{use (a)}] = \co(z_1,\ldots,z_k,z_{m+1}) = [\text{use Lemma \ref{lemma_fund}}] = \co(z_1,\ldots,z_m,z_{m+1}) = \ms(z_1,\ldots,z_k) = \co(z_1,\ldots,z_m)$.

Over the set where $\co(\bz_1,\ldots,\bz_{m+1}) = \ms(\bz_1,\ldots,\bz_k)$, it therefore holds that
\begin{eqnarray}
	\bphi_{m+1} & = & \Pr \Big\{ \co(\co(\bz_1,\ldots,\bz_{m+1}),\bz_{m+2}) \neq \co(\bz_1,\ldots,\bz_{m+1}) | \bz_1,\ldots,\bz_{m+1} \Big\} \nonumber \\
	& = & \Pr \Big\{ \co(\co(\bz_1,\ldots,\bz_{m}),\bz_{m+2}) \neq \co(\bz_1,\ldots,\bz_{m}) | \bz_1,\ldots,\bz_{m+1} \Big\} \nonumber \\
	&  & \big( \text{where we have used (a), which gives } \co(\bz_1,\ldots,\bz_{m+1}) = \co(\bz_1,\ldots,\bz_{m}) \big) \nonumber \\
	& = & \Pr \Big\{ \co(\co(\bz_1,\ldots,\bz_{m}),\bz_{m+2}) \neq \co(\bz_1,\ldots,\bz_{m}) | \bz_1,\ldots,\bz_{m} \Big\} \nonumber \\
	& = & \bphi_m \hspace*{2cm} \Pr \text{-almost surely}, \label{eq:Phi(m+1)=Phi(k)}
\end{eqnarray}
so that the right-had side of \eqref{intBm} can be re-written as 
\begin{eqnarray}
\label{m+i-in place of-m}
\Pr \Big\{ \co(\bz_1,\ldots,\bz_{m+1})
= \ms(\bz_1,\ldots,\bz_k) \text{ and } \bphi_{m+1} \in B \Big\} \nonumber \\
= 
\Pr \Big\{ \co(\bz_1,\ldots,\bz_{m+1})
= \ms(\bz_1,\ldots,\bz_k) \text{ and } \bphi_m \in B \Big\}. 
\end{eqnarray}
On the other hand, an application of (a) and (b) also gives\footnote{Importantly, inequality in \eqref{eq:preference_gives<=} may be strict. For example, suppose that $\bz$ is uniformly distributed over a circle with unitary circumference and that $\co(z_1,\ldots,z_n)$ selects the two points whose gap is smallest (i.e. $\co(z_1,\ldots,z_n) = \ms(z_{i_1},z_{i_2})$ such that no other pair of points is closer -- if a tie occurs use an arbitrary tie-break rule). Take $m=3$ and $B=[0,1]$. The left-hand side of \eqref{eq:preference_gives<=} equals $1/6$, as is obvious by observing that any choice of two points has the same probability of being selected. Instead, the right-hand side is the probability that $\co(\bz_1,\bz_2,\bz_4) = \ms(\bz_1,\bz_2)$ and $\co(\bz_1,\bz_2,\bz_3) = \ms(\bz_1,\bz_2)$. Naming $\bx$ the length of the arc connecting $\bz_1$ and $\bz_2$, we have: i. $\bx$ has uniform density equal to $2$ over $[0,1/2]$; ii. if $\bx > 1/3$, then adding one more point certainly changes the compression; and iii. if $\bx \leq 1/3$, then the probability that
one more point changes the compression is $3\bx$. Hence, the right-hand side of \eqref{eq:preference_gives<=} has value
$$
\Pr \Big\{ \co(\bz_1,\bz_2,\bz_4) = \ms(\bz_1,\bz_2) \text{ and } \co(\bz_1,\bz_2,\bz_3) = \ms(\bz_1,\bz_2) \Big\} = \int_0^{\frac{1}{3}} (1-3x)^2 \; 2 \dd x =
\frac{2}{9},
$$
which is strictly larger than $1/6$. We shall see in Theorem \ref{th:compression_2} that, by strengthening the assumptions of the theorem with the
introduction of the \emph{non-associativity} Property \ref{non-associativity}, inequality in \eqref{eq:preference_gives<=} turns into an equality and this provides lower bounds on the probability of change of compression in addition to the upper bound of the present theorem.}
\begin{eqnarray} \label{eq:preference_gives<=}
\lefteqn{ \Pr \Big\{ \co(\bz_1,\ldots,\bz_{m+1}) = \ms(\bz_1,\ldots,\bz_k) \text{ and } \bphi_m \in B \Big\} } \nonumber \\
& \leq & \Pr \Big\{ \co(\co(\bz_1,\ldots,\bz_m),\bz_{m+1}) = \co(\bz_1,\ldots,\bz_m) \text{ and } \nonumber \\
& & \hspace*{0.75cm} \co(\bz_1,\ldots,\bz_{m}) = \ms(\bz_1,\ldots,\bz_k) \text{
and } \bphi_m \in B \Big\}.
\end{eqnarray}
Using \eqref{m+i-in place of-m} and \eqref{eq:preference_gives<=} in \eqref{intBm}, we obtain
\begin{eqnarray} \label{eq:intB=Pr}
    \int_B \dd  \mathfrak{m}^+_{k,m+1}
    & \leq & \Pr \Big\{ \co(\co(\bz_1,\ldots,\bz_m),\bz_{m+1}) = \co(\bz_1,\ldots,\bz_m) \text{ and } \nonumber \\
    & & \quad \co(\bz_1,\ldots,\bz_{m}) = \ms(\bz_1,\ldots,\bz_k) \text{ and } \bphi_m \in B \Big\}. 
\end{eqnarray}
The proof of (ii) is now established by noticing that the right-hand side
of \eqref{eq:intB=Pr} can be re-written as follows
\begin{align*}
& \E  \Big[ \E \big[ \One{ \{ \co(\co(\bz_1,\ldots,\bz_m),\bz_{m+1}) =
\co(\bz_1,\ldots,\bz_m) \} } \cdot \One{ \{ \co(\bz_1,\ldots,\bz_{m}) = \ms(\bz_1,\ldots,\bz_k) \text{ and } \bphi_m \in B\} } | \bz_1,\ldots,\bz_m \big] \Big] 
\nonumber \\
& \quad = \E  \Big[ \E \big[  \One{ \{ \co(\co(\bz_1,\ldots,\bz_m),\bz_{m+1}) =
\co(\bz_1,\ldots,\bz_m) \} } | \bz_1,\ldots,\bz_m \big] \\
& \quad\quad\quad\ \cdot \One{ \{ \co(\bz_1,\ldots,\bz_{m}) = \ms(\bz_1,\ldots,\bz_k) \text{ and } \bphi_m \in B\} } \Big] \nonumber \\
& \quad =  \E  \Big[ (1-\bphi_m) \cdot \One{ \{ \co(\bz_1,\ldots,\bz_{m}) = \ms(\bz_1,\ldots,\bz_k) \text{ and } \bphi_m \in B\} } \Big] \nonumber \\
& \quad = \int_B (1-\alpha) \; \dd  \mathfrak{m}^+_{k,m}.
\end{align*}
\qed
\end{itemize}
We are now ready to upper-bound $\Pr \left\{ \bphi_N > \eps_\bk \right\}$ by
taking the $\sup$ of the right-hand side of \eqref{eq:objective_primal} under conditions (i) and (ii) (in addition to the fact that measures $\mathfrak{m}^+_{k,m}$ belong to the cone $\mathcal{M}^+$ of positive finite measures on $[0,1]$). This gives
\begin{equation} \label{eq:Pr_ch_comp<=gamma}
\Pr \left\{ \bphi_N > \eps_\bk \right\} \leq \gamma,
\end{equation}
where $\gamma$ is defined as the value of the optimization problem

\begin{eqnarray} \label{eq:primal_infinity}
\gamma = \sup_{\substack{\mathfrak{m}^+_{k,m} \in \mathcal{M}^+ \\ m=0,1,\ldots, \; \; k=0,\ldots,m }}
& & \sum_{k=0}^N {N \choose k} \int_{(\eps_k,1]} \dd \mathfrak{m}^+_{k,N}  \\
\textrm{subject to:} & & \sum_{k=0}^m {m \choose k} \int_{[0,1]} \dd \mathfrak{m}^+_{k,m} = 1, \quad m=0,1,\ldots \nonumber \\
& & \mathfrak{m}^+_{k,m+1} - (1-\alpha) \mathfrak{m}^+_{k,m} \in \mathcal{M}^-, \quad m=0,1,\ldots; \; k=0,\ldots,m. \nonumber
\end{eqnarray}
To evaluate $\gamma$, we consider a truncated version of problem \eqref{eq:primal_infinity} that only includes the measures $\mathfrak{m}^+_{k,m}$
 for $m \leq M$ (we take $M \geq N$). We then dualize the truncated problem and let $M$ increase. \\

The truncated problem is

\begin{subequations} \label{eq:primal_M}
\begin{eqnarray}
\gamma_M = \sup_{\substack{\mathfrak{m}^+_{k,m} \in \mathcal{M}^+ \\ m=0,\ldots, M, \; \; k=0,\ldots,m }}
& & \sum_{k=0}^N {N \choose k} \int_{(\eps_k,1]} \dd \mathfrak{m}^+_{k,N}  \label{eq:primal_M_cost} \\
\textrm{subject to:} & & \sum_{k=0}^m {m \choose k} \int_{[0,1]} \dd \mathfrak{m}^+_{k,m} = 1, \quad m=0,\ldots,M  \label{eq:primal_M_equality_constr} \\
& & \mathfrak{m}^+_{k,m+1} - (1-\alpha) \mathfrak{m}^+_{k,m} \in \mathcal{M}^-, \nonumber \\
& & m=0,\ldots,M-1; \; k=0,\ldots,m. \label{eq:primal_M_inequality_constr}
\end{eqnarray}
\end{subequations}
As $M$ increases, one adds new constraints (which also contain new variables), while the cost function and previous constraints remain unchanged. Hence, $\gamma_M$ does not increase with $M$ and
\begin{equation}
\label{gammaM>gamma}
\gamma \leq \gamma_M,
\end{equation}
for all $M$. To dualize \eqref{eq:primal_M}, consider the Lagrangian:

\begin{eqnarray} \label{eq:Lagrangian}
\mathfrak{L} & = & \sum_{k=0}^N {N \choose k} \int_{(\eps_k,1]} \dd \mathfrak{m}^+_{k,N}
- \sum_{m=0}^M \lambda_m \left( \sum_{k=0}^m {m \choose k} \int_{[0,1]} \dd \mathfrak{m}^+_{k,m} - 1 \right) \nonumber \\
& & - \sum_{m=0}^{M-1} \sum_{k=0}^m \int_{[0,1]} \mu^+_{k,m}(\alpha) \; \dd [ \mathfrak{m}^+_{k,m+1} - (1-\alpha) \mathfrak{m}^+_{k,m}],
\end{eqnarray}
which is a function of
\begin{itemize}
\item[$\diamond$] $\mathfrak{m}^+_{k,m} \in \mathcal{M}^+, \; \; m=0,\ldots,M, \; \; k=0,\ldots,m,$
\end{itemize}
and the Lagrange multipliers
\begin{itemize}
\item[$\diamond$] $\lambda_m \in \R, \; \; m=0, \ldots, M$, \\
\vspace*{-6mm}
\item[$\diamond$] $\mu^+_{k,m} \in \textsf{C}^0_+[0,1], \; \; m=0,\ldots,M-1, \; \; k=0,\ldots,m$,
\end{itemize}
where
$\textsf{C}^0_+[0,1]$ is the set of positive and continuous functions over $[0,1]$. \\
\\
We show below that\footnote{In various parts of this paper from here onward, the set of measures $\mathfrak{m}^+_{k,m}, \; m=0,\ldots,M, \; \; k=0,\ldots,m,$ is indicated by the notation $\{\mathfrak{m}^+_{k,m}\}$, where the range of variability for $m$ and $k$ is suppressed for brevity. Similar notations apply to $\lambda_m$ and $\mu^+_{k,m}$ and other collections alike.}
\begin{equation}
\label{equality-duality}
\gamma_M
~\stackrel{\mathrm{(A)}}{=}~
\sup_{\{\mathfrak{m}^+_{k,m}\}} \inf_{\substack{\{\lambda_m\} \\ \{ \mu^+_{k,m} \} }} \mathfrak{L}
~\stackrel{\mathrm{(B)}}{=}~
\inf_{\substack{\{\lambda_m\} \\ \{ \mu^+_{k,m} \} }} \sup_{ \{\mathfrak{m}^+_{k,m}\} } \mathfrak{L}
~\stackrel{\mathrm{(C)}}{=}~
\gamma^\ast_M,
\end{equation}
where $\gamma^\ast_M$ is the value of the dual of problem \eqref{eq:primal_M} ($\One{}$ denotes the indicator function):
\begin{subequations} \label{eq:dual_M}
\begin{eqnarray}
\gamma^\ast_M = \inf_{\substack{\lambda_m, \; m=0,\ldots,M \\ \mu^+_{k,m} \in \textsf{C}^0_{+}[0,1], \; m=0,\ldots,M-1, \; \; k=0,\ldots,m
}} & & \sum_{m=0}^{M}  \lambda_m  \label{eq:dual_M_cost} \\
\textrm{subject to:} & & {m \choose k} \One{\alpha \in (\eps_k,1]} \One{m=N} + (1 \! - \! \alpha) \mu^+_{k,m}(\alpha) \One{m \neq M}  \nonumber \\
& & \leq  \lambda_m {m \choose k} + \mu^+_{k,m-1}(\alpha)
\One{m \neq k}, \; \forall \alpha \in [0,1], \nonumber \\
& &  k = 0, \ldots, M; \; m = k, \ldots, M \label{eq:dual_M_constr}
\end{eqnarray}
\end{subequations}
(note that in \eqref{eq:dual_M_constr} the indexes run over a range such that there appear functions, for instance $\mu^+_{0,-1}$, that are not listed as optimization variables;
however, these functions are all multiplied by an indicator function that
is zero and they therefore disappear; we have used this way of writing the constraints because it simplifies the notation).
\begin{itemize}
\item[] \textbf{Proof of (A) in \eqref{equality-duality}:} If measures ${\mathfrak{m}^+_{k,m}}$ do not satisfy the constraints in \eqref{eq:primal_M_equality_constr} and \eqref{eq:primal_M_inequality_constr}, then $\inf_{\{\lambda_m\}, \{\mu^+_{k,m}\}} \mathfrak{L}$ is equal to $-\infty$. This is true for \eqref{eq:primal_M_equality_constr} because, if for some $m$ the term
    $$
    \left( \sum_{k=0}^m {m \choose k} \int_{[0,1]} \dd \mathfrak{m}^+_{k,m} - 1 \right)
    $$
    in the right-hand side of \eqref{eq:Lagrangian} is not null, then $\lambda_m$ can be taken any large with sign equal to that of that term, bringing $\mathfrak{L}$ down to arbitrary large negative values. Likewise, if \eqref{eq:primal_M_inequality_constr} is not satisfied for a given pair
$(k,m)$, then the last term in the right-hand side of \eqref{eq:Lagrangian} can be made any large negative by selecting a suitable positive large continuous function $\mu^+_{k,m}$.\footnote{\label{key_footnote}Intuitively, this is achieved by a function $\mu^+_{k,m}$ that is concentrated over the domain where $\mathfrak{m}^+_{k,m+1} - (1-\alpha) \mathfrak{m}^+_{k,m}$ is positive. In this footnote, we provide the interested reader with
a formal construction of such a function $\mu^+_{k,m}$. Note that, if condition \eqref{eq:primal_M_inequality_constr} is not satisfied for a given
pair $(k,m)$, then there is a Borel set $B$ in $[0,1]$ such that
    \begin{equation}
    \label{int-int}
    \int_B \dd \mathfrak{m}^+_{k,m+1} - \int_B (1-\alpha) \; \dd \mathfrak{m}^+_{k,m} > 0.
    \end{equation}
    Letting $\mathfrak{m}$ be the measure $\mathfrak{m}^+_{k,m+1} + \mathfrak{m}^+_{k,m}$, $B$ can be sandwiched between a closed set $C$ and an open set $O$ ($C \subseteq B \subseteq O$, note that $O \subseteq \R$, but
it may not be restricted to $[0,1]$) such that $\mathfrak{m}(O)-\mathfrak{m}(C) < \varepsilon$ for any arbitrarily small $\varepsilon$ (Theorem 12.3 in \citealp{Billingsley}).
Let now
$$
g(\alpha) = \frac{\mathrm{dist}(\alpha,O^c)}{\mathrm{dist}(\alpha,O^c)+\mathrm{dist}(\alpha,C)}, \quad \forall \alpha \in \R,
$$
where $\mathrm{dist}(\alpha,X) = \inf \{ |\alpha-x|: \; x \in X \}$
and $O^c$ is the complement of $O$. $g$ is a continuous function with codomain [0,1] and $g(\alpha) = 0$ on $O^c$ while $g(\alpha) = 1$ on $C$. We have that
\begin{eqnarray*}
    \lefteqn{ \int_{[0,1]} g(\alpha) \; \dd [ \mathfrak{m}^+_{k,m+1} - (1-\alpha) \mathfrak{m}^+_{k,m}] } \\
    & = &
    \int_{[0,1]} g(\alpha) \; \dd \mathfrak{m}^+_{k,m+1} - \int_{[0,1]} g(\alpha) \cdot (1-\alpha) \; \dd \mathfrak{m}^+_{k,m} \\
        & = &
    \int_O g(\alpha) \; \dd \mathfrak{m}^+_{k,m+1} - \int_O g(\alpha) \cdot (1-\alpha) \; \dd \mathfrak{m}^+_{k,m} \\
    & & \quad \quad (\text{since } O \text{ can expand beyond } [0,1] \text{ but } \mathfrak{m}^+_{k,m+1} \text{ and } \mathfrak{m}^+_{k,m} \text{ are supported in } [0,1]) \\
    & \geq &
    \int_C g(\alpha) \; \dd \mathfrak{m}^+_{k,m+1} - \int_C g(\alpha) \cdot (1-\alpha) \; \dd \mathfrak{m}^+_{k,m} - \varepsilon \\
    & &
    \quad \quad (\text{since } 0 \leq g(\alpha) \cdot (1-\alpha) \leq 1 \text{ for } \alpha \in [0,1] \text{ and } \int_{O \setminus C} \; \dd \mathfrak{m}^+_{k,m} < \varepsilon) \\
    & = &
    \int_C \; \dd \mathfrak{m}^+_{k,m+1} - \int_C (1-\alpha) \; \dd \mathfrak{m}^+_{k,m} - \varepsilon \\
    & \geq &
    \int_B \; \dd \mathfrak{m}^+_{k,m+1} - \varepsilon - \int_B (1-\alpha) \; \dd \mathfrak{m}^+_{k,m} - \varepsilon \\
    & &
    \quad \quad ( \text{since } \int_{B \setminus C} \dd \mathfrak{m}^+_{k,m+1} < \varepsilon ) \\
    & > &
    0, \ \ \ \text{for } \varepsilon \text{ small enough and using \eqref{int-int}.}
    \end{eqnarray*}
Taking $\mu^+_{k,m}$ to be an arbitrarily rescaled version of the restriction of $g$ to $[0,1]$, one obtains that the last term in the right-hand side of \eqref{eq:Lagrangian} can be made any large negative.
}
Hence, the $\sup_{\{\mathfrak{m}^+_{k,m}\}}$ of $\inf_{\{\lambda_m\},
\{\mu^+_{k,m}\}} \mathfrak{L}$ is attained at measures $\mathfrak{m}^+_{k,m}$ satisfying \eqref{eq:primal_M_equality_constr} and \eqref{eq:primal_M_inequality_constr} and, once \eqref{eq:primal_M_equality_constr} and \eqref{eq:primal_M_inequality_constr} hold, $\inf_{\{\lambda_m\}, \{\mu^+_{k,m}\}} \mathfrak{L}$ is achieved by setting the second and third terms in the right-hand side of \eqref{eq:Lagrangian} to zero (choose $\lambda_m$ to be any value and $\mu^+_{k,m}$, e.g., equal to zero for all $m$ and $k$). This leads to the conclusion that $\sup_{\{\mathfrak{m}^+_{k,m}\}} \inf_{\{\lambda_m\} , \{\mu^+_{k,m}\}} \mathfrak{L}$ equals $\gamma_M$ of problem \eqref{eq:primal_M}. \qed
\item[] \textbf{Proof of (B) in \eqref{equality-duality}:} This long and technical proof is provided in Appendix \ref{Appendix strong duality}. \qed
\item[] \textbf{Proof of (C) in \eqref{equality-duality}:} First note that the Lagrangian can be rewritten as follows (in the second last term we have used the change of running index $j = m \! + \! 1$)
\begin{eqnarray*}
\mathfrak{L} & = & \sum_{m=0}^M \sum_{k=0}^m \int_{[0,1]} {m \choose k} \One{\alpha \in (\eps_k,1]} \One{m=N} \; \dd \mathfrak{m}^+_{k,m}
- \sum_{m=0}^M \sum_{k=0}^m \int_{[0,1]} \lambda_m {m \choose k} \; \dd \mathfrak{m}^+_{k,m} \nonumber \\
& & + \sum_{m=0}^M \lambda_m   - \sum_{j=0}^{M} \sum_{k=0}^j \int_{[0,1]} \mu^+_{k,j-1}(\alpha) \One{j \neq k} \; \dd \mathfrak{m}^+_{k,j} \nonumber \\
& & + \sum_{m=0}^{M} \sum_{k=0}^m \int_{[0,1]} \mu^+_{k,m}(\alpha) \cdot (1-\alpha) \One{m \neq M} \; \dd \mathfrak{m}^+_{k,m}.
\end{eqnarray*}
By renaming $j$ as $m$ in the second last term and re-arranging the summations $\sum_{m=0}^{M} \sum_{k=0}^m$ as $\sum_{k=0}^M \sum_{m=k}^{M}$, we then obtain:
\begin{eqnarray} \label{eq:Lagrangian_2}
\mathfrak{L} & = & \sum_{m=0}^M \lambda_m + \sum_{k=0}^M \sum_{m=k}^M \int_{[0,1]} \Bigg[ {m \choose k} \One{\alpha \in (\eps_k,1]} \One{m=N} + (1-\alpha) \mu^+_{k,m}(\alpha) \One{m \neq M} \nonumber \\
&& - \lambda_m {m \choose k} - \mu^+_{k,m-1}(\alpha) \One{m \neq k} \Bigg] \; \dd \mathfrak{m}^+_{k,m}.
\end{eqnarray}
Now, if for some pair $(k,m)$ the constraint in \eqref{eq:dual_M_constr} is not satisfied for a given $\alpha = \bar{\alpha}$, then $\sup_{\{\mathfrak{m}^+_{k,m}\}} \mathfrak{L}$ can be sent to $+\infty$ by choosing $\mathfrak{m}^+_{k,m}$ that has an arbitrarily large mass concentrated in $\bar{\alpha}$. Hence, the $\inf_{\{\lambda_m\}, \{\mu^+_{k,m}\}}$ of $\sup_{\{\mathfrak{m}^+_{k,m}\}} \mathfrak{L}$ is attained at $\lambda_m$'s and $\mu^+_{k,m}$'s satisfying \eqref{eq:dual_M_constr} and, once \eqref{eq:dual_M_constr} holds, $\sup_{\{\mathfrak{m}^+_{k,m}\}} \mathfrak{L}$ is achieved by setting the second term in the right-hand side of \eqref{eq:Lagrangian_2} to zero (choose, e.g., $\mathfrak{m}^+_{k,m}=0$ for all $k$ and $m$). This leads to the conclusion that $\inf_{\{\lambda_m\} , \{\mu^+_{k,m}\}} \sup_{\{\mathfrak{m}^+_{k,m}\}} \mathfrak{L}$ equals $\gamma^\ast_M$ of problem \eqref{eq:dual_M}. \qed
\end{itemize}
Next we want to evaluate $\gamma^\ast_M$ of problem \eqref{eq:dual_M}. \\
\\
For a better visualization of the constraints in \eqref{eq:dual_M_constr}, we write them more explicitly in groups indexed by $k$ as follows:
\begin{subequations}\label{spataffiata}
\begin{equation}
\begin{array}{lrcll}
\lefteqn{\underline{k=0,\ldots,N \! - \! 1}} \\
& (1 \! - \! \alpha) \mu^+_{k,k}(\alpha) & \leq & \lambda_k {k \choose k}
& \phantom{A}_{m = k} \\
& (1 \! - \! \alpha) \mu^+_{k,k+1}(\alpha) & \leq & \lambda_{k+1} {k+1 \choose k} + \mu^+_{k,k}(\alpha) & \phantom{A}_{m = k+1} \\
& & \vdots & \\
& (1 \! - \! \alpha) \mu^+_{k,N-1}(\alpha) & \leq & \lambda_{N-1} {N-1 \choose k} +\mu^+_{k,N-2}(\alpha) & \phantom{A}_{m = N-1} \\
& {N \choose k} \One{\alpha \in (\eps_k,1]} + (1 \! - \! \alpha) \mu^+_{k,N}(\alpha) & \leq & \lambda_N {N \choose k} + \mu^+_{k,N-1}(\alpha)  & \phantom{A}_{m = N} \\
& (1 \! - \! \alpha) \mu^+_{k,N+1}(\alpha) & \leq & \lambda_{N+1} {N+1 \choose k} + \mu^+_{k,N}(\alpha) & \phantom{A}_{m = N+1} \\
& & \vdots & \\
& (1 \! - \! \alpha) \mu^+_{k,M-1}(\alpha) & \leq & \lambda_{M-1} {M-1 \choose k} + \mu^+_{k,M-2}(\alpha) & \phantom{A}_{m = M-1} \\
& 0 & \leq & \lambda_{M} {M \choose k} + \mu^+_{k,M-1}(\alpha)  & \phantom{A}_{m = M} \\
\underline{k=N} & & & \\
& {N \choose N} \One{\alpha \in (\eps_N,1]} + (1 \! - \! \alpha) \mu^+_{N,N}(\alpha) & \leq & \lambda_N {N \choose N}  & \phantom{A}_{m = N} \\
& (1 \! - \! \alpha) \mu^+_{N,N+1}(\alpha) & \leq & \lambda_{N+1} {N+1 \choose N} + \mu^+_{N,N}(\alpha) & \phantom{A}_{m = N+1} \\
& & \vdots & \\
& (1 \! - \! \alpha) \mu^+_{N,M-1}(\alpha) & \leq & \lambda_{M-1} {M-1 \choose N} + \mu^+_{N,M-2}(\alpha) & \phantom{A}_{m = M-1} \\
& 0 & \leq & \lambda_{M} {M \choose N} + \mu^+_{N,M-1}(\alpha) & \phantom{A}_{m = M}
\end{array}
\end{equation}
\begin{equation}
\begin{array}{lrcll}
\lefteqn{\underline{k=N\!+\!1,\ldots,M-1} } \\
& (1 \! - \! \alpha) \mu^+_{k,k}(\alpha) & \leq & \lambda_{k} {k \choose k} & \phantom{A}_{m = k} \\
& (1 \! - \! \alpha) \mu^+_{k,k+1}(\alpha) & \leq & \lambda_{k+1} {k+1 \choose k} + \mu^+_{k,k}(\alpha) & \phantom{A}_{m = k+1} \\
& &  \vdots & \\
& (1 \! - \! \alpha) \mu^+_{k,M-1}(\alpha) & \leq & \lambda_{M-1} {M-1 \choose k} + \mu^+_{k,M-2}(\alpha) & \phantom{A}_{m = M-1} \\
& 0 & \leq & \lambda_{M} {M \choose k} + \mu^+_{k,M-1}(\alpha) & \phantom{A}_{m = M} \\
\underline{k=M} & & & \\
& 0 & \leq & \lambda_{M} {M \choose M} & \phantom{A}_{m = M}
\end{array}
\end{equation}
\end{subequations}
For any given $k \in \{0,\ldots,M\}$, consider the corresponding set of inequalities and multiply both sides of the first inequality by $(1-\alpha)^0$, both sides of the second inequality by $(1-\alpha)^1$, and so on till the last inequality, which is multiplied by $(1-\alpha)^{M-k}$. Then, summing side-by-side the so-obtained inequalities, and noting that all functions $\mu^+_{k,m}(\alpha)$ cancel out, one obtains that the constraints in \eqref{spataffiata} imply the following inequalities:
\begin{eqnarray} \label{eq:dual_constr_simple}
\underline{k=0,\ldots,N}
& &
{N \choose k} (1-\alpha)^{N-k} \One{\alpha \in (\eps_k,1]} \leq \sum_{m=k}^M \lambda_m {m \choose k} (1-\alpha)^{m-k} \nonumber \\
\underline{k=N\!+\!1,\ldots,M}
& &
0 \leq \sum_{m=k}^M \lambda_m {m \choose k} (1-\alpha)^{m-k}.
\end{eqnarray}
We next show that the optimal value of problem \eqref{eq:dual_M} equals the optimal value of an optimization problem with the same cost function as in problem \eqref{eq:dual_M} and the constraints \eqref{eq:dual_constr_simple} complemented with the condition $\lambda_m = 0$ for $m = N+1,\ldots, M$, viz.
\begin{subequations} \label{eq:dual_simple}
\begin{eqnarray}
\gamma^\ast_M = \inf_{\lambda_m, \; m=0,\ldots,M} & & \sum_{m=0}^{M}  \lambda_m  \label{eq:dual_simple_cost} \\
\textrm{subject to:} & & {N \choose k} (1-\alpha)^{N-k} \One{\alpha \in (\eps_k,1]} \leq  \sum_{m=k}^M \lambda_m {m \choose k} (1-\alpha)^{m-k},
\nonumber \\
& & \forall \alpha \in [0,1], \; k = 0,\ldots, N \label{eq:dual_simple_constr_N} \\
& & 0 \leq \sum_{m=k}^M \lambda_m {m \choose k} (1-\alpha)^{m-k}, \quad \forall \alpha \in [0,1], \nonumber \\
& & k = N+1,\ldots,M  \label{eq:dual_simple_constr_M} \\
& & \lambda_m = 0 \text{ for } m = N+1,\ldots, M \label{eq:dual_simple_constr_null}
\end{eqnarray}
\end{subequations}
(clearly, \eqref{eq:dual_simple_constr_M} is automatically satisfied in view of \eqref{eq:dual_simple_constr_null}). \\
\\
To show that the values of $\gamma^\ast_M$ given by \eqref{eq:dual_M} and
\eqref{eq:dual_simple} are actually the same, start by noting that adding
the condition $\lambda_m = 0$ for $m = N+1,\ldots, M$ to problem $\eqref{eq:dual_M}$ does not change its optimal value. This requires a short proof:
\begin{itemize}
\item[] The constraints in \eqref{spataffiata} imply that $\lambda_m \geq
0$ for $m=0, \ldots, M$ as it can be seen from the first inequality ($m
= k$) of each group ($k=0,\ldots,M$) evaluated at $\alpha = 1$. Now, given a feasible point of \eqref{spataffiata} that does not have $\lambda_m = 0$ for $m = N+1,\ldots, M$, consider a modified point by setting $\lambda_m = 0$ for $m = N+1,\ldots, M$ and $\mu^+_{k,m} = 0$ for $k=0,\ldots,M-1$ and $m = \max\{k,N\},\ldots, M-1$, while maintaining the original choices for all other $\lambda_m$ and $\mu^+_{k,m}$. This
point is still feasible for \eqref{spataffiata} because, for all $k$, all
the inequalities for $m \geq N+1$ become $0 \leq 0$, the inequality for $m = N$ is \emph{a-fortiori} satisfied (recall that function $\mu^+_{k,N}$ in the left-hand side of this inequality is $\geq 0$ so that setting it to $0$ relaxes the constraint) and all other inequalities are not affected. On the other hand, the value of problem \eqref{eq:dual_M} corresponding to the modified point outdoes the value at the original point since all $\lambda_m$ in the original feasible point were nonnegative and some of
them have been set to zero in the modified point.
\end{itemize}
Since the condition $\lambda_m = 0$ for $m=N+1,\ldots,M$ in \eqref{eq:dual_simple_constr_null} can be added to \eqref{eq:dual_M} without affecting its optimal value, and considering that the other constraints in \eqref{eq:dual_simple} for $k=0,\ldots,M$ are implied by those already present in \eqref{eq:dual_M} (as
shown before equation \eqref{eq:dual_constr_simple}), the optimal value of \eqref{eq:dual_simple} is not bigger than the optimal value of \eqref{eq:dual_M}. The reverse inequality that the optimal value of \eqref{eq:dual_M} is not bigger than the optimal value of \eqref{eq:dual_simple} is proven by showing that for any feasible point of \eqref{eq:dual_simple} one
can find a feasible point of \eqref{eq:dual_M} that attains the same value. This is shown in the following.
\begin{itemize}
\item[] Consider a feasible point of \eqref{eq:dual_simple}. Evaluating all constraints \eqref{eq:dual_simple_constr_N} for $k=0,\ldots,N$, at $\alpha = 1$, one sees that $\lambda_m \geq 0$ for $m=0, \ldots, N$. Moreover, it holds that $\lambda_m
= 0$ for $m = N+1,\ldots, M$. To find the sought feasible point of \eqref{eq:dual_M}, consider the same $\lambda_m$ as those for the feasible point of \eqref{eq:dual_simple} and complement them with the following functions $\mu^+_{k,m}$. For $k=0,\ldots,M-1$, $m = \max\{k,N\},\ldots, M-1$, take $\mu^+_{k,m} = 0$. With this choice, all the inequalities in \eqref{spataffiata} for $k=0,\ldots,M$, $m = \max\{k,N+1\}, \ldots, M$
become $0 \leq 0$ and are therefore satisfied. The expressions of $\mu^+_{k,m}$ for the remaining indexes are first defined over $[0,1)$ and then extended to the closed interval $[0,1]$. Over $[0,1)$, consider the inequalities in \eqref{spataffiata} for $k=0,\ldots,N-1$, $m = k, \ldots, N-1$ and take $\mu^+_{k,m}(\alpha)$ such that these inequalities are satisfied with equality, starting from top and then proceeding downwards. This gives
    \begin{eqnarray}
    \label{eq:moose}
    \mu^+_{k,k}(\alpha) & = & \frac{\lambda_k {k \choose k}}{1-\alpha},
\nonumber \\
    \mu^+_{k,k+1}(\alpha) & = & \frac{\lambda_{k+1} {k+1 \choose k}}{1-\alpha} + \frac{\lambda_k {k \choose k}}{(1-\alpha)^2} \nonumber \\
    & \vdots & \\
    \mu^+_{k,N-1}(\alpha) & = & \sum_{j=k}^{N-1}\frac{\lambda_j {j \choose k}}{(1-\alpha)^{N-j}}. \nonumber
    \end{eqnarray}
    Since $\lambda_m \geq 0$, the obtained $\mu^+_{k,m}(\alpha)$'s are all positive and, moreover, are continuous over $[0,1)$. We next show that choice \eqref{eq:moose} satisfies over $[0,1)$ the remaining inequalities (those in \eqref{spataffiata} for $k=0,\ldots,N$ and $m = N$). For $k=0,\ldots,N-1$ and $m = N$, substituting $\mu^+_{k,N-1}(\alpha) = \sum_{j=k}^{N-1}\frac{\lambda_j {j \choose k}}{(1-\alpha)^{N-j}}$ and $\mu^+_{k,N}(\alpha) = 0$ gives
    \begin{equation} \label{eq:central_ineq_N-1}
    {N \choose k} \One{\alpha \in (\eps_k,1]} \leq \sum_{j=k}^{N} \lambda_j {j \choose k} \frac{1}{(1-\alpha)^{N-j}},
    \end{equation}
    while for $k=N$ and $m = N$, substituting $\mu^+_{N,N}(\alpha) = 0$ we have
    \begin{equation} \label{eq:central_ineq_N}
    {N \choose N} \One{\alpha \in (\eps_N,1]} \leq \lambda_N {N \choose N}.
    \end{equation}
    Equations \eqref{eq:central_ineq_N-1} and \eqref{eq:central_ineq_N} are satisfied because they coincide with \eqref{eq:dual_simple_constr_N} (recall that $\lambda_m = 0$ for $m = N+1, \ldots, M$ -- see \eqref{eq:dual_simple_constr_null}). As for $\alpha = 1$, note that functions $\mu^+_{k,m}$ defined in \eqref{eq:moose} tend to infinity when $\alpha \to 1$. This poses a problem of existence for $\alpha = 1$, which, however, can be easily circumvented by truncating the functions $\mu^+_{k,m}$ in the interval $\alpha \in [1 \! - \! \rho,1]$ at the value $\mu^+_{k,m}(1
\! - \! \rho)$ to obtain
    $$
    \mu^{+,\rho}_{k,m}(\alpha) =
    \begin{cases}
    \mu^+_{k,m}(\alpha) & \alpha < 1-\rho \\
    \mu^+_{k,m}(1 \! - \! \rho) & \alpha \geq 1-\rho,
    \end{cases}
    $$
    and noting that all the inequalities are satisfied over $[0,1]$ if $\rho$ is chosen small enough.
\end{itemize}
Summarizing the results so far, we have
\begin{equation}
\label{proof-summary}
\Pr \left\{ \bphi_N > \eps_\bk \right\} \stackrel{\eqref{eq:Pr_ch_comp<=gamma}}{\leq}\gamma \stackrel{\eqref{gammaM>gamma} }{\leq} \gamma_M \stackrel{\eqref{equality-duality}}{=} \gamma^\ast_M,
\end{equation}
where $\gamma^\ast_M$ is given by \eqref{eq:dual_simple}. Notice now that
increasing $M$ beyond $N$ in \eqref{eq:dual_simple} does not change the problem because $\lambda_m = 0$ for $m \geq N+1$, so that $\gamma^\ast_M
= \gamma^\ast_N$ for all $M \geq N$. The proof of the theorem (under condition \eqref{eq:no_conc_mass_proof}) is concluded by showing that $\gamma^\ast_N \leq \delta$, which is what we do next. \\
\\
For $M=N$, problem \eqref{eq:dual_simple} becomes
\begin{eqnarray} \label{eq:dual_simple_M=N}
\gamma^\ast_N = \inf_{\lambda_m, \; m=0,\ldots,N} & & \sum_{m=0}^{N}  \lambda_m  \\
\textrm{subject to:} & & {N \choose k} (1-\alpha)^{N-k} \One{\alpha \in (\eps_k,1]} \leq  \sum_{m=k}^N \lambda_m {m \choose k} (1-\alpha)^{m-k},
\nonumber \\
& & \alpha \in [0,1], \quad k = 0,\ldots, N. \nonumber
\end{eqnarray}
Take $\lambda_m = \frac{\delta}{N}$ for $m=0,\ldots,N-1$ and $\lambda_N = 0$,  so that $\sum_{m=0}^N \lambda_m = \delta$. We show that these $\lambda_m$'s are feasible for $\eqref{eq:dual_simple_M=N}$ so that $\gamma^\ast_N  \leq \sum_{m=0}^N \lambda_m = \delta$. The inequality for $k=N$ is satisfied because the left-hand side is $0$ (recall that $\eps_N = 1$ so that the indicator function is $1$ over an empty set). For $k=0,\ldots,N-1$, the inequalities for $\alpha = 1$ become $0
\leq \lambda_k$, which is true, while for $\alpha \in [0,1)$ the inequalities can be rewritten as
$$
\One{\alpha \in (\eps_k,1]} \leq \frac{\delta}{N}\sum_{m=k}^{N-1} \frac{\binom{m}{k}}{\binom{N}{k}} (1-\alpha)^{-(N-m)}, \quad k = 0,\ldots, N-1,
$$
and are satisfied in view of the definition of $\eps_k$, see \eqref{epsilonk}. This concludes the proof under condition \eqref{eq:no_conc_mass_proof}. \\
\\
Next we remove condition \eqref{eq:no_conc_mass_proof}. \\
\\
Let us augment each random element $\bz_i$ with a random variable $\btheta_i$ uniformly distributed over $[0,1]$ and independent of $\bz_i$ so as to form an i.i.d. sequence $\bz'_1 = (\bz_1,\btheta_1), \bz'_2 = (\bz_2,\btheta_2),\ldots$.\footnote{Note that the augmented random elements $\bz'_i$ are mere mathematical tools used to draw conclusions on $\bz_i$, which remain the measured and relevant variables.} Clearly, condition \eqref{eq:no_conc_mass_proof} applies (\emph{mutatis mutandis}) to $\bz'_i$, viz.
\begin{equation}
\label{eq:no_conc_mass_proof'}
\Pr \Big\{ \bz'_i = z' \Big\} = 0, \; \forall z' \in \scZ \times [0,1].
\end{equation}
Given a multiset of augmented examples $\ms(z'_1,\dots,z'_n)$, let $$
\proj[\ms(z'_1,\dots,z'_n)] = \ms(z_1,\dots, z_n),
$$
i.e., $\proj$ is the
extractor of the $z_i$ components. We define a compression $\co'$ to be applied to multisets of augmented examples as the compression that satisfies the following rule: $\proj[\co'(z'_1,\dots,z'_n)] = \co(z_1,\ldots,z_n)$ and, among sub-multisets of $\ms(z'_1,\ldots,z'_n)$ whose projections is $\co(z_1,\ldots,z_n)$, $\co'$ favors augmented examples with lower second components $\theta_i$. We next show that $\co'$ inherits from $\co$
the \emph{preference} property. Suppose that $\co'(z'_1,\dots,z'_n,z') \subseteq
\ms(z'_1,\dots,z'_n)$. This implies that $\co(z_1,\dots,z_n,z) \subseteq \ms(z_1,\dots,z_n)$. Then,
\begin{eqnarray}
\proj[\co'(z'_1,\dots,z'_n,z')]
& = & \co(z_1,\dots,z_n,z) \nonumber \\
& = & \co(z_1,\dots,z_n) \quad \quad (\mbox{because of the \emph{preference} property of } \co) \nonumber \\
& = & \proj[\co'(z'_1,\dots,z'_n)] \nonumber
\end{eqnarray}
and, hence, the $z$ components of $\co'(z'_1,\dots,z'_n,z')$ and those of
$\co'(z'_1,\dots,z'_n)$ coincide. Moreover, also the $\theta$ components coincide by the rule that favors lower second components. This establishes the \emph{preference} property of $\co'$. \\
\\
In view of \eqref{eq:no_conc_mass_proof'} and the fact that $\co'$ has the \emph{preference} property, we are in the position to apply to $\co'$ the proof that has been developed before under the assumption of no concentrated mass. Defining $\bphi'_N = \Pr \{ \co'(\co'(\bz'_1,\ldots,\bz'_N),\bz'_{N+1}) \neq \co'(\bz'_1,\ldots,\bz'_N) | \bz'_1,\ldots,\bz'_N \}$ and $\bk' = |\co'(\bz'_1,\ldots,\bz'_N)|$, we have
$$
\Pr \{ \bphi'_N > \eps_{\bk'} \} \leq \delta.
$$
On the other hand, 
$$
\co(\co(z_1,\ldots,z_N),z_{N+1}) \neq \co(z_1,\ldots,z_N)
$$
implies that 
$$
\co'(\co'(z'_1,\ldots,z'_N),z'_{N+1}) \neq \co'(z'_1,\ldots,z'_N)
$$
(while the vice-versa does not hold), which gives $\bphi_N \leq \bphi'_N$ $\Pr$-almost surely. Moreover, $\bk = \bk'$. Hence, 
$$
\Pr \{ \bphi_N > \eps_\bk \} \leq \Pr \{ \bphi'_N > \eps_\bk \} = \Pr \{ \bphi'_N > \eps_{\bk'} \} \leq \delta.
$$
This concludes the proof. \qed

\subsection{Proof of Theorem \ref{th:compression_2}}
\label{proof-theorem-2}

We prove the equivalent statement that
$$
\Pr \Big\{ \bphi_N < \underline{\eps}_\bk \text{ or } \bphi_N > \overline{\eps}_\bk \Big\} \leq \delta.
$$
The proof parallels that of Theorem \ref{th:compression_1} and we highlight here the differences. \\
\\
Result \eqref{eq:objective_primal} holds unaltered in the present context, with the only notational difference that $\eps_k$ is now $\overline{\eps}_k$. By proving a similar equation for $\Pr \{ \bphi_N < \underline{\eps}_\bk \}$, we come to the result
\begin{equation}
\label{PPHI<PHI>}
\Pr \Big\{ \bphi_N < \underline{\eps}_\bk \text{ or } \bphi_N > \overline{\eps}_\bk \Big\} = \sum_{k=0}^N {N \choose k} \int_{[0,\underline{\eps}_k)
\cup (\overline{\eps}_k,1]} \dd \mathfrak{m}^+_{k,N}.
\end{equation}
One main difference arises next in connection with (i) and (ii): while (i) holds as before, (ii) holds in this context with equality, which we write in the following way:
\begin{itemize}
\item[\textbf{(ii)$'$}] For $m=0,1,\ldots$ and $k=0,\ldots,m$, it holds that
$$
\mathfrak{m}^+_{k,m+1} - (1-\alpha) \mathfrak{m}^+_{k,m} = 0.
$$
\end{itemize}
\begin{itemize}
\item[] {\bf Proof of (ii)$'$:} Follow the derivation of (ii) till equation
\eqref{eq:preference_gives<=}. Next, we prove that, in the present context of Theorem \ref{th:compression_2}, equation \eqref{eq:preference_gives<=} also holds with reversed inequality, so proving that the two sides of \eqref{eq:preference_gives<=} are in fact equal. To see this, notice that in the event under the sign of probability in the right-hand side of \eqref{eq:preference_gives<=} it holds that $\co(z_1,\ldots,z_{m}) = \ms(z_1,\ldots,z_k)$, which, owing to the \emph{preference} Property \ref{preference},
implies
$$
\co(z_1,\ldots,z_k,z_i) = \ms(z_1,\ldots,z_k), \quad \quad i = k+1, \ldots, m,
$$
and 
$$
\co(z_1,\ldots,z_k) = \ms(z_1,\ldots,z_k). 
$$
In addition, in the same event it also holds that (simply substitute $\co(z_1,\ldots,z_{m})$ with $\ms(z_1,\ldots,z_k)$ in the first condition that defines the event)
$$
\co(z_1,\ldots,z_k,z_{m+1}) = \ms(z_1,\ldots,z_k). \nonumber
$$
Hence,
\begin{align*}
& \Pr \Big\{ \co(\co(\bz_1,\ldots,\bz_m),\bz_{m+1}) = \co(\bz_1,\ldots,\bz_m)  \\
& \quad \quad \text{ and } \co(\bz_1,\ldots,\bz_{m}) = \ms(\bz_1,\ldots,\bz_k) \text{ and } \bphi_m \in B \Big\}  \nonumber \\
& \quad \leq \Pr \Big\{ \co(\bz_1,\ldots,\bz_k,\bz_i) = \co(\bz_1,\ldots,\bz_k), \ \ i = k+1, \ldots m+1 \\
& \quad \quad \text{ and } \co(z_1,\ldots,z_k) = \ms(z_1,\ldots,z_k) \text{ and } \bphi_m \in B \Big\} \nonumber \\
& \quad \leq \Pr \Big\{ \co(\bz_1,\ldots,\bz_{m+1}) = \ms(\bz_1,\ldots,\bz_k) \text{ and } \bphi_m \in B \Big\}, \nonumber
\end{align*}
where the last inequality follows from the \emph{non-associativity} Property \ref{non-associativity}. This establishes the reversed inequality of \eqref{eq:preference_gives<=} and, therefore, that \eqref{eq:preference_gives<=} and  \eqref{eq:intB=Pr} hold with equality. The
final part of the proof of (ii)$'$ consists in re-writing the right-hand side of \eqref{eq:intB=Pr} as is done in the proof of (ii). \qed
\end{itemize}
We are now ready to upper-bound $\Pr \Big\{ \bphi_N < \underline{\eps}_\bk \text{ or } \bphi_N > \overline{\eps}_\bk \Big\}$ by taking the $\sup$ of \eqref{PPHI<PHI>} under conditions (i) and (ii)$'$ (in addition to the fact that measures $\mathfrak{m}^+_{k,m}$ belong to the cone $\mathcal{M}^+$ of positive finite measures on $[0,1]$). This gives
\begin{equation} \label{eq:Pr_ch_comp<=gamma-2}
\Pr \Big\{ \bphi_N < \underline{\eps}_\bk \text{ or } \bphi_N > \overline{\eps}_\bk \Big\} \leq \gamma,
\end{equation}
where $\gamma$ is defined as the value of the optimization problem
\begin{eqnarray*}
\gamma = \sup_{\substack{\mathfrak{m}^+_{k,m} \in \mathcal{M}^+ \\ m=0,1,\ldots, \; \; k=0,\ldots,m }}
& & \sum_{k=0}^N {N \choose k} \int_{[0,\underline{\eps}_k) \cup (\overline{\eps}_k,1]} \dd \mathfrak{m}^+_{k,N}  \\
\textrm{subject to:} & & \sum_{k=0}^m {m \choose k} \int_{[0,1]} \dd \mathfrak{m}^+_{k,m} = 1, \quad m=0,1,\ldots \nonumber \\
& & \mathfrak{m}^+_{k,m+1} - (1-\alpha) \mathfrak{m}^+_{k,m} = 0, \quad
m=0,1,\ldots; \; k=0,\ldots,m. \nonumber
\end{eqnarray*}
To evaluate $\gamma$, we consider as before a truncated version of the problem
\begin{subequations} \label{eq:primal M-Theorem2}
\begin{eqnarray}
\gamma_M = \sup_{\substack{\mathfrak{m}^+_{k,m} \in \mathcal{M}^+ \\ m=0,\ldots, M, \; \; k=0,\ldots,m }}
& & \sum_{k=0}^N {N \choose k} \int_{[0,\underline{\eps}_k) \cup (\overline{\eps}_k,1]} \dd \mathfrak{m}^+_{k,N}  \label{eq:primal_M2_cost} \\
\textrm{subject to:} & & \sum_{k=0}^m {m \choose k} \int_{[0,1]} \dd \mathfrak{m}^+_{k,m} = 1, \quad m=0,\ldots,M  \label{eq:primal_M2_equality_constr} \\
& & \mathfrak{m}^+_{k,m+1} - (1-\alpha) \mathfrak{m}^+_{k,m} = 0, \nonumber \\
& & m=0,\ldots,M-1; \; k=0,\ldots,m. \label{eq:primal_M2 inequality_constr}
\end{eqnarray}
\end{subequations} and observe that
\begin{equation}
\label{gammaM>gamma-Theorem2}
\gamma \leq \gamma_M,
\end{equation}
for all $M$. In evaluating $\gamma_M$ by dualization one important difference with the proof of Theorem \ref{th:compression_1} occurs in the Lagrangian
\begin{eqnarray} \label{eq:Lagrangian-Theorem2}
\mathfrak{L} & = & \sum_{k=0}^N {N \choose k} \int_{[0,\underline{\eps}_k) \cup (\overline{\eps}_k,1]} \dd \mathfrak{m}^+_{k,N}
- \sum_{m=0}^M \lambda_m \left( \sum_{k=0}^m {m \choose k} \int_{[0,1]} \dd \mathfrak{m}^+_{k,m} - 1 \right) \nonumber \\
& & - \sum_{m=0}^{M-1} \sum_{k=0}^m \int_{[0,1]} \mu_{k,m}(\alpha) \; \dd [ \mathfrak{m}^+_{k,m+1} - (1-\alpha) \mathfrak{m}^+_{k,m}],
\end{eqnarray}
because functions $\mu_{k,m} \in \textsf{C}^0[0,1]$ are now required only
to be continuous, while their sign is arbitrary (this difference stems from the equality condition on measures in (ii)$'$ as opposed to the inequality condition in (ii)). \\
\\
Tantamount to \eqref{equality-duality}, we here have
\begin{equation}
\label{equality-duality-Theorem2}
\gamma_M
~\stackrel{\mathrm{(A)}}{=}~
\sup_{\{\mathfrak{m}^+_{k,m}\}} \inf_{\substack{\{\lambda_m\} \\ \{ \mu_{k,m} \} }} \mathfrak{L}
~\stackrel{\mathrm{(B)}}{=}~
\inf_{\substack{\{\lambda_m\} \\ \{ \mu_{k,m} \} }} \sup_{ \{\mathfrak{m}^+_{k,m}\} } \mathfrak{L}
~\stackrel{\mathrm{(C)}}{=}~
\gamma^\ast_M,
\end{equation}
where $\gamma^\ast_M$ is the value of the dual of problem \eqref{eq:primal M-Theorem2}:
\begin{subequations} \label{eq:dual_M-Theorem2}
\begin{eqnarray}
\gamma^\ast_M = \inf_{\substack{\lambda_m, \; m=0,\ldots,M \\ \mu_{k,m} \in \textsf{C}^0[0,1], \; m=0,\ldots,M-1, \; \; k=0,\ldots,m }} & & \sum_{m=0}^{M}  \lambda_m  \label{eq:dual_M_cost-Theorem2} \\
\textrm{subject to:} & & {m \choose k} \One{\alpha \in [0,\underline{\eps}_k) \cup (\overline{\eps}_k,1]} \One{m=N} + (1 \! - \! \alpha) \mu_{k,m}(\alpha) \One{m \neq M}  \nonumber \\
& & \leq  \lambda_m {m \choose k} + \mu_{k,m-1}(\alpha) \One{m \neq k}, \quad \forall \alpha \in [0,1], \nonumber \\
& & k = 0, \ldots, M, \; m = k, \ldots, M. \label{eq:dual_M_constr-Theorem2}
\end{eqnarray}
\end{subequations}
\begin{itemize}
\item[] \textbf{Proof of (A) in \eqref{equality-duality-Theorem2}:} If measures ${\mathfrak{m}^+_{k,m}}$ do not satisfy the constraints in \eqref{eq:primal_M2_equality_constr} and \eqref{eq:primal_M2 inequality_constr}, then $\inf_{\{\lambda_m\}, \{\mu_{k,m}\}} \mathfrak{L}$ is equal to $-\infty$. The reason why this is true for \eqref{eq:primal_M2_equality_constr} is the same as the reason why this is true for \eqref{eq:primal_M_equality_constr} in the proof of (A) in Theorem \ref{th:compression_1}. As for \eqref{eq:primal_M2 inequality_constr}, note that in the present context functions $\mu_{k,m}$ have more flexibility than $\mu^+_{k,m}$ in Theorem \ref{th:compression_1} because they need not be positive. By concentrating on $\mu_{k,m}$'s that are indeed positive, we have as before that $\mathfrak{m}^+_{k,m+1} - (1-\alpha) \mathfrak{m}^+_{k,m} \in \mathcal{M}^-$; similarly, with negative $\mu_{k,m}$'s one concludes that $\mathfrak{m}^+_{k,m+1} - (1-\alpha) \mathfrak{m}^+_{k,m} \in \mathcal{M}^+$, and these two facts together imply \eqref{eq:primal_M2 inequality_constr}. To close the proof of (A), we simply
notice that the Lagrangian \eqref{eq:Lagrangian-Theorem2} with \eqref{eq:primal_M2_equality_constr} and \eqref{eq:primal_M2 inequality_constr} in place reduces to \eqref{eq:primal_M2_cost}. \qed
\item[] \textbf{Proof of (B) in \eqref{equality-duality-Theorem2}:} As in
the proof of Theorem \ref{th:compression_1}, matters of convenience suggest to introduce a modified Lagrangian $\mathfrak{L}_\tau$ that corresponds to a continuous cost function. To this aim, for $k = 0,1,\ldots,N$, the integral $\int_{[0,\underline{\eps}_k) \cup (\overline{\eps}_k,1]} \dd \mathfrak{m}^+_{k,N}$ in the first term of the Lagrangian is rewritten as $\int_{[0,1]}\One{\alpha \in [0,\underline{\eps}_k) \cup (\overline{\eps}_k,1]} \dd \mathfrak{m}^+_{k,N}$ and the indicator function $\One{\alpha \in [0,\underline{\eps}_k) \cup (\overline{\eps}_k,1]}$ is replaced with a continuous function $\varphi_{k,\tau}(\alpha)$ that perturbs the Lagrangian in a vanishing way as $\tau \to \infty$. Precisely, to obtain a continuous transition, we tilt the edges of the indicator function by a small enough quantity $\tau$ (which creates linear slopes over the intervals $(\underline{\eps}_k, \underline{\eps}_k + \tau]$ and $[\overline{\eps}_k - \tau, \overline{\eps}_k)$ while leaving the indicator function unaltered for other values of $\alpha$) with the only advice that: if $\underline{\eps}_k = 0$ (so that the left edge does not exist), then $\varphi_{k,\tau}(\alpha)$ continues at the value $0$ till $\alpha = 0$; likewise, $\varphi_{k,\tau}(\alpha)$ continues at value $0$ till $\alpha = 1$ if $\overline{\eps}_k = 1$. The modified Lagrangian is
\begin{eqnarray*}
\mathfrak{L}_\tau & = & \sum_{k=0}^N {N \choose k} \int_{[0,1]} \varphi_{k,\tau}(\alpha) \; \dd \mathfrak{m}^+_{k,N}
- \sum_{m=0}^M \lambda_m \left( \sum_{k=0}^m {m \choose k} \int_{[0,1]} \dd \mathfrak{m}^+_{k,m} - 1 \right) \nonumber \\
& & - \sum_{m=0}^{M-1} \sum_{k=0}^m \int_{[0,1]} \mu_{k,m}(\alpha) \;
\dd \! \left[  \mathfrak{m}^+_{k,m+1} - (1-\alpha) \mathfrak{m}^+_{k,m} \right].
\end{eqnarray*}
In full analogy with \eqref{fundamnetal relations} in Theorem \ref{th:compression_1}, the proof consists in showing the validity of the following relations:
\begin{equation}
\label{fundamnetal relations-2}
\begin{array}{ccc}
\sup_{\{\mathfrak{m}^+_{k,m}\}} \inf_{\substack{\{\lambda_m\} \\ \{\mu_{k,m}\}}} \mathfrak{L}_\tau &
= & \inf_{\substack{\{\lambda_m\} \\ \{\mu_{k,m}\}}} \sup_{\{\mathfrak{m}^+_{k,m}\}} \mathfrak{L}_\tau \\
\downarrow_{\tau \downarrow 0} & & \rotgeq \\
\sup_{\{\mathfrak{m}^+_{k,m}\}} \inf_{\substack{\{\lambda_m\}\\ \{\mu_{k,m}\}}} \mathfrak{L} &
\leq  & \inf_{\substack{\{\lambda_m\} \\ \{\mu_{k,m}\}}} \sup_{\{\mathfrak{m}^+_{k,m}\}} \mathfrak{L},
\end{array}
\end{equation}
where the only difference with \eqref{fundamnetal relations} is that the positive functions $\mu^+_{k,m}$ are now the functions $\mu_{k,m}$ that are undefined in sign. Here, as in Theorem \ref{th:compression_1}, we need
only to show the validity of the $=$ at top and the convergence $\downarrow_{\tau \downarrow 0}$ on the left. \\
\\
To show the validity of the top equality
\begin{equation} \label{eq:supinf=infsup_tau-2}
\sup_{\{\mathfrak{m}^+_{k,m}\}} \inf_{\substack{\{\lambda_m\} \\ \{\mu_{k,m}\}}} \mathfrak{L}_\tau = \inf_{\substack{\{\lambda_m\} \\ \{\mu_{k,m}\}}} \sup_{\{\mathfrak{m}^+_{k,m}\}} \mathfrak{L}_\tau,
\end{equation}
one follows the same argument as in Theorem \ref{th:compression_1} after noting that there is no need here to introduce the positive measures $\mathfrak{p}^+_{k,m}$ (in Theorem \ref{th:compression_1}, the $\mathfrak{p}^+_{k,m}$'s served the purpose of making null the measures $\mathfrak{q}_{k,m}  = \mathfrak{m}^+_{k,m+1} - (1-\alpha) \; \mathfrak{m}^+_{k,m} + \mathfrak{p}^+_{k,m}$ to evaluate the value $V$ and the supervalue $\bar{V}$; this is not needed here because $\mathfrak{m}^+_{k,m+1} - (1-\alpha) \; \mathfrak{m}^+_{k,m}$ is downright zero and not just negative). Hence, for precise reference, we make explicit that set $H$ in the present context becomes
\begin{eqnarray}
H & := & \Big\{(v,\{r_m\},\big\{ \mathfrak{q}_{k,m} \big\}) \in \R \times \R^{M+1} \times \mathcal{M}^{\frac{(M+1)M}{2}} : \nonumber \\
& & \quad \quad \quad v = \sum_{k=0}^N {N \choose k} \int_{[0,1]} \varphi_{k,\tau}(\alpha) \ \dd \mathfrak{m}^+_{k,N}, \nonumber \\
& & \quad \quad \quad \{r_m\} = \left\{ \sum_{k=0}^m {m \choose k} \int_{[0,1]} \dd \mathfrak{m}^+_{k,m} - 1 \right\}, \nonumber \\
& & \quad \quad \quad \big\{ \mathfrak{q}_{k,m} \big\} = \big\{\mathfrak{m}^+_{k,m+1} - (1-\alpha) \; \mathfrak{m}^+_{k,m} \big\}, \nonumber \\
& & \quad \mbox{where, for all } m \mbox{ and } k, \ \ \mathfrak{m}^+_{k,m} \in \mathcal{M}^+ \Big\}. \label{definition H-2}
\end{eqnarray}
All derivations from here till the equivalent of equation \eqref{eq:inf_sup<=Vbar_1} are identical to those developed in Theorem \ref{th:compression_1} with the only notice that functions $\mu_{k,m}^\varepsilon$ must not be non-negative in the present context (hence, delete from the derivations in Theorem \ref{th:compression_1} the sentence ``Moreover, noting $\ldots$ in place of $\mu_{k,m}^\varepsilon$.''). This way one arrives in the present context to the following equivalent of \eqref{eq:inf_sup<=Vbar_1}, which only differs from \eqref{eq:inf_sup<=Vbar_1} because functions $\mu_{k,m}$, which have undefined sign, take the place of $\mu^+_{k,m}$:
\begin{equation} \label{eq:inf_sup<=Vbar_1-2}
\inf_{\substack{\{\lambda_m\} \\ \{\mu_{k,m}\}}} \sup_{(v, \{ r_m\}, \{ \mathfrak{q}_{k,m} \}) \in H} \;
\left\{ v - \sum_{m=0}^M \lambda_m r_m - \sum_{m=0}^{M-1} \sum_{k=0}^m \int_{[0,1]} \mu_{k,m}(\alpha) \; \dd \mathfrak{q}_{k,m} \right\}
\leq
\bar{V}.
\end{equation}
By recalling the expression of $v$, $r_m$, $\mathfrak{q}_{k,m}$ in the definition of $H$ given in \eqref{definition H-2} and
noticing that the curly bracket in the left-hand side is nothing but $\mathfrak{L}_\tau$ (in Theorem \ref{th:compression_1} we had to digress to take care of measures $\mathfrak{p}^+_{k,m}$), \eqref{eq:inf_sup<=Vbar_1-2} immediately gives the counterpart of \eqref{infsupL<V}:
$$
\inf_{\substack{\{\lambda_m\} \\ \{\mu_{k,m}\}}} \sup_{\{\mathfrak{m}^+_{k,m}\}}
\; \mathfrak{L}_\tau \leq \bar{V}.
$$
The last portion after \eqref{infsupL<V} holds unaltered in the present context to conclude that 
$$
	\inf_{\substack{\{\lambda_m\} \\ \{\mu_{k,m}\}}} \sup_{\{\mathfrak{m}^+_{k,m}\}}
	\; \mathfrak{L}_\tau = \bar{V}.
$$
The next step that $V = \bar{V}$ becomes simpler in the present context. Indeed, one has just to suppress $\mathfrak{p}^{+,i}_{k,m}$ wherever encountered to show the validity of equations \eqref{eq:m_bar_attains_cost} and \eqref{eq:m_bar_attains_=_constr}, while, by a derivation similar to that used to obtain \eqref{eq:m_bar_attains_cost} and \eqref{eq:m_bar_attains_=_constr}, one obtains from \eqref{eq:weak_conv=0_i=infty} (without $\mathfrak{p}^+_{k,m}$) relation
\begin{align}
& \int_{[0,1]} g_j(\alpha) \; \dd [\bar{\mathfrak{m}}^+_{k,m+1} - (1-\alpha) \; \bar{\mathfrak{m}}^+_{k,m}] = 0, \nonumber \\
& \forall g_j(\alpha), \; j=1,2,\ldots, \ \ \ m=0,1,\ldots,M-1, \
\ \ k=0,\ldots,m. \label{eq:nonpositivity_of_integral_seq-2}
\end{align}
The last part now becomes: Taking now any function $f$ in $\textsf{C}^0[0,1]$ and noting that $f$ can be arbitrarily approximated in the $\sup$ norm by a function $g_j$, \eqref{eq:nonpositivity_of_integral_seq-2} yields
$$
\int_{[0,1]} f(\alpha) \; \dd [\bar{\mathfrak{m}}^+_{k,m+1} - (1-\alpha) \; \bar{\mathfrak{m}}^+_{k,m}] = 0,
$$
from which
\begin{equation}
\label{m-satisfied}
\bar{\mathfrak{m}}^+_{k,m+1} - (1-\alpha) \; \bar{\mathfrak{m}}^+_{k,m} =
0
\end{equation}
(recall Footnote \ref{key_footnote}). In the light of \eqref{eq:m_bar_attains_cost}, \eqref{eq:m_bar_attains_=_constr} and \eqref{m-satisfied} one sees that $\{\bar{\mathfrak{m}}^+_{k,m} \}$ maps into the point $(\bar{V},\{r_m = 0 \},\big\{ \mathfrak{q}_{k,m} = 0 \big\})$, which proves
that this point is in $H$. Hence, it holds that $V = \bar{V}$ and equation \eqref{eq:supinf=infsup_tau-2} remains proven. \\
\\
Turning now to equation
\begin{equation} \label{eq:limsupinfLtau=supinfL-2}
\lim_{\tau \to 0} \; \sup_{\{\mathfrak{m}^+_{k,m}\}} \inf_{\substack{\{\lambda_m\} \\ \{\mu_{k,m}\}}} \mathfrak{L}_\tau = \sup_{\{\mathfrak{m}^+_{k,m}\}} \inf_{\substack{\{\lambda_m\} \\ \{\mu_{k,m}\}}} \mathfrak{L}
\end{equation}
(which is the only remaining relation to prove in \eqref{fundamnetal relations-2}), we notice that this is a step that contains some major differences from Theorem \ref{th:compression_1}, for which reason we prefer to repeat the whole derivation even at the price of duplicating some parts already contained in the proof of Theorem \ref{th:compression_1}. \\
\\
Notice that, in both sides of \eqref{eq:limsupinfLtau=supinfL-2}, the $\inf$ operator sends the value to $-\infty$ whenever the constraints in \eqref{eq:primal_M2_equality_constr} or \eqref{eq:primal_M2 inequality_constr} are not satisfied by $\{ \mathfrak{m}^+_{k,m}\}$: hence, \eqref{eq:primal_M2_equality_constr} and \eqref{eq:primal_M2 inequality_constr} must
be satisfied and are always assumed from now on. Under \eqref{eq:primal_M2_equality_constr} and \eqref{eq:primal_M2 inequality_constr}, \eqref{eq:limsupinfLtau=supinfL-2} is rewritten as
\begin{eqnarray} \label{lim_sup_cost_tau=lim_sup_cost-2}
\lefteqn{ \lim_{\tau \to 0} \; \sup_{\{\mathfrak{m}^+_{k,m}\}}
\sum_{k=0}^N {N \choose k} \int_{[0,1]} \varphi_{k,\tau}(\alpha) \; \dd \mathfrak{m}^+_{k,N} } \nonumber \\
& = & \sup_{\{\mathfrak{m}^+_{k,m}\}}
\sum_{k=0}^N {N \choose k} \int_{[0,1]} \One{\alpha \in [0,\underline{\eps}_k) \cup (\overline{\eps}_k,1]} \; \dd \mathfrak{m}^+_{k,N}.
\end{eqnarray}
To show the validity of \eqref{lim_sup_cost_tau=lim_sup_cost-2}, we discretize $\tau$ into $\tau_i$, $i=1,2,\ldots$, $\tau_i \to 0$, and consider a sequence $\{ \breve{\mathfrak{m}}^{+,i}_{k,m} \}$, $i=1,2,\ldots$ (where measures $\breve{\mathfrak{m}}^{+,i}_{k,m}$ satisfy \eqref{eq:primal_M2_equality_constr} and \eqref{eq:primal_M2 inequality_constr} for
any $i$), such that
$$
\lim_{i\to \infty} \sum_{k=0}^N {N \choose k} \int_{[0,1]} \varphi_{k,\tau_i}(\alpha) \; \dd \breve{\mathfrak{m}}^{+,i}_{k,N}
$$
equals the left-hand side of \eqref{lim_sup_cost_tau=lim_sup_cost-2} (for this to hold, $\breve{\mathfrak{m}}^{+,i}_{k,m}$ must achieve a progressively closer and closer approximation of $\sup_{\{\mathfrak{m}^+_{k,m}\}}$ in the left-hand side of \eqref{lim_sup_cost_tau=lim_sup_cost-2} as $i$ increases); then, we construct from $\{ \breve{\mathfrak{m}}^{+,i}_{k,m} \}$ a new sequence $\{ \tilde{\mathfrak{m}}^{+,i}_{k,m} \}$, $i=1,2,\ldots$ (still satisfying \eqref{eq:primal_M2_equality_constr} and \eqref{eq:primal_M2 inequality_constr}), such that
\begin{equation} \label{eq:lim_cost_phi_m<=lim_cost_One_m_tilde-2}
\lim_{i\to \infty} \sum_{k=0}^N {N \choose k} \int_{[0,1]} \varphi_{k,\tau_i}(\alpha) \; \dd \breve{\mathfrak{m}}^{+,i}_{k,N} \leq \lim_{i\to \infty} \sum_{k=0}^N {N \choose k} \int_{[0,1]} \One{\alpha \in [0,\underline{\eps}_k) \cup (\overline{\eps}_k,1]} \; \dd \tilde{\mathfrak{m}}^{+,i}_{k,N},
\end{equation}
which shows that the left-hand side of \eqref{lim_sup_cost_tau=lim_sup_cost-2} is upper-bounded by a value that is no bigger than the right-hand side
of \eqref{lim_sup_cost_tau=lim_sup_cost-2}. Since, on the other hand, the left-hand side of \eqref{lim_sup_cost_tau=lim_sup_cost-2} cannot be smaller than the right-hand side of \eqref{lim_sup_cost_tau=lim_sup_cost-2} because $\varphi_{k,\tau}(\alpha) \geq \One{\alpha \in [0,\underline{\eps}_k) \cup (\overline{\eps}_k,1]}$, \eqref{lim_sup_cost_tau=lim_sup_cost-2} remains proven.
\\
\\
The construction of $\{ \tilde{\mathfrak{m}}^{+,i}_{k,m} \}$ is in three steps:
\begin{itemize}
	\item[Step 1.] [construction of $\{ \check{\mathfrak{m}}^{+,i}_{k,m} \}$] For all $k \leq N$ for which $\overline{\eps}_k \neq 1$ and for all $m$, move the probabilistic mass of $\breve{\mathfrak{m}}^{+,i}_{k,m}$ contained in the interval $(\overline{\eps}_k-\tau_i,\overline{\eps}_k]$ into a concentrated mass in point $\overline{\eps}_k+\tau_i$ and, for all $k \leq N$ for which $\underline{\eps}_k \neq 0$ and for all $m$, move the probabilistic mass of $\breve{\mathfrak{m}}^{+,i}_{k,m}$ contained in the interval $[\underline{\eps}_k,\underline{\eps}_k+\tau_i)$ into a concentrated mass in point $\underline{\eps}_k-\tau_i$; let $\check{\mathfrak{m}}^{+,i}_{k,m}$ be the corresponding measures.
 	\item[Step 2.] [construction of $\{ \hat{\mathfrak{m}}^{+,i}_{k,m} \}$]
The mass shift in Step 1 can lead to measures $\check{\mathfrak{m}}^{+,i}_{k,m}$ that violate condition \eqref{eq:primal_M2 inequality_constr} in $\overline{\eps}_k+\tau_i$ and/or $\underline{\eps}_k-\tau_i$; the new measures $\hat{\mathfrak{m}}^{+,i}_{k,m}$ restore the validity of condition \eqref{eq:primal_M2 inequality_constr}. The construction is in two steps that focus on $\overline{\eps}_k+\tau_i$ and $\underline{\eps}_k-\tau_i$, respectively. For all $k > N$ and all $k \leq N$ for which $\overline{\eps}_k =
1$, let $\hat{\mathfrak{m}}^{+,i(1)}_{k,m} = \check{\mathfrak{m}}^{+,i}_{k,m}$, for all $m = k, \ldots, M$. For all other $k$'s, let $\hat{\mathfrak{m}}^{+,i(1)}_{k,k} = \check{\mathfrak{m}}^{+,i}_{k,k}$; then, verify sequentially for $m = k, \ldots, M-1$ whether the condition
	\begin{equation}
    \label{m(1)}
	\check{\mathfrak{m}}^{+,i}_{k,m+1}(\{\overline{\eps}_k+\tau_i\}) - (1-(\overline{\eps}_k+\tau_i)) \; \hat{\mathfrak{m}}^{+,i(1)}_{k,m}(\{\overline{\eps}_k+\tau_i\}) = 0
	\end{equation}
is satisfied; if yes, let $\hat{\mathfrak{m}}^{+,i(1)}_{k,m+1} = \check{\mathfrak{m}}^{+,i}_{k,m+1}$, otherwise trim $\check{\mathfrak{m}}^{+,i}_{k,m+1}(\{\overline{\eps}_k+\tau_i\})$ to the value $(1-(\overline{\eps}_k+\tau_i)) \; \hat{\mathfrak{m}}^{+,i(1)}_{k,m}(\{\overline{\eps}_k+\tau_i\})$ and define $\hat{\mathfrak{m}}^{+,i(1)}_{k,m+1}$ as the trimmed version of $\check{\mathfrak{m}}^{+,i}_{k,m+1}$.\footnote{Note that, if \eqref{m(1)} is violated, its left-hand side is necessarily greater than zero (so that the ``trimming'' operation achieves its intended goal). Reason is that the initial masses in $(\overline{\eps}_k-\tau_i,\overline{\eps}_k]$ are balanced (i.e., $\breve{\mathfrak{m}}^{+,i}_{k,m+1}(\overline{\eps}_k-\tau_i,\overline{\eps}_k] - \int_{(\overline{\eps}_k-\tau_i,\overline{\eps}_k]} (1 - \alpha) \dd \breve{\mathfrak{m}}^{+,i}_{k,m} = 0$) and the shift to $\overline{\eps}_k+\tau_i$ reduces the coefficient $(1 - \alpha)$ in the second term, which is the negative one.} Likewise, we need to restore validity of condition \eqref{eq:primal_M2 inequality_constr} in $\underline{\eps}_k-\tau_i$. Since we again want to trim -- i.e., reducing rather than raising -- measures (for reasons that will become clear in Step 3) and in this case the mass shift is leftward, we are well-advised to scan the values of $m$ from bottom to top. For all $k > N$ and all $k \leq N$ for which $\underline{\eps}_k = 0$, let $\hat{\mathfrak{m}}^{+,i}_{k,m} = \hat{\mathfrak{m}}^{+,i(1)}_{k,m}$, for all $m = k, \ldots, M$. For all other $k$'s, let $\hat{\mathfrak{m}}^{+,i}_{k,M} = \hat{\mathfrak{m}}^{+,i(1)}_{k,M}$; then, verify sequentially for $m = M-1, M-2 \ldots, k$ whether the condition
$$
\hat{\mathfrak{m}}^{+,i}_{k,m+1}(\{\underline{\eps}_k-\tau_i\}) - (1-(\underline{\eps}_k-\tau_i)) \; \hat{\mathfrak{m}}^{+,i(1)}_{k,m}(\{\underline{\eps}_k-\tau_i\}) = 0
$$
is satisfied; if yes, let $\hat{\mathfrak{m}}^{+,i}_{k,m} = \hat{\mathfrak{m}}^{+,i(1)}_{k,m}$, otherwise trim $\hat{\mathfrak{m}}^{+,i(1)}_{k,m}(\{\underline{\eps}_k-\tau_i\})$ to the value $\hat{\mathfrak{m}}^{+,i}_{k,m+1}(\{\underline{\eps}_k-\tau_i\})/(1-(\underline{\eps}_k-\tau_i)) \;
$ and define $\hat{\mathfrak{m}}^{+,i}_{k,m}$ as the trimmed version of $\hat{\mathfrak{m}}^{+,i(1)}_{k,m}$.
\item[Step 3.] [construction of $\{ \tilde{\mathfrak{m}}^{+,i}_{k,m} \}$] The trimming operation in Step 2 may have unbalanced some equalities in \eqref{eq:primal_M2_equality_constr}, i.e., it may be that
	$$
	\sum_{k=0}^m {m \choose k} \int_{[0,1]} \dd \hat{\mathfrak{m}}^{+,i}_{k,m} < 1
	$$
for some $m$. If so, re-gain balance by adding to measure $\hat{\mathfrak{m}}^{+,i}_{m,m}$ a suitable probabilistic mass concentrated in $\alpha=1$, while leaving other measures $\hat{\mathfrak{m}}^{+,i}_{k,m}$, $k \neq m$, unaltered. The so-obtained measures are $\tilde{\mathfrak{m}}^{+,i}_{k,m}$. Note that this operation preserves the validity of condition $\tilde{\mathfrak{m}}^{+,i}_{m,m+1} - (1-\alpha) \; \tilde{\mathfrak{m}}^{+,i}_{m,m} = 0$ (since we have altered measures $\hat{\mathfrak{m}}^{+,i}_{m,m}$ only in $\alpha = 1$ where coefficient $(1 - \alpha)$ is null),
so that $\{ \tilde{\mathfrak{m}}^{+,i}_{k,m} \}$ satisfies \eqref{eq:primal_M2 inequality_constr} besides \eqref{eq:primal_M2_equality_constr}.
\end{itemize}
Since the mass shift in Step 1 has only moved masses into points where $\varphi_{k,\tau_i}(\alpha)$ is bigger, this mass shift can only increase $\sum_{k=0}^N {N \choose k} \int_{[0,1]} \varphi_{k,\tau_i}(\alpha) \; \dd \breve{\mathfrak{m}}^{+,i}_{k,N}$; moreover, any trimming and re-balancing in Steps 2 and 3 involve vanishing masses as $\tau_i \to 0$. Therefore,
\begin{equation} \label{eq:lim_cost_phi_m<=lim_cost_phi_m_tilde-2}
\lim_{i\to \infty} \sum_{k=0}^N {N \choose k} \int_{[0,1]} \varphi_{k,\tau_i}(\alpha) \; \dd \breve{\mathfrak{m}}^{+,i}_{k,N} \leq \lim_{i\to \infty} \sum_{k=0}^N {N \choose k} \int_{[0,1]} \varphi_{k,\tau_i}(\alpha) \; \dd \tilde{\mathfrak{m}}^{+,i}_{k,N}.
\end{equation}
On the other hand, by construction, $\varphi_{k,\tau_i}(\alpha) = \One{\alpha \in [0,\underline{\eps}_k) \cup (\overline{\eps}_k,1]}$ if $\underline{\eps}_k = 0$ and $\overline{\eps}_k = 1$, while, for $\underline{\eps}_k \neq 0$ and/or $\overline{\eps}_k \neq 1$, $\varphi_{k,\tau_i}(\alpha) \neq \One{\alpha \in [0,\underline{\eps}_k) \cup (\overline{\eps}_k,1]}$ only occurs where $\tilde{\mathfrak{m}}^{+,i}_{k,N}$ is null. Hence,
$$
\sum_{k=0}^N {N \choose k} \int_{[0,1]} \varphi_{k,\tau_i}(\alpha) \; \dd \tilde{\mathfrak{m}}^{+,i}_{k,N} = \sum_{k=0}^N {N \choose k} \int_{[0,1]} \One{\alpha \in [0,\underline{\eps}_k)\cup(\overline{\eps}_k,1]} \; \dd \tilde{\mathfrak{m}}^{+,i}_{k,N},
$$
which, substituted in \eqref{eq:lim_cost_phi_m<=lim_cost_phi_m_tilde-2}, gives \eqref{eq:lim_cost_phi_m<=lim_cost_One_m_tilde-2}.
This concludes the proof of (B). \qed
\item[] \textbf{Proof of (C) in \eqref{equality-duality-Theorem2}}: The proof of point (C) follows, \emph{mutatis mutandis}, that of Theorem \ref{th:compression_1}. Here, the Lagrangian can be re-written as
\begin{eqnarray*} 
\mathfrak{L} & = & \sum_{m=0}^M \lambda_m + \sum_{k=0}^M \sum_{m=k}^M \int_{[0,1]} \Bigg[ {m \choose k} \One{\alpha \in [0,\underline{\eps}_k) \cup (\overline{\eps}_k,1]} \One{m=N} + (1-\alpha) \mu_{k,m}(\alpha) \One{m \neq M} \nonumber \\
&& - \lambda_m {m \choose k} - \mu_{k,m-1}(\alpha) \One{m \neq k} \Bigg] \; \dd \mathfrak{m}^+_{k,m},
\end{eqnarray*}
and the argument is closed as in Theorem \ref{th:compression_1}, in which derivation one has only to substitute $\mu^+_{k,m}$ with $\mu_{k,m}$. \qed
\end{itemize}
Next we want to evaluate $\gamma^\ast_M$ for problem \eqref{eq:dual_M-Theorem2}. \\
\\
In the present context, the constraints can be rewritten more explicitly as in \eqref{spataffiata} with the only change that all functions $\mu^+_{k,m}$, for any index $m,k$, need not be positive, so that the superscript ``+'' must be dropped everywhere and, moreover, the indicator
function $\One{\alpha \in (\eps_k,1]}$ now becomes $\One{\alpha \in [0,\underline{\eps}_k) \cup (\overline{\eps}_k,1]}$. The discussion that follows \eqref{spataffiata} remains the same with only a major difference: the discussion leading up to the conclusion that one can add the constraints $\lambda_m = 0 \quad \text{for } m = N+1,\ldots, M$ ceases to be valid here because it was grounded on the fact that functions $\mu^+_{k,m}$ were positive, while here
functions $\mu_{k,m}$ do not undergo this requirement. Hence, one comes to the following problem:
\begin{subequations} \label{eq:dual_simple-2}
\begin{eqnarray}
\inf_{\lambda_m, \; m=0,\ldots,M} & & \sum_{m=0}^{M}  \lambda_m  \label{eq:dual_simple_cost-2} \\
\textrm{subject to:} & & {N \choose k} (1-\alpha)^{N-k} \One{\alpha \in [0,\underline{\eps}_k) \cup (\overline{\eps}_k,1]} \leq  \sum_{m=k}^M \lambda_m {m \choose k} (1-\alpha)^{m-k}, \nonumber \\
& & \forall \alpha \in [0,1], \; k = 0,\ldots, N \label{eq:dual_simple_constr_N-2}
\\
& & 0 \leq \sum_{m=k}^M \lambda_m {m \choose k} (1-\alpha)^{m-k}, \quad
\forall \alpha \in [0,1], \nonumber \\
& & k = N+1,\ldots,M  \label{eq:dual_simple_constr_M-2}
\end{eqnarray}
\end{subequations}
and the claim is that it returns the same optimal value $\gamma^\ast_M$ as \eqref{eq:dual_M-Theorem2}. The fact that the optimal value of \eqref{eq:dual_simple-2} is not bigger than the optimal value of \eqref{eq:dual_M-Theorem2} is obvious. The converse result that the optimal value of \eqref{eq:dual_M-Theorem2} is not bigger than the optimal value of \eqref{eq:dual_simple-2} is proven by showing that for any feasible point of \eqref{eq:dual_simple-2} one can find a feasible point of \eqref{eq:dual_M-Theorem2} that attains the same value, which is shown in the following.
\begin{itemize}
\item[] Consider a feasible point of \eqref{eq:dual_simple-2}.  To find the sought feasible point of \eqref{eq:dual_M-Theorem2}, consider the same $\lambda_m$ as those for the feasible point of \eqref{eq:dual_simple-2} and augment them with the functions $\mu_{k,m}$ defined as follows. Consider the inequalities in \eqref{spataffiata} where the $\mu^+_{k,m}$'s are replaced by the $\mu_{k,m}$'s as it must be in the present context. The inequality for $k=M$ and $m=M$ is satisfied in view of \eqref{eq:dual_simple_constr_M-2} for $k=M$. Next, satisfy the inequalities corresponding to $k=0,\ldots,M-1$, $m = \max\{k,N\}+1,\ldots, M$ with equality starting
from bottom and then proceeding upward. This gives:
    \begin{eqnarray}
    \label{eq:moose-2-1}
    \mu_{k,M-1}(\alpha) & = & -\lambda_M {M \choose k}, \nonumber \\
    \mu_{k,M-2}(\alpha) & = & -\lambda_{M-1} {M-1 \choose k} - \lambda_M {M \choose k}(1-\alpha) \nonumber \\
    & \vdots & \\
    \mu_{k,\max\{k,N\}}(\alpha) & = & -\sum_{j=\max\{k,N\}+1}^M \lambda_j {j \choose k}(1-\alpha)^{j-\max\{k,N\}-1}. \nonumber
    \end{eqnarray}
    Note that for $k=N+1,\ldots,M-1$, this choice also satisfies the inequalities for $m=k$ in view of \eqref{eq:dual_simple_constr_M-2}. The expression of $\mu_{k,m}$ over $[0,1)$ for the remaining indexes are instead defined as in Theorem \ref{th:compression_1} by equations \eqref{eq:moose}. We show that choices \eqref{eq:moose-2-1} and \eqref{eq:moose} satisfy over $[0,1)$ the remaining inequalities (those for $k=0,\ldots,N$ and $m = N$). Substituting $\mu_{k,N-1}(\alpha) = \sum_{j=k}^{N-1}\frac{\lambda_j {j \choose k}}{(1-\alpha)^{N-j}}$ and $\mu_{k,N}(\alpha) = -\sum_{j=N+1}^M\lambda_j {j \choose k}(1-\alpha)^{j-N-1}$ in these inequalities gives
    \begin{equation} \label{eq:central_ineq_N-1-2}
    {N \choose k} \One{\alpha \in [0,\underline{\eps}_k) \cup (\overline{\eps}_k,1]} \leq \sum_{j=k}^{N} \lambda_j {j \choose k} \frac{1}{(1-\alpha)^{N-j}} + \sum_{j=N+1}^M\lambda_j {j \choose k}(1-\alpha)^{j-N}.
    \end{equation}
    Equation \eqref{eq:central_ineq_N-1-2} is satisfied because it coincides with \eqref{eq:dual_simple_constr_N-2}. As for $\alpha = 1$, note that functions $\mu_{k,m}$ defined in \eqref{eq:moose} tend to infinity when $\alpha \to 1$. This poses a problem of existence for $\alpha = 1$, which, however, can be easily circumvented as in Theorem \ref{th:compression_1} by truncating the functions $\mu_{k,m}$ in the interval $\alpha \in [1 \! - \! \rho,1]$ at the value $\mu_{k,m}(1 \! - \! \rho)$ to obtain
    $$
    \mu^{\rho}_{k,m}(\alpha) =
    \begin{cases}
    \mu_{k,m}(\alpha) & \alpha < 1-\rho \\
    \mu_{k,m}(1 \! - \! \rho) & \alpha \geq 1-\rho,
    \end{cases}
    $$
    and noting that all the inequalities are satisfied over $[0,1]$ if $\rho$ is chosen small enough.
\end{itemize}
Summarizing the results so far, we have
$$
\Pr \Big\{ \bphi_N < \underline{\eps}_\bk \text{ or } \bphi_N > \overline{\eps}_\bk \Big\} \stackrel{\eqref{eq:Pr_ch_comp<=gamma-2}}{\leq}\gamma \stackrel{\eqref{gammaM>gamma-Theorem2}}{\leq} \gamma_M \stackrel{\eqref{equality-duality-Theorem2}}{=} \gamma^\ast_M,
$$
where $\gamma^\ast_M$ is given by \eqref{eq:dual_simple-2}. Choose now $M =
4N$. The proof of the theorem is concluded by showing that $\gamma^\ast_{4N} \leq \delta$, which is what we do next. \\
\\
Take $\lambda_m = \frac{\delta}{2N}$ for $m=0,\ldots,N-1$, $\lambda_N
= 0$ and $\lambda_m = \frac{\delta}{6N}$ for $m=N+1,\ldots,4N$, so that $\sum_{m=0}^{4N} \lambda_m = \delta$. Inequalities \eqref{eq:dual_simple_constr_M-2}  are clearly satisfied because all $\lambda_m$ are non-negative. The inequalities in \eqref{eq:dual_simple_constr_N-2} for $\alpha = 1$ are satisfied because they become $0 \leq \lambda_k$ (for $k
< N$, the term $(1-\alpha)^{N-k}$ in the left-hand side of \eqref{eq:dual_simple_constr_N-2} annihilates, while for $k=N$ it is the indicator function that annihilates because $\overline{\eps}_N = 1$). For $\alpha \in [0,1)$, \eqref{eq:dual_simple_constr_N-2} can be rewritten as
\begin{eqnarray*}
\One{\alpha \in [0,\underline{\eps}_k) \cup (\overline{\eps}_k,1]} & \leq & \frac{\delta}{2N}\sum_{m=k}^{N-1} \frac{\binom{m}{k}}{\binom{N}{k}} (1-\alpha)^{-(N-m)} + \frac{\delta}{6N}\sum_{m=N+1}^{4N} \frac{\binom{m}{k}}{\binom{N}{k}} (1-\alpha)^{m-N}, \\
& & \quad k = 0,\ldots, N-1, \\
\One{\alpha \in [0,\underline{\eps}_N)} & \leq & \frac{\delta}{6N}\sum_{m=N+1}^{4N} \binom{m}{N} (1-\alpha)^{m-N}, \quad k = N,
\end{eqnarray*}
and are satisfied in view of the definition of $\underline{\eps}_k$ and $\overline{\eps}_k$, see \eqref{underline_epsilonk} and \eqref{overline_epsilonk}. This concludes the proof. \qed 

\begin{remark} \label{rmk:weaker_conditions_for_th2}
Theorem \ref{th:compression_2} preserves its validity under the following slightly weaker assumption than the \emph{non-concentrated mass} Property \ref{no-concentrated-mass} (while maintaining the \emph{preference} and \emph{non-associativity} Properties \ref{preference} and \ref{non-associativity}): \emph{with probability $1$, if for some $n$ an example $z$ appears in $\co(\bz_1,\ldots,\bz_n)$, then it appears in $\co(\bz_1,\ldots,\bz_n)$ as many times as it does in $\ms(\bz_1,\ldots,\bz_n)$.} 

Clearly, the \emph{non-concentrated mass} property implies the assumption stated here because the \emph{non-concentrated mass} property gives that any multiset does not have repetitions with probability $1$. To instead prove that Theorem \ref{th:compression_2} preserves its validity under this extension, one proceeds similarly to the last part of the proof of Theorem \ref{th:compression_1} (where the temporary condition of non-concentrated mass was removed). Precisely, define $\bz'_i$, $\co'$, $\bphi'_N $, and $\bk'$ as it was done there. Then, the \emph{preference} property of $\co'$ holds as shown in the proof of Theorem \ref{th:compression_1}, while the \emph{non-associativity} property of $\co'$ follows from the \emph{non-associativity} of $\co$ along a similar argument. Moreover, $\bz'_i$ satisfies the \emph{non-concentrated mass} property by construction. Using Theorem \ref{th:compression_2} for $\bz'_i$, $\co'$, $\bphi'_N $, and $\bk'$ now gives
$$
\Pr \{ \underline{\eps}_{\bk'} \leq \bphi'_N \leq \overline{\eps}_{\bk'} \} \geq 1-\delta.
$$
Then, to prove Theorem \ref{th:compression_2} for the initial setting, one has to show that $\bk = \bk'$ and $\bphi_N = \bphi'_N$, with probability $1$. $\bk = \bk'$ is obviously true. As for $\bphi_N = \bphi'_N$, in the proof of Theorem \ref{th:compression_1}, it was already noted that $\co(\co(z_1,\ldots,z_N),z_{N+1}) \neq \co(z_1,\ldots,z_N)$ implies that $\co'(\co'(z'_1,\ldots,z'_N),z'_{N+1}) \neq \co'(z'_1,\ldots,z'_N)$. On the other hand, the assumption introduced in the present remark straightforwardly gives that $\co'(\co'(z'_1,\ldots,z'_N),z'_{N+1}) \neq \co'(z'_1,\ldots,z'_N)$ implies that  $\co(\co(z_1,\ldots,z_N),z_{N+1}) \neq \co(z_1,\ldots,z_N)$, except for at most a probability zero event. This gives $\bphi_N = \bphi'_N$, and we therefore have
$$
\Pr \{ \underline{\eps}_{\bk} \leq \bphi_N \leq \overline{\eps}_{\bk} \} = \Pr \{ \underline{\eps}_{\bk'} \leq \bphi'_N \leq \overline{\eps}_{\bk'} \} \geq 1-\delta.
$$
\qed
\end{remark}

\subsection{Proof of Proposition \ref{th:bounds4asympt}} \label{proof_bounds4asympt}

We first establish that $\frac{k}{N} \leq \eps_k \leq \overline{\eps}_k$. The result is obvious for $k=N$ since $\eps_k = 1 = \overline{\eps}_k$. For $k < N$, recall that $\eps_k$ is the unique solution to $\Psi_{k,\delta}(\alpha) = 1$, while $\overline{\eps}_k$ is the solution to $\tilde{\Psi}_{k,\delta}(\alpha) = 1$ bigger than $\frac{k}{N}$. Notice that $\tilde{\Psi}_{k,\delta}(\alpha) = \frac{1}{2} \Psi_{k,\delta}(\alpha) + \nu_{k,\delta}(\alpha)$, where $\nu_{k,\delta}(\alpha) = \frac{\delta}{6N} \sum_{m= N+1}^{4N} \frac{{m \choose k}}{{N \choose k}} (1-\alpha)^{m-N}$. Using the same argument given in Appendix \ref{Appendix_Psi_tilde} to show that $\tilde{\Psi}_{k,\delta}(\frac{k}{N}) < 1$, it is easy to prove that $\nu_{k,\delta}(\frac{k}{N}) \leq \frac{\delta}{2} < \frac{1}{2}$, which, along with the fact that $\nu_{k,\delta}(\alpha)$ is strictly decreasing over the interval of interest $(0,1)$, yields $\nu_{k,\delta}(\alpha) < \frac{1}{2}$ for $\alpha \in [\frac{k}{N},1)$. Since $\eps_k > \frac{k}{N}$ (remember that $\Psi_{k,\delta}(\alpha)$ is monotonically increasing and the argument in Appendix \ref{Appendix_Psi_tilde} can be used again to show that $\Psi_{k,\delta}(\frac{k}{N}) \leq \delta < 1$), we have that 
$$
\begin{array}{rcll}
	\tilde{\Psi}_{k,\delta}(\eps_k) & = & \frac{1}{2} \Psi_{k,\delta}(\eps_k) + \nu_{k,\delta}(\eps_k) & \\
	& = & \frac{1}{2} + \nu_{k,\delta}(\eps_k) & \quad (\text{since } \Psi_{k,\delta}(\eps_k) = 1) \\
	& < & 1. & \quad (\text{since } \nu_{k,\delta}(\eps_k) < \frac{1}{2})
\end{array}
$$
This, along with the fact that $\tilde{\Psi}_{k,\delta}(\alpha)$ is first decreasing and then increasing, gives $\eps_k \leq \overline{\eps}_k$. \\
\\
We next prove equations \eqref{eq:bounds4asympt_up} and \eqref{eq:bounds4asympt_low}. \\
\\
For $k=0,1,\ldots,N-1$, consider the equation (note that the left-hand side is not exactly equal to function $\tilde{\Psi}_{k,\delta}$ in \eqref{Psi-tilde} because a ``$2$'' in the first term has been substituted with a ``$6$'') 
\begin{equation} \label{eq:var_tilde_psi_eq}
	\frac{\delta}{6N} \sum_{m=k}^{N-1} \frac{{m \choose k}}{{N \choose k}} (1-\alpha)^{m-N} + \frac{\delta}{6N} \sum_{m=N+1}^{4N} \frac{{m \choose k}}{{N \choose k}} (1-\alpha)^{m-N} = 1.
\end{equation}
For every $\alpha$ in $(-\infty,1)$, the left-hand side is no bigger than $\tilde{\Psi}_{k,\delta}(\alpha)$ and, as a function of $\alpha$, it has a behavior similar to $\tilde{\Psi}_{k,\delta}$ (i.e., it is first decreasing and then increasing, diverging to $+\infty$ for both $\alpha \to -\infty$ and $\alpha \to 1$). Thus, \eqref{eq:var_tilde_psi_eq} admits exactly two solutions in $(-\infty,1)$, one smaller than $k/N$ and one bigger than $k/N$, and these solutions give a lower bound and an upper bound to $\underline{\eps}_k$ and $\overline{\eps}_k$, respectively. Since $(1-\alpha)^{N-k} \neq 0$ over $(-\infty,1)$, equation \eqref{eq:var_tilde_psi_eq} is equivalent to
\begin{equation} 
	\label{eq:var_tilde_psi_eq_2}
	\frac{\delta}{6N} \sum_{m=k}^{4N} {m \choose k} (1-\alpha)^{m-k} =  \left(1+\frac{\delta}{6N}\right) {N \choose k} (1-\alpha)^{N-k},
\end{equation}
which is obtained by multiplying both sides of \eqref{eq:var_tilde_psi_eq} by ${N \choose k} (1-\alpha)^{N-k}$ and then adding on the left and on the right the term $\frac{\delta}{6N} {N \choose k} (1-\alpha)^{N-k}$. Notice that \eqref{eq:var_tilde_psi_eq_2} is meaningful for $k=N$ too, and is indeed equivalent to the equation defining $\underline{\eps}_N$, i.e., $\tilde{\Psi}_{N,\delta}(\alpha) = 1$ (for $k=N$, \eqref{eq:var_tilde_psi_eq_2} admits only one solution, which coincides with $\underline{\eps}_N$). Thus, summarizing, the solution to \eqref{eq:var_tilde_psi_eq_2} smaller than $k/N$ provides a lower bound to $\underline{\eps}_k$ for all $k=0,1,\ldots,N$, while the solution greater than $k/N$, which exists for $k=0,1,\ldots,N-1$, provides an upper bound to $\overline{\eps}_k$. \\
\\
We now rewrite \eqref{eq:var_tilde_psi_eq_2} in a form that is better suited to obtain an explicit evaluation of its solutions. \\
\\
For $k=0$, the summation in the left-hand side of \eqref{eq:var_tilde_psi_eq_2} is 
$$
\sum_{m=0}^{4N} (1-\alpha)^m = \frac{1-(1-\alpha)^{4N+1}}{\alpha},
$$
and, noticing that for a generic $k$ it holds that
$$
\sum_{m=k}^{4N} {m \choose k} (1-\alpha)^{m-k} = \frac{(-1)^k}{k!} \cdot \frac{\dd^k}{\dd \alpha^k} \left[ \sum_{m=0}^{4N} (1-\alpha)^m \right],
$$
a cumbersome, but straightforward, computation gives the following general formula for the left-hand side of \eqref{eq:var_tilde_psi_eq_2}
\begin{equation} \label{eq:finite_series_via_beta}
\frac{\delta}{6N} \sum_{m=k}^{4N} {m \choose k} (1-\alpha)^{m-k} = \frac{\delta}{6N} \frac{1-\sum_{i=0}^k {4N+1 \choose i} \alpha^i (1-\alpha)^{4N+1-i}}{\alpha^{k+1}}.
\end{equation}
Substituting in \eqref{eq:var_tilde_psi_eq_2} and multiplying on the left and on the right by $N \alpha^{k+1}$, we then obtain
\begin{equation} \label{eq:var_tilde_psi_eq_3}
	\frac{\delta}{6} \left( 1-\sum_{i=0}^k {4N+1 \choose i} \alpha^i (1-\alpha)^{4N+1-i} \right) =  \left(1+\frac{\delta}{6N}\right) N {N \choose k} \alpha^{k+1} (1-\alpha)^{N-k},
\end{equation}
which, in view of the multiplication by $N \alpha^{k+1}$, is equivalent to \eqref{eq:var_tilde_psi_eq_2} for $\alpha \neq 0$. In what follows, the usage of equation \eqref{eq:var_tilde_psi_eq_3} to investigate the solutions to \eqref{eq:var_tilde_psi_eq_2} will be limited to the interval $(0,1)$ (and, hence, the condition $\alpha \neq 0$ becomes irrelevant). As is clear, this is always the case for the solution greater than $k/N$, which upper bounds $\overline{\eps}_k$, while for the solution lower than $k/N$, which lower bounds $\underline{\eps}_k$, a sufficient condition for this solution to be in $(0,1)$ is that 
\begin{equation}
\label{k>log}
k > \ln(3/\delta), 
\end{equation}
as shown in the following derivation.
\begin{itemize}
	\item[] Start by noticing that
	$$
	k > \ln(3/\delta) \Rightarrow \frac{\delta}{6} \ee^{k+1} > 1 \Rightarrow \frac{\delta}{6} \cdot \frac{4^{k+1}}{k+1} > 1, 
	$$ 
	where the latter implication follows from $4^x/x > \ee^x$ for $x \geq 0$. Since $\frac{4N+1}{N+\frac{\delta}{6}} > 4$ and $\frac{4N-i}{N-i} \geq 4$ for $i=0,1,\ldots,k-1$, the previous inequality implies that 
	$$
	\frac{\delta}{6} \cdot \frac{4N+1}{N+\frac{\delta}{6}} \cdot \frac{4N}{N} \cdots \frac{4N+1-k}{N+1-k} \cdot \frac{1}{k+1} > 1,
	$$
	which, re-arranging the terms and dividing the left-hand and right-hand sides by $k!$, yields
	$$
	\frac{\delta}{6N} \frac{(4N+1) \cdots (4N+1-k)}{(k+1)\cdot k!} > \left(1+\frac{\delta}{6N}\right) \frac{N \cdots (N+1-k)}{k!},
	$$
	or, in a more compact form, $\frac{\delta}{6N} {4N+1 \choose k+1}  > \left(1+\frac{\delta}{6N}\right) {N \choose k}$. Using in this last expression the hockey-stick identity (which asserts that ${4N+1 \choose k+1} = \sum_{m=k}^{4N} {m \choose k}$), one concludes that 
	$$
	\frac{\delta}{6N} \sum_{m=k}^{4N} {m \choose k}  > \left(1+\frac{\delta}{6N}\right) {N \choose k}.
	$$
	This means that the left-hand side of \eqref{eq:var_tilde_psi_eq_2} evaluated at $\alpha = 0$ is larger than the right-hand side evaluated at $\alpha = 0$, and this implies that the two solutions of \eqref{eq:var_tilde_psi_eq_2} (and, thereby, those of \eqref{eq:var_tilde_psi_eq_3}) are both within the interval $(0,1)$.
\end{itemize}
Apart from the factor $\frac{\delta}{6}$, the left-hand side of \eqref{eq:var_tilde_psi_eq_3} is a so-called Beta($k+1$,$4N+1-k$) cumulative distribution function and, as such, it takes value $0$ for $\alpha = 0$ and is strictly increasing in $(0,1)$ converging to the value $1$. 
\begin{figure}[t]
	\centering
	\includegraphics[width=0.9\columnwidth]{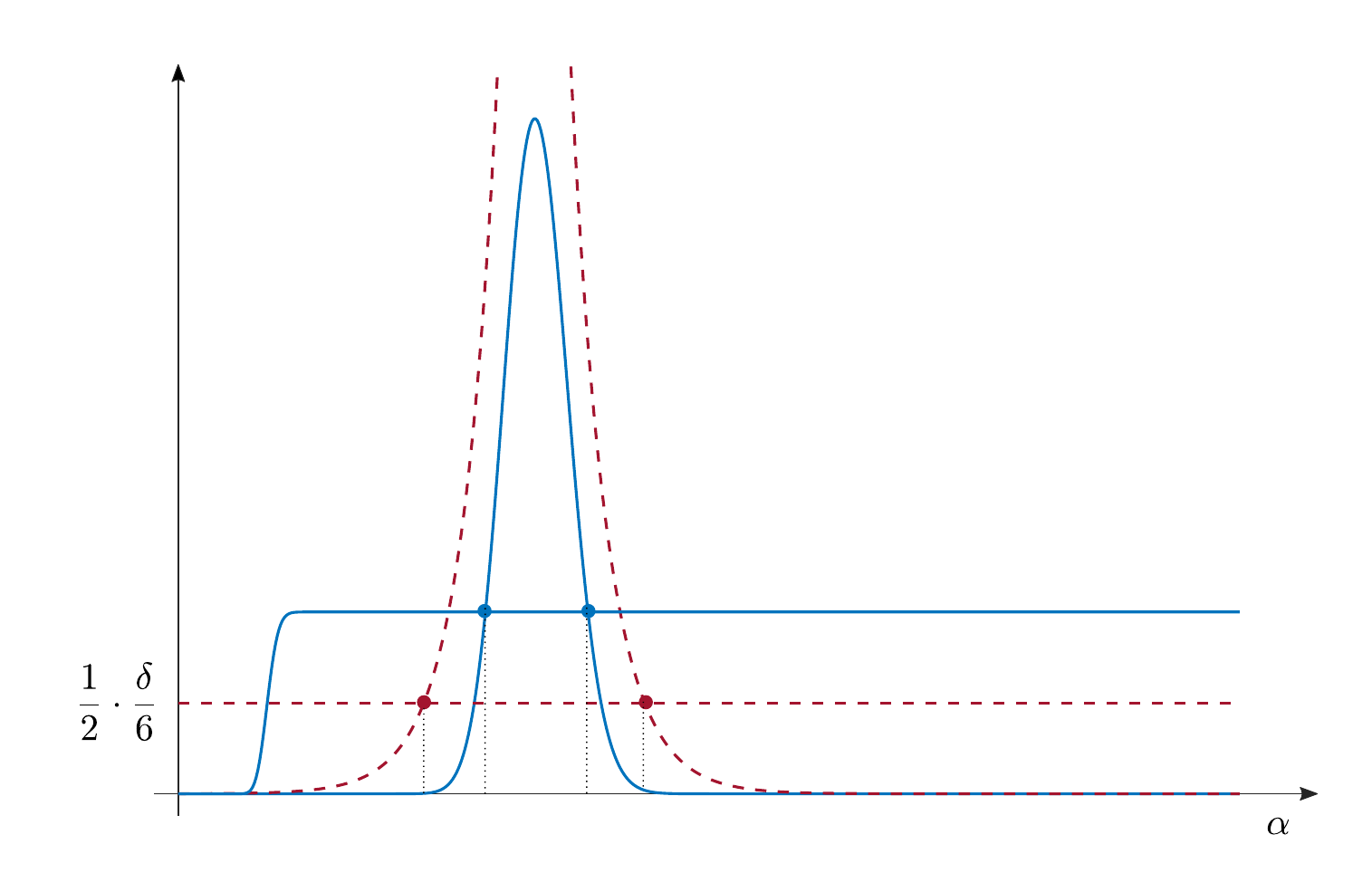}
	\caption{Graph of the left-hand and right-hand sides of \eqref{eq:var_tilde_psi_eq_3} (solid blue lines) and of the curves that are used to obtain suitable lower and upper bounds to the solutions to \eqref{eq:var_tilde_psi_eq_3} (dashed red lines).}
	\label{fig:epsLUbounds_graph}
\end{figure}
The right-hand side, instead, for $k=0,\ldots,N-1$ takes value $0$ for both $\alpha = 0$ and $\alpha = 1$ and is first increasing and then decreasing, while for $k=N$ is $0$ for $\alpha = 0$ only and it is increasing. A graphical illustration of the typical trend for the left-hand and right-hand sides of \eqref{eq:var_tilde_psi_eq_3} for $\ln(3/\delta) < k \leq N -1$ is given in Figure \ref{fig:epsLUbounds_graph} (solid blue lines). Given this state of things, it is clear that if the left-hand side of \eqref{eq:var_tilde_psi_eq_3}, call it $L(\alpha)$, is replaced by a function that lies below $L(\alpha)$, while the right-hand side, call it $R(\alpha)$, is replaced by a monotonically decreasing function that stays above $R(\alpha)$, the so-obtained equation has the property that any solution to it upper bounds the solution to \eqref{eq:var_tilde_psi_eq_3} bigger than $k/N$, which in turn upper bounds $\overline{\eps}_k$. Similarly, keeping the same replacement for the left-hand side of \eqref{eq:var_tilde_psi_eq_3}, but this time substituting the right-hand side $R(\alpha)$ with a monotonically increasing function that stays above $R(\alpha)$, an equation is obtained whose solutions provide lower-bounds to $\underline{\eps}_k$. In the following, the sought upper bound to $\overline{\eps}_k$ and lower bound to $\underline{\eps}_k$ will be obtained by replacing $L(\alpha)$ with a suitable constant function and the right-hand side $R(\alpha)$ with a decreasing exponential and an increasing exponential function, respectively. See again Figure \ref{fig:epsLUbounds_graph} for a graphical illustration (dashed red lines). \\
\\
$\diamond$ \emph{Upper bound to $\overline{\eps}_k$} \\
\\
Consider \eqref{eq:var_tilde_psi_eq_3} for $k=0,1,\ldots,N-1$. Start with the left-hand side and notice that   
\begin{equation} \label{eq:below_func}
	\frac{\delta}{12} \leq \frac{\delta}{6} \left( 1-\sum_{i=0}^k {4N+1 \choose i} \alpha^i (1-\alpha)^{4N+1-i} \right) 
\end{equation}
for $\alpha \geq \frac{k+1}{4N+2}$. As a matter of fact, $\frac{k+1}{4N+2}$ is the mean of the Beta distribution having cumulative distribution function $1-\sum_{i=0}^k {4N+1 \choose k} \alpha^i (1-\alpha)^{4N+1-i}$ and the mean of this Beta distribution is greater than its median, see \cite{PaytonYoungYoung1989}. \\
As for the right-hand side of \eqref{eq:var_tilde_psi_eq_3}, pick any number $c > 1$ and notice that\footnote{This computation is insipred by \cite{Alamoetal2015}.} 
\begin{eqnarray*}
	\lefteqn{\left(1+\frac{\delta}{6N}\right) N {N \choose k} \alpha^{k+1} (1-\alpha)^{N-k}} \\
	& \leq & \left(1+\frac{\delta}{6N}\right) (k+1){N+1 \choose k+1} \alpha^{k+1}(1-\alpha)^{N+1-(k+1)} \\
	& \leq & \frac{7}{6}(k+1)\sum_{i=0}^{k+1} {N+1 \choose i}\alpha^i(1-\alpha)^{N+1-i} \\
	& \leq & \frac{7}{6}(k+1) c^{k+1} \sum_{i=0}^{k+1} {N+1 \choose i} \left(\frac{\alpha}{c}\right)^i(1-\alpha)^{N+1-i} \\
	& \leq & \frac{7}{6}(k+1) c^{k+1} \sum_{i=0}^{N+1} {N+1 \choose i} \left(\frac{\alpha}{c}\right)^i(1-\alpha)^{N+1-i} \\
	& = & \frac{7}{6}(k+1) (1-(1-c))^{k+1} \left(1-\alpha \frac{c-1}{c} \right)^{N+1} \\
	& \leq & \frac{7}{6}(k+1) \ee^{-(1-c)(k+1)} \ee^{-\alpha \frac{c-1}{c} (N+1)} \\
	& \leq & \frac{7}{6}(k+1) \ee^{-(1-c)(k+1)} \ee^{- \alpha \frac{c-1}{c} N},
\end{eqnarray*}
where the second-last inequality derives from relation $1-x \leq \ee^{-x}$. Taking now 
$$
c = 1+\frac{\sqrt{ \ln(k+1)+\ln \frac{14}{\delta} } }{\sqrt{k+1}},
$$
one obtains
\begin{eqnarray} 
	\lefteqn{ \left(1+\frac{\delta}{6N}\right) N {N \choose k} \alpha^{k+1} (1-\alpha)^{N-k} } \nonumber \\
	& \leq & \frac{7}{6}(k+1) \ee^{\sqrt{ \ln(k+1)+\ln \frac{14}{\delta} } \sqrt{k+1}} \ee^{-\alpha \frac{\sqrt{ \ln(k+1)+\ln \frac{14}{\delta} } \cdot N}{\sqrt{k+1}+\sqrt{ \ln(k+1)+\ln \frac{14}{\delta} } } }, \label{eq:dec_above_func}
\end{eqnarray}
where the right-hand side is a monotonically decreasing function of $\alpha$. Using \eqref{eq:below_func} and \eqref{eq:dec_above_func} together, in view of the argument given after \eqref{eq:var_tilde_psi_eq_3} we have that the solution to 
\begin{equation} \label{eq:below_dec_above}
	\frac{\delta}{12} = \frac{7}{6}(k+1) \ee^{\sqrt{ \ln(k+1)+\ln \frac{14}{\delta} } \sqrt{k+1}} \ee^{-\alpha \frac{\sqrt{ \ln(k+1)+\ln \frac{14}{\delta} } \cdot N}{\sqrt{k+1}+\sqrt{ \ln(k+1)+\ln \frac{14}{\delta} } } }
\end{equation}
upper bounds $\overline{\eps}_k$ as long as the solution turns out to be no smaller than $\frac{k+1}{4N+2}$ (which is the condition to ensure that \eqref{eq:below_func} is indeed true). Solving \eqref{eq:below_dec_above} for $\alpha$ gives
$$
\overline{\alpha} = \frac{k+1}{N} + 2 \frac{\sqrt{k+1}\sqrt{ \ln(k+1)+\ln \frac{14}{\delta} } }{N} + \frac{ \ln(k+1)+\ln \frac{14}{\delta} }{N},
$$
which is always greater than $\frac{k+1}{4N+2}$. Thus, it follows that, for $k=0,1,\ldots,N-1$,
$$
\overline{\eps}_k \leq \frac{k+1}{N} + 2 \frac{\sqrt{k+1}\sqrt{ \ln(k+1)+\ln \frac{14}{\delta} } }{N} + \frac{ \ln(k+1)+\ln \frac{14}{\delta} }{N}.
$$
Moreover, the bound turns out to be valid for $k=N$ too, since $\overline{\eps}_N = 1$, while the right-hand side of the previous inequality is greater than $1$ for $k=N$. Using the fact that $\sqrt{ \ln(k+1) + \ln \frac{14}{\delta} } \leq \sqrt{\ln(k+1)} + \sqrt{\ln(14)} + \sqrt{\ln \frac{1}{\delta}} $ in the previous expression, and re-arranging the terms, one obtains: 
$$
\overline{\eps}_k \leq \frac{k}{N} + \frac{2\sqrt{k+1} \left(\sqrt{\ln(k+1)} + \sqrt{\ln(14)} \right) + \ln(k+1) + \ln(14) + 1 }{N} + 2\frac{\sqrt{k+1}\sqrt{\ln \frac{1}{\delta}}}{N} + \frac{\ln \frac{1}{\delta}}{N}.
$$
It is easy to verify that $\ln(k+1) + \ln(14) + 1  < 4 \sqrt{k+1}$; using this fact in the numerator of the second term together with $\sqrt{\ln(14)} + 2 < 4$ yields
$$
\overline{\eps}_k \leq \frac{k}{N} + 2\frac{\sqrt{k+1} }{N}\left(\sqrt{\ln(k+1)} + 4 \right) + 2\frac{\sqrt{k+1}\sqrt{\ln \frac{1}{\delta}}}{N} + \frac{\ln \frac{1}{\delta}}{N},
$$
which is the bound in \eqref{eq:bounds4asympt_up}. \\
\\
$\diamond$ \emph{Lower bound to $\underline{\eps}_k$} \\
\\
Consider \eqref{eq:var_tilde_psi_eq_3} again, this time for $k > \ln(\frac{3}{\delta})$ as given in \eqref{k>log} (which, as we have seen, ensures that the solution lower than $k/N$ takes value in the interval $(0,1)$). The right-hand side of \eqref{eq:var_tilde_psi_eq_3} is upper bounded as follows ($c$ is any number bigger than $1$):
\begin{eqnarray*}
	\lefteqn{\left(1+\frac{\delta}{6N}\right) N {N \choose k} \alpha^{k+1} (1-\alpha)^{N-k}} \\
	& \leq & \left(1+\frac{\delta}{6N}\right)(k+1) {N+1 \choose k+1}\alpha^{k+1}(1-\alpha)^{N+1-(k+1)} \\
	& \leq & \frac{7}{6}(k+1)\sum_{i=k+1}^{N+1} {N+1 \choose i}\alpha^i(1-\alpha)^{N+1-i} \\
	& \leq & \frac{7}{6}(k+1) \frac{1}{c^{k+1}} \sum_{i=k+1}^{N+1} {N+1 \choose i} (\alpha c)^i(1-\alpha)^{N+1-i} \\
	& \leq & \frac{7}{6}(k+1) \frac{1}{c^{k+1}} \sum_{i=0}^{N+1} {N+1 \choose i} (\alpha c)^i(1-\alpha)^{N+1-i} \\	
	& = & \frac{7}{6}(k+1) \frac{\big(1+\alpha(c-1)\big)^{N+1}}{\big(1+(c-1)\big)^{k+1}} \\
	& \leq & \frac{7}{6}(k+1) \ee^{\alpha(c-1)(N+1)} \ee^{-\left(c-1-\frac{(c-1)^2}{2}\right)(k+1)},
\end{eqnarray*}
where the last inequality derives from $\ee^{x-\frac{x^2}{2}} \leq 1+x \leq \ee^{x}$ for $x \geq 0$. With the choice 
$$
c = 1+\frac{\sqrt{\ln(k+1)+\ln \frac{14}{\delta}}}{\sqrt{k+1}},
$$
we now obtain
\begin{eqnarray} 
	\lefteqn{ \left(1+\frac{\delta}{6N}\right) N {N \choose k} \alpha^{k+1} (1-\alpha)^{N-k} } \nonumber \\
	& \leq & \frac{7}{6} (k+1) \ee^{\alpha \frac{\sqrt{\ln(k+1)+\ln \frac{14}{\delta}} \cdot (N+1)}{\sqrt{k+1}}} \ee^{-\sqrt{\ln(k+1)+\ln \frac{14}{\delta}}\sqrt{k+1}+\frac{1}{2}\left(\ln(k+1)+\ln \frac{14}{\delta}\right)} , \label{eq:inc_above_func}
\end{eqnarray}
where the right-hand side is an increasing function of $\alpha$. Using \eqref{eq:below_func} and \eqref{eq:inc_above_func} together, in view of the argument given after \eqref{eq:var_tilde_psi_eq_3} we have that the solution to 
\begin{equation} \label{eq:below_inc_above}
	\frac{\delta}{12} = \frac{7}{6} (k+1) \ee^{\alpha \frac{\sqrt{\ln(k+1)+\ln \frac{14}{\delta}} \cdot (N+1)}{\sqrt{k+1}}} \ee^{-\sqrt{\ln(k+1)+\ln \frac{14}{\delta}}\sqrt{k+1}+\frac{1}{2}\left(\ln(k+1)+\ln \frac{14}{\delta}\right)} 
\end{equation}
lower bounds $\underline{\eps}_k$ as long as the solution turns out to be no smaller than $\frac{k+1}{4N+2}$ (which is the condition to ensure that \eqref{eq:below_func} is indeed true). Solving \eqref{eq:below_inc_above} for $\alpha$ gives
\begin{equation} 
\label{underline-alpha}
\underline{\alpha} = \frac{k+1}{N+1} - \frac{3}{2}\frac{\sqrt{k+1}}{N+1} \sqrt{\ln(k+1)+\ln \frac{14}{\delta}}.
\end{equation}
Hence, for those $k$ for which $\underline{\alpha} \geq \frac{k+1}{4N+2}$, we have that $\underline{\eps}_k \geq \underline{\alpha}$, which also clearly gives the looser inequality 
\begin{equation}
\label{eq:epsL_k_bound_aux}
\underline{\eps}_k \geq \frac{k+1}{N+1} - 3\frac{\sqrt{k+1}}{N+1} \sqrt{\ln(k+1)+\ln \frac{14}{\delta}}. 
\end{equation}
On the other hand, if $k$ is such that $\underline{\alpha} < \frac{k+1}{4N+2}$, then, substituting the expression for $\underline{\alpha}$ given in \eqref{underline-alpha} in relation $\underline{\alpha} < \frac{k+1}{4N+2}$ yields 
$$
\frac{3}{2}\frac{\sqrt{k+1}}{N+1} \sqrt{\ln(k+1)+\ln \frac{14}{\delta}} > \frac{k+1}{N+1} - \frac{k+1}{4N+2} > \frac{1}{2}\cdot\frac{k+1}{N+1},
$$
from which
$$
\frac{k+1}{N+1} - 3\frac{\sqrt{k+1}}{N+1} \sqrt{\ln(k+1)+\ln \frac{14}{\delta}} < 0,
$$
and, since $\underline{\eps}_k \geq 0$ by definition, we conclude that \eqref{eq:epsL_k_bound_aux} remains valid in this case as well. \\
\\
Go now back to re-consider condition $k > \ln(\frac{3}{\delta})$ introduced at the beginning of this part about lower bounding $\underline{\eps}_k$, and consider the opposite case that $k \leq \ln(\frac{3}{\delta})$. If so, it also holds that $k+1 \leq 3 \sqrt{k+1} \sqrt{\ln(k+1)+\ln \frac{14}{\delta}}$, which in turn implies that the right-hand side of \eqref{eq:epsL_k_bound_aux} is no bigger than $0$. This, in view of relation $\underline{\eps}_k \geq 0$, proves that \eqref{eq:epsL_k_bound_aux} is valid in this case as well. Hence, we conclude that \eqref{eq:epsL_k_bound_aux} is valid for all $k=0,1,\ldots,N$. Since $\frac{k+1}{N+1} \geq \frac{k}{N}$, $\frac{1}{N+1} < \frac{1}{N}$, and $\sqrt{\ln(k+1)+\ln \frac{14}{\delta}} \leq \sqrt{\ln(k+1)}+\sqrt{\ln(14)} + \sqrt{ \ln \frac{1}{\delta}}$, from \eqref{eq:epsL_k_bound_aux} we also obtain
$$
\underline{\eps}_k \geq \frac{k}{N} - 3\frac{\sqrt{k+1}}{N} \left(\sqrt{\ln(k+1)} + \sqrt{\ln(14)} \right) -3\frac{\sqrt{k+1} \sqrt{\ln \frac{1}{\delta}}}{N} .
$$
The bound in \eqref{eq:bounds4asympt_low} is finally obtained by noticing that $\sqrt{\ln(14)}<2$. \\
\\
This concludes the proof. \qed

\acks{The authors would like to thank professor Nicolò Cesa-Bianchi for insightful and inspiring discussions on various parts of this manuscript. They also express their gratitude to the editor and to anonimous reviewers for generously providing comments that helped improve the paper. \\
	
The research presented in this article has been partly supported by MUR under the PRIN 2022 project ``The Scenario Approach for Control and Non-Convex Design'' (project number D53D23001440006) and by FAIR (Future Artificial Intelligence Research) project, funded by the NextGenerationEU program within the PNRR-PE-AI scheme (M4C2, Investment 1.3, Line on Artificial Intelligence).
}

\newpage

\appendix

\section{A study of the graph of $\tilde{\Psi}_{k,\delta}$} \label{Appendix_Psi_tilde}

For $k = 0,1,\ldots,N-1$, the derivative of $\tilde{\Psi}_{k,\delta}$, i.e.,
$$
\frac{\dd  \tilde{\Psi}_{k,\delta}}{\dd \alpha}(\alpha) = \frac{\delta}{2N} \sum_{m=k}^{N-1} \frac{{m \choose k}}{{N \choose k}} (N-m) (1-\alpha)^{-(N-m+1)} - \frac{\delta}{6N} \sum_{m=N+1}^{4N} \frac{{m \choose k}}{{N \choose k}} (m-N) (1-\alpha)^{m-N-1},
$$
is a strictly increasing function with limit $-\infty$ for $\alpha \to -\infty$ and limit $+\infty$ for $\alpha \to 1$. Hence, over the interval $(-\infty,1)$, $\tilde{\Psi}_{k,\delta}$ is first decreasing and then increasing, with a unique minimum. Since $\tilde{\Psi}_{k,\delta} \to +\infty$ for both $\alpha \to -\infty$ and $\alpha \to 1$, to prove that the equation $\tilde{\Psi}_{k,\delta}(\alpha) = 1$ has indeed two solutions, it is enough to exhibit an $\bar{\alpha} \in (-\infty,1)$ for which $\tilde{\Psi}_{k,\delta}(\bar{\alpha}) < 1$. To this aim, consider $\bar{\alpha} = \frac{k}{N}$ and notice that
\begin{equation} \label{eq:psi_tilde_in_k/N}
\tilde{\Psi}_{k,\delta}(k/N) = \frac{\delta}{2} \left[ \frac{1}{N} \sum_{m=k}^{N-1} \frac{{m \choose k}}{{N \choose k}} \frac{N^{N-m}}{(N-k)^{N-m}} + \frac{1}{3N} \sum_{m=N+1}^{4N} \frac{{m \choose k}}{{N \choose k}} \frac{(N-k)^{m-N}}{N^{m-N}} \right].
\end{equation}
For $m = k, \ldots, N-1$, we have
$$
\frac{{m \choose k}}{{N \choose k}} \frac{N^{N-m}}{(N-k)^{N-m}} = \prod_{i=0}^{N-m-1} \left( \frac{N-k-i}{N-i} \cdot \frac{N}{N-k} \right) \leq 1,
$$
where the inequality is satisfied since each term $\frac{N-k-i}{N-i} \cdot \frac{N}{N-k}$ in the product is no bigger than $1$. Similarly, for $m = N+1, \ldots, 4N$ we have
$$
\frac{{m \choose k}}{{N \choose k}} \frac{(N-k)^{m-N}}{N^{m-N}} = \prod_{i=1}^{m-N} \left( \frac{N+i}{N-k+i} \cdot \frac{N-k}{N} \right) \leq 1
$$
(again it is straightforward to verify that each term $\frac{N+i}{N-k+i} \cdot \frac{N-k}{N}$ is no bigger than $1$). We therefore see that both sums in the square brackets in the right-hand side of \eqref{eq:psi_tilde_in_k/N} are arithmetic means of quantities no bigger than $1$, and therefore they are no bigger than $1$ as well. This gives $\tilde{\Psi}_{k,\delta}(k/N) \leq \delta < 1$, and it also shows that $\underline{\alpha}_k < \frac{k}{N} < \overline{\alpha}_k$. 

\section{\textsf{MATLAB} code} \label{appendix:MATLAB_code}

In this Appendix, we provide efficient bisection algorithms in the \textsf{MATLAB} computing environment for the evaluation of $\eps_k$, $\underline{\eps}_k$ and $\overline{\eps}_k$. The algorithms take advantage of some reformulations of the equations $\Psi_{k,\delta}(\alpha) = 1$ and $\tilde{\Psi}_{k,\delta}(\alpha) = 1$ as discussed in the beginning of the proof of Proposition \ref{th:bounds4asympt} in Section \ref{proof_bounds4asympt}. A summary of these reformulations are provided before the \textsf{MATLAB} codes. 

\subsection{Bisection algorithm for the computation of $\eps_k$}
\label{appendix-bisection-algo1}

Equation $\Psi_{k,\delta}(\alpha) = 1$ for $k=0,\ldots,N-1$ can be rewritten as
$$
\frac{\delta}{N} \sum_{m=k}^{N-1} {m \choose k} (1-\alpha)^{m-k} = {N \choose k} (1-\alpha)^{N-k},
$$
and, using a formula analogous to \eqref{eq:finite_series_via_beta} in Section \ref{proof_bounds4asympt}, but with $N-1$ in place of $4N$, this equation is shown to be equivalent over the interval $(0,1)$ to
$$
\delta \left( 1-\sum_{i=0}^k {N \choose k} \alpha^i (1-\alpha)^{N-i} \right) =  \alpha N  {N \choose k} \alpha^k (1-\alpha)^{N-k}.
$$
The only solution to this equation in $(0,1)$ is $\eps_k$ and it can be computed by bisection starting from $0$ and $1$ as initial extremes. Apart from the coefficient $\delta$, the left-hand side is an incomplete Beta function with parameters $k+1$ and $N-k$, which can be efficiently and accurately evaluated with the $\texttt{betainc}$ function of \textsf{MATLAB}. Similarly, apart from the term $\alpha N$, the right-hand side can be computed as the difference of two incomplete beta functions. The following code provides a ready-to-use implementation of the bisection algorithm in the \textsf{MATLAB} environment.  

\begin{verbatim}
function eps = find_eps(k,N,delta)

if k==N
eps = 1;
else    
t1 = 0;
t2 = 1;
while t2-t1 > 1e-10
t = (t1+t2)/2;
left = delta*betainc(t,k+1,N-k);
right = t*N*(betainc(t,k,N-k+1)-betainc(t,k+1,N-k));
if left > right
t2=t;
else
t1=t;
end        
end
eps = t2;
end

end
\end{verbatim}

\subsection{Bisection algorithm for the computation of $\underline{\eps}_k$ and $\overline{\eps}_k$}
\label{appendix-bisection-algo2}

Following the same argument given in the proof of Proposition \ref{th:bounds4asympt} to derive \eqref{eq:var_tilde_psi_eq_2} (see Section \ref{proof_bounds4asympt}), it is easy to see that equation $\tilde{\Psi}_{k,\delta}(\alpha) = 1$ can be reformulated as
$$
\frac{\delta}{3N} \sum_{m=k}^{N-1} {m \choose k} (1-\alpha)^{m-k} + \frac{\delta}{6N} \sum_{m=k}^{4N} {m \choose k} (1-\alpha)^{m-k} = \left(1 + \frac{\delta}{6N} \right) {N \choose k} (1-\alpha)^{N-k}.
$$
Using again formula \eqref{eq:finite_series_via_beta} for the second term in the left-hand side and its variant with $N-1$ in place of $4N$ for the first term, one obtains that $\tilde{\Psi}_{k,\delta}(\alpha) = 1$ is equivalent over the interval $(0,1)$ to the equation
\begin{align*}
& \frac{\delta}{3} \left( 1-\sum_{i=0}^k {N \choose k} \alpha^i (1-\alpha)^{N-i} \right) + \frac{\delta}{6} \left( 1-\sum_{i=0}^k {4N+1 \choose k} \alpha^i (1-\alpha)^{4N+1-i} \right)  \\
& \quad = \left(1+\frac{\delta}{6N}\right) \alpha N {N \choose k} \alpha^{k} (1-\alpha)^{N-k},
\end{align*}
where, again, the left- and the right-hand sides can be conveniently computed via the incomplete Beta function. A bisection algorithm with $\frac{k}{N}$ and $1$ as extremes can be used to compute $\overline{\alpha}_k = \overline{\eps}_k$ for $k=0,\ldots,N-1$; instead, for $k=0,\ldots,N$, using $0$ and $\frac{k}{N}$ as extremes, the bisection algorithm converges to $\underline{\alpha}_k = \underline{\eps}_k$ when $\underline{\alpha}_k > 0$ and to $0 = \underline{\eps}_k$ when $\underline{\alpha}_k \leq 0$. The following code provides an implementation in \textsf{MATLAB}.

\begin{verbatim}
function [epsL, epsU] = find_epsLU(k,N,delta)

t1 = 0;
t2 = k/N;
while t2-t1 > 1e-10
t = (t1+t2)/2;
left = delta/3*betainc(t,k+1,N-k)+delta/6*betainc(t,k+1,4*N+1-k);
right = (1+delta/6/N)*t*N*(betainc(t,k,N-k+1)-betainc(t,k+1,N-k));
if left > right
t1=t;
else
t2=t;
end
end
epsL = t1;

if k==N
epsU = 1;
else
t1 = k/N;
t2 = 1;
while t2-t1 > 1e-10
t = (t1+t2)/2;
left = (delta/2-delta/6)*betainc(t,k+1,N-k)+delta/6*betainc(t,k+1,4*N+1-k);
right = (1+delta/6/N)*t*N*(betainc(t,k,N-k+1)-betainc(t,k+1,N-k));
if left > right
t2=t;
else
t1=t;
end
end
epsU = t2;
end 

end
\end{verbatim}

\section{Proof of (B) in \eqref{equality-duality}}
\label{Appendix strong duality}

Let $\tau > 0$ be a number smaller than $1-\eps_k$ for all $k$'s for which $\eps_k \neq 1$. Matters of convenience (as shown later) suggest to introduce a modified Lagrangian that corresponds to a continuous cost function as follows
\begin{eqnarray*}
\mathfrak{L}_\tau & = & \sum_{k=0}^N {N \choose k} \int_{[0,1]} \varphi_{k,\tau}(\alpha) \; \dd \mathfrak{m}^+_{k,N}
- \sum_{m=0}^M \lambda_m \left( \sum_{k=0}^m {m \choose k} \int_{[0,1]} \dd \mathfrak{m}^+_{k,m} - 1 \right) \nonumber \\
& & - \sum_{m=0}^{M-1} \sum_{k=0}^m \int_{[0,1]} \mu^+_{k,m}(\alpha) \; \dd \! \left[  \mathfrak{m}^+_{k,m+1} - (1-\alpha) \mathfrak{m}^+_{k,m} \right],
\end{eqnarray*}
where: for all $k$ for which $\eps_k \neq 1$, $\varphi_{k,\tau}(\alpha)$ is a continuous function equal to $0$ for $\alpha \in [0,\eps_k - \tau]$,
equal to $1$ for $\alpha \in [\eps_k, 1]$, and with a linear slope connecting $0$ to $1$ in between; while $\varphi_{k,\tau}(\alpha)$ is identically zero when $\eps_k = 1$. We show below the validity of the following relations: 
{ \everymath={\displaystyle}
\begin{equation}
\label{fundamnetal relations}
\begin{array}{ccc}
\sup_{\{\mathfrak{m}^+_{k,m}\}} \inf_{\substack{\{\lambda_m\} \\ \{\mu^+_{k,m}\}}} \mathfrak{L}_\tau &
= & \inf_{\substack{\{\lambda_m\} \\ \{\mu^+_{k,m}\}}} \sup_{\{\mathfrak{m}^+_{k,m}\}} \mathfrak{L}_\tau \\
\downarrow_{\tau \downarrow 0} & & \rotgeq \\
\sup_{\{\mathfrak{m}^+_{k,m}\}} \inf_{\substack{\{\lambda_m\}\\ \{\mu^+_{k,m}\}}} \mathfrak{L} &
\leq  & \inf_{\substack{\{\lambda_m\} \\ \{\mu^+_{k,m}\}}} \sup_{\{\mathfrak{m}^+_{k,m}\}} \mathfrak{L}.
\end{array}
\end{equation}
}
Notice that the above relations imply the sought result that
$$
\sup_{\{\mathfrak{m}^+_{k,m}\}} \inf_{\substack{\{\lambda_m\} \\ \{\mu^+_{k,m}\}}} \mathfrak{L} = \inf_{\substack{\{\lambda_m\} \\ \{\mu^+_{k,m}\}}} \sup_{\{\mathfrak{m}^+_{k,m}\}} \mathfrak{L}
$$
because 
$$
\inf_{\substack{\{\lambda_m\} \\ \{\mu^+_{k,m}\}}} \sup_{\{\mathfrak{m}^+_{k,m}\}} \mathfrak{L}
$$
is in sandwich between
$$
\sup_{\{\mathfrak{m}^+_{k,m}\}} \inf_{\substack{\{\lambda_m\} \\ \{\mu^+_{k,m}\}}} \mathfrak{L}
$$
and
$$
\inf_{\substack{\{\lambda_m\} \\ \{\mu^+_{k,m}\}}} \sup_{\{\mathfrak{m}^+_{k,m}\}} \mathfrak{L}_\tau,
$$
two quantities that converge one onto the other as $\tau \downarrow 0$. \\
\\
The two inequalities in \eqref{fundamnetal relations} are justified as follows: the $\leq$ at the bottom of \eqref{fundamnetal relations} is valid because relation ``$\sup \inf \leq \inf \sup$'' is always true, while the $\rotgeq$ on the right follows from the fact that $\varphi_{k,\tau}(\alpha)$ in $\mathfrak{L}_\tau$ is greater than or equal to $\One{\alpha \in (\eps_k,1]}$ in $\mathfrak{L}$. \\
\\
What remains to show is thus the $=$ at the top of \eqref{fundamnetal relations} and the convergence $\downarrow_{\tau \downarrow 0}$ on the left. \\
\\
We first show that
\begin{equation} \label{eq:supinf=infsup_tau}
\sup_{\{\mathfrak{m}^+_{k,m}\}} \inf_{\substack{\{\lambda_m\} \\ \{\mu^+_{k,m}\}}} \mathfrak{L}_\tau = \inf_{\substack{\{\lambda_m\} \\ \{\mu^+_{k,m}\}}} \sup_{\{\mathfrak{m}^+_{k,m}\}} \mathfrak{L}_\tau,
\end{equation}
for which purpose we need to introduce a proper topological vector space,
\cite{Rudin_FA}, as specified in the following.
\begin{itemize}
\item[]
Consider the vector space $\mathcal{M}$ of finite signed measures $\mathfrak{m}$ with support on $[0,1]$. Moreover, let $\mathcal{LF}$ be the vector space of linear functionals on $\mathcal{M}$ of the form $\int_{[0,1]} \mu(\alpha) \; \dd \mathfrak{m}$, where $\mu$ is a  continuous function ($\mu \in \textsf{C}^0[0,1]$). In $\mathcal{M}$, introduce the weak topology induced by $\mathcal{LF}$, see Section 3.8 in \cite{Rudin_FA}. This weak topology makes $\mathcal{M}$ into a locally convex topological vector space whose dual space coincides with $\mathcal{LF}$, see Theorem 3.10 in \cite{Rudin_FA}.\footnote{For the applicability of Theorem 3.10, one needs that $\mathcal{LF}$ ``separates'' $\mathcal{M}$, a fact that follows from Footnote \ref{key_footnote}.} By also considering the standard topology of $\R$ generated by open intervals, the ambient space in which we are going to work is the topological vector space given by $\R \times \R^{M+1} \times \mathcal{M}^{\frac{(M+1)M}{2}} =: \mathcal{S}$ equipped with the product topology. \\
\\
The interpretation of $\mathcal{S}$ is that it is the codomain of an operator that maps $\mathfrak{m}^+_{k,m}$, $m=0,1,\ldots,M$, $k=,0,\ldots,m$ into an element of $\mathcal{S}$ according to the rule:
\begin{eqnarray*}
\big\{ \mathfrak{m}^+_{k,m} \big\}_{\begin{subarray}{l}m=0,1,\ldots,M \\ k=0,\ldots,m \end{subarray}} & \to & \left\{
\begin{array}{ll}
\sum_{k=0}^N {N \choose k} \int_{[0,1]} \varphi_{k,\tau}(\alpha) \ \dd \mathfrak{m}^+_{k,N} & \quad (\in \R) \\
\left\{ \sum_{k=0}^m {m \choose k} \int_{[0,1]} \dd \mathfrak{m}^+_{k,m} - 1 \right\}_{m=0,1,\ldots,M}  & \quad (\in \R^{M+1}) \\
\big\{ \mathfrak{m}^+_{k,m+1} - (1-\alpha) \; \mathfrak{m}^+_{k,m} \big\}_{\begin{subarray}{l}m=0,1,\ldots,M-1 \\ k=0,\ldots,m \end{subarray}} & \quad (\in \mathcal{M}^{\frac{(M+1)M}{2}})
\end{array}
\right.
\end{eqnarray*}
(note that this operator returns various terms that are found in $\mathfrak{L}_\tau$). We next consider the image of this operator, that is, the range of points
in $\mathcal{S}$ that are reached as $\{\mathfrak{m}^+_{k,m}\}$ varies in
its domain $\left(\mathcal{M}^+\right)^{\frac{(M+2)(M+1)}{2}}$. To this image, we further add an arbitrary positive measure $\mathfrak{p}^+_{k,m}$
to each term  $\mathfrak{m}^+_{k,m+1} - (1-\alpha) \; \mathfrak{m}^+_{k,m}$ (the reason for this will become clear shortly). The final set that is
obtained as $\{\mathfrak{m}^+_{k,m}\}$ and $\{\mathfrak{p}^+_{k,m}\}$ vary over their domains is denoted by $H$:
\begin{eqnarray}
H & := & \Big\{(v,\{r_m\},\big\{ \mathfrak{q}_{k,m} \big\}) \in \R \times \R^{M+1} \times \mathcal{M}^{\frac{(M+1)M}{2}} : \nonumber \\
& & \quad \quad \quad v = \sum_{k=0}^N {N \choose k} \int_{[0,1]} \varphi_{k,\tau}(\alpha) \ \dd \mathfrak{m}^+_{k,N}, \nonumber \\
& & \quad \quad \quad \{r_m\} = \left\{ \sum_{k=0}^m {m \choose k} \int_{[0,1]} \dd \mathfrak{m}^+_{k,m} - 1 \right\}, \nonumber \\
& & \quad \quad \quad \big\{ \mathfrak{q}_{k,m} \big\} = \big\{\mathfrak{m}^+_{k,m+1} - (1-\alpha) \; \mathfrak{m}^+_{k,m} + \mathfrak{p}^+_{k,m} \big\}, \nonumber \\
& & \quad \mbox{where, for all } m \mbox{ and } k, \ \ \mathfrak{m}^+_{k,m} \in \mathcal{M}^+, \mathfrak{p}^+_{k,m} \in \mathcal{M}^+ \Big\}. \label{definition H}
\end{eqnarray}
The closure of $H$ in the topology of $\mathcal{S}$ is denoted by $\bar{H}$.\footnote{The closure $\bar{H}$ is formed by all contact points of $H$, where a point is of contact if any neighborhood of the point contains at least one point in $H$; clearly, any point $h \in H$ also belongs to $\bar{H}$.} The following definitions refer to the restrictions of $H$ and $\bar{H}$ to the line where all $r_m$ and $\mathfrak{q}_{k,m}$ are set to
$0$ (i.e., the zero element in $\R$ and $\mathcal{M}$, respectively): quantities
\begin{eqnarray*}
V & := & \sup \Big\{ v : \; (v,\{r_m = 0 \},\big\{ \mathfrak{q}_{k,m}
= 0 \big\}) \in H \Big\} \\
\bar{V} & := & \sup \Big\{ v : \; (v,\{ r_m = 0 \}, \big\{ \mathfrak{q}_{k,m} = 0 \big\}) \in \bar{H} \Big\}
\end{eqnarray*}
are called \emph{value} and \emph{supervalue}, respectively.\footnote{Note that, in the definition of $V$, $\sup$ is taken over a nonempty set. As
a matter of fact, owing to \eqref{condition(i)} and \eqref{condition(ii)}, any compression scheme satisfying the \emph{preference} property gives rise to
measures $\mathfrak{m}^+_{k,m}$ such that $r_m = 0$ for all $m$ and $\mathfrak{q}_{k,m} = 0$ for all $m$ and $k$ by choosing $\mathfrak{p}^+_{k,m} = -( \mathfrak{m}^+_{k,m+1} - (1-\alpha) \; \mathfrak{m}^+_{k,m})$. It is also worth noticing that $\bar{V}$ (and hence $V$ too) is finite and no bigger than $1$. In fact, by the definition of $H$, every point in $H$ satisfies $v \leq r_N+1$. On the other hand, if it were that $\bar{V} > 1$, then there would exist a contact point of $H$ such that $v > 1$ and
$r_N = 0$, which is in contradiction with the fact that $v \leq r_N+1$ for all points in $H$.
}
With this notation, we have
$$
\sup_{\{\mathfrak{m}^+_{k,m}\}} \inf_{\substack{\{\lambda_m\} \\ \{\mu^+_{k,m}\}}} \mathfrak{L}_\tau = V
$$
(this fact easily follows from an argument similar to the proof of equality (A) in \eqref{equality-duality} after noting that $V$ in the present context plays the same role as $\gamma_M$ in left-hand side of\eqref{equality-duality}). On the other hand, we also have
\begin{equation}
\label{tbp-Hahn-Banach}
\inf_{\substack{\{\lambda_m\} \\ \{\mu^+_{k,m}\}}} \sup_{\{\mathfrak{m}^+_{k,m}\}} \mathfrak{L}_\tau = \bar{V},
\end{equation}
which requires the proof given below, based on Hahn-Banach theorem. After showing this, the proof of \eqref{eq:supinf=infsup_tau} is concluded by proving that $V = \bar{V}$.
\begin{itemize}
\item[]
To prove \eqref{tbp-Hahn-Banach}, note that $\bar{H}$ is convex and closed and, for any $\varepsilon > 0$, point $s^\varepsilon := (\bar{V} + \varepsilon, \{ r_m = 0 \}, \{ \mathfrak{q}_{k,m} = 0 \}) \notin \bar{H}$. By an application of Hahn-Banach theorem (see Theorem 3.4 in \citealp{Rudin_FA}), one can therefore find a linear continuous functional defined over $\mathcal{S}$ that ``separates'' $\bar{H}$ from $s^\varepsilon$ in such
a way that the functional computed at any point of $\bar{H}$ is strictly smaller than the functional computed at $s^\varepsilon$.

A generic linear continuous functional defined over $\mathcal{S}$ is written as
\begin{equation}
\label{functionalHB}
a \cdot v - \sum_{m=0}^M \lambda_m r_m - \sum_{m=0}^{M-1} \sum_{k=0}^m \int_{[0,1]} \mu_{k,m}(\alpha) \; \dd \mathfrak{q}_{k,m},
\end{equation}
where $a,\lambda_m \in \R^{}$ and $\mu_{k,m} \in \textsf{C}^0[0,1]$, and hence the separation condition yields
\begin{equation} \label{inequality-HB}
\begin{array}{r}
a^\varepsilon \cdot v - \sum_{m=0}^M \lambda_m^\varepsilon r_m - \sum_{m=0}^{M-1} \sum_{k=0}^m \int_{[0,1]} \mu_{k,m}^\varepsilon(\alpha) \; \dd \mathfrak{q}_{k,m} < a^\varepsilon \cdot (\bar{V} + \varepsilon), \\
\forall (v, \{ r_m\}, \{ \mathfrak{q}_{k,m} \}) \in \bar{H},
\end{array}
\end{equation}
where $a^\varepsilon, \lambda_m^\varepsilon, \mu_{k,m}^\varepsilon(\alpha)$ are specific choices of $a, \lambda_m, \mu_{k,m}(\alpha)$ in \eqref{functionalHB}. Specializing \eqref{inequality-HB} to a point in $\bar{H}$ with $\{r_m = 0\}$ and $\{ \mathfrak{q}_{k,m} = 0 \}$ yields $a^\varepsilon \cdot v < a^\varepsilon \cdot (\bar{V} + \varepsilon)$, which implies $a^\varepsilon > 0$. Moreover, noting that $\mathfrak{q}_{k,m}$ contains $\mathfrak{p}^+_{k,m}$, which is positive and arbitrarily large, one concludes that $\mu_{k,m}^\varepsilon$ must be non-negative for the inequality to hold over the whole $\bar{H}$. To take notice of this fact, in subsequent derivations we write $\mu_{k,m}^{\varepsilon, +}$ in place of $\mu_{k,m}^\varepsilon$. Dividing by $a^\varepsilon$, inequality \eqref{inequality-HB} now gives
$$
v - \sum_{m=0}^M \frac{\lambda_m^\varepsilon}{a^\varepsilon} r_m - \sum_{m=0}^{M-1} \sum_{k=0}^m \int_{[0,1]} \frac{\mu_{k,m}^{\varepsilon,+}(\alpha)}{a^\varepsilon } \; \dd \mathfrak{q}_{k,m}
<
\bar{V} + \varepsilon, \ \ \ \forall (v, \{ r_m\}, \{ \mathfrak{q}_{k,m} \}) \in \bar{H}.
$$
Given the arbitrariness of $\varepsilon$ and restricting attention to $H \subseteq \bar{H}$, one concludes that
\begin{equation} \label{eq:inf_sup<=Vbar_1}
\inf_{\substack{\{\lambda_m\} \\ \{\mu^+_{k,m}\}}} \sup_{(v, \{ r_m\}, \{
\mathfrak{q}_{k,m} \}) \in H} \;
\left\{ v - \sum_{m=0}^M \lambda_m r_m - \sum_{m=0}^{M-1} \sum_{k=0}^m \int_{[0,1]} \mu_{k,m}^+(\alpha) \; \dd \mathfrak{q}_{k,m} \right\}
\leq
\bar{V}.
\end{equation}
On the other hand, recalling the expression of $v$, $r_m$, $\mathfrak{q}_{k,m}$ in the definition of $H$ given in \eqref{definition H}, the left-hand side of \eqref{eq:inf_sup<=Vbar_1} can be rewritten as
$$
\inf_{\substack{\{\lambda_m\} \\ \{\mu^+_{k,m}\}}} \sup_{\{\mathfrak{m}^+_{k,m}\} , \{\mathfrak{p}^+_{k,m}\}} \;
\left\{
\mathfrak{L}_\tau - \sum_{m=0}^{M-1} \sum_{k=0}^m \int_{[0,1]} \mu_{k,m}^+(\alpha) \; \dd \mathfrak{p}^+_{k,m}
\right\},
$$
which further becomes
\begin{equation}  \label{eq:inf_sup_rewritten}
\inf_{\substack{\{\lambda_m\} \\ \{\mu^+_{k,m}\}}} \left\{ \sup_{\{\mathfrak{m}^+_{k,m}\}} \; \mathfrak{L}_\tau + \sup_{\{\mathfrak{p}^+_{k,m}\}} \left\{ - \sum_{m=0}^{M-1} \sum_{k=0}^m \int_{[0,1]} \mu_{k,m}^+(\alpha) \; \dd \mathfrak{p}^+_{k,m} \right\} \right\}
= \inf_{\substack{\{\lambda_m\} \\ \{\mu^+_{k,m}\}}} \sup_{\{\mathfrak{m}^+_{k,m}\}} \mathfrak{L}_\tau,
\end{equation}
where in the last equality the second term has been suppressed because $\sup_{\{\mathfrak{p}^+_{k,m}\}}$ is taken over non-positive quantities and
$\mathfrak{p}^+_{k,m} = 0$ is admissible. Altogether, \eqref{eq:inf_sup<=Vbar_1} and \eqref{eq:inf_sup_rewritten} give the relation
\begin{equation}
\label{infsupL<V}
\inf_{\substack{\{\lambda_m\} \\ \{\mu^+_{k,m}\}}} \sup_{\{\mathfrak{m}^+_{k,m}\}}
\; \mathfrak{L}_\tau \leq \bar{V}.
\end{equation}
To prove \eqref{tbp-Hahn-Banach}, we show that strict inequality in \eqref{infsupL<V} cannot hold. Indeed, in the opposite, there would exist a linear continuous functional of the form \eqref{functionalHB} that separates $H$ from $\bar{P} := (\bar{V}, \{ r_m = 0 \}, \{ \mathfrak{q}_{k,m} = 0 \})$. If we now consider the open set $A$ obtained as counter-image of the reals greater than the value taken by this functional at
$\bar{P}$ minus a small enough margin, then $A$ contains $\bar{P}$, while
$A$ leaves out all $H$, contradicting the fact that $\bar{P}$ is a contact point of $H$.\footnote{\label{footnote-contact point}$\bar{P}$ is a contact point of $H$ because $\bar{V}$ is defined via a $\sup$ operation over contact points and, therefore, any neighborhood of $\bar{P}$ is also a neighborhood of a contact point $(v,\{r_m = 0 \},\big\{ \mathfrak{q}_{k,m} = 0 \big\})$ with $v$ close enough to $\bar{V}$, so that the neighborhood must contain a point of $H$.}
\end{itemize}

We now show that $V = \bar{V}$, so closing the proof of \eqref{eq:supinf=infsup_tau}. We start by constructing a sequence of neighborhoods of $\bar{P} = (\bar{V}, \{ r_m = 0 \}, \{ \mathfrak{q}_{k,m} = 0 \})$ that exhibit asymptotic properties of interest. Consider a countable set of continuous functions $g_1, g_2, \ldots$ dense in $\textsf{C}^0[0,1]$ with respect to the $\sup$ norm (e.g., polynomials with rational coefficients, see
Theorem 7.26 in \citealp{Rudin_PMA}). For $i=1,2,\ldots$, the neighborhoods of $\bar{P}$ are defined as follows:
\begin{eqnarray*}
O_i & := & \Big\{(v,\{r_m\},\{ \mathfrak{q}_{k,m} \}) \mbox{ with } |v - \bar{V}| < 1/i; \ |r_m| < 1/i, m=0,1,\ldots,M; \mbox{ and } \\
& &  \max_{j=1,\ldots,i}\left| \int_{[0,1]} g_j(\alpha) \; \dd \mathfrak{q}_{k,m} \right| < 1/i, m=0,1,\ldots,M-1 \mbox{ and } k=0,\ldots,m \Big\}.
\end{eqnarray*}
Further, for any $m=0,1,\ldots,M$ and $k=0,1,\ldots,m$ consider sequences $\mathfrak{m}^{+,i}_{k,m}$ and $\mathfrak{p}^{+,i}_{k,m}$ indexed in $i$ such that, for each $i=1,2,\ldots$, the pair $(\{\mathfrak{m}^{+,i}_{k,m} \},\{\mathfrak{p}^{+,i}_{k,m}\})$  maps into a point of $H$ that is also in $O_i$
(such sequences certainly exist since $\bar{P}$ is a contact point of $H$, see Footnote \ref{footnote-contact point}). For these sequences we have
\begin{align}
& \lim_{i \to \infty} \sum_{k=0}^N {N \choose k} \int_{[0,1]} \varphi_{k,\tau}(\alpha) \; \dd \mathfrak{m}^{+,i}_{k,N} = \bar{V}; \label{eq:cost=V_i=infty} \\
& \lim_{i \to \infty} \left[ \sum_{k=0}^m {m \choose k} \int_{[0,1]} \dd \mathfrak{m}^{+,i}_{k,m} - 1 \right] = 0, \ \ \  m=0,1,\ldots,M; \label{eq:vinc_int=0_i=infty} \\
& \lim_{i \to \infty} \int_{[0,1]} g_j(\alpha) \; \dd [\mathfrak{m}^{+,i}_{k,m+1} - (1-\alpha) \; \mathfrak{m}^{+,i}_{k,m} + \mathfrak{p}^{+,i}_{k,m}] = 0, \nonumber \\
& \quad \forall g_j(\alpha), \; j=1,2,\ldots, \ \ \ m=0,1,\ldots,M-1,
\ \ \ k=0,\ldots,m. \label{eq:weak_conv=0_i=infty}
\end{align}
In view of \eqref{eq:vinc_int=0_i=infty}, for a given $m$ and $k$, measures $\mathfrak{m}^{+,i}_{k,m}$ are uniformly bounded in $i$ (i.e., $\mathfrak{m}^{+,i}_{k,m} ( [0,1] ) \leq C$, $\forall i$, for some positive constant $C < +\infty$). Since measures $\mathfrak{m}^{+,i}_{k,m}$ are supported in $[0,1]$, by Prokhorov's theorem (see \citealp[Theorem 1, Section 2, Chapter III]{Shiryaev}),\footnote{In fact, a straightforward extended version of Prokhorov's theorem for positive and uniformly bounded  measures.} we then conclude that there exists a sub-sequence of indexes $i_h$ such that $\mathfrak{m}^{+,i_h}_{k,m}$ has weak limit $\bar{\mathfrak{m}}^+_{k,m} \in \mathcal{M}^+$. By repeating the same reasoning in a nested manner, we can further find a sub-sequence of the indexes $i_h$ such that weak convergence holds for a new choice of $m$ and $k$. Proceeding the same way for all choices of $m$ and $k$, we conclude that there exists a sub-sequence of indexes (which, with a little abuse of notation, we still indicate as $i_h$) for which
\begin{equation} \label{eq:Prokhorov_weak_conv}
\int_{[0,1]} f(\alpha) \; \dd \bar{\mathfrak{m}}^+_{k,m} = \lim_{i_h \to \infty} \int_{[0,1]} f(\alpha) \; \dd \mathfrak{m}^{+,i_h}_{k,m}, \quad
\forall \; m,k,
\end{equation}
holds for any continuous function $f \in \textsf{C}^0[0,1]$. \\
\\
Since $\varphi_{k,\tau}$, as well as the constant function equal to $1$, are continuous, \eqref{eq:Prokhorov_weak_conv} together with \eqref{eq:cost=V_i=infty} and \eqref{eq:vinc_int=0_i=infty} yield
\begin{equation} \label{eq:m_bar_attains_cost}
\sum_{k=0}^N {N \choose k} \int_{[0,1]} \varphi_{k,\tau}(\alpha) \; \dd
\bar{\mathfrak{m}}^+_{k,N} = \bar{V}
\end{equation}
and
\begin{equation} \label{eq:m_bar_attains_=_constr}
\sum_{k=0}^m {m \choose k} \int_{[0,1]} \dd \bar{\mathfrak{m}}^+_{k,m} - 1 = 0, \quad m=0,1,\ldots,M.
\end{equation}
Turn now to consider \eqref{eq:weak_conv=0_i=infty}, from which we have
\begin{equation} \label{eq:2_limits}
\lim_{i_h \to \infty} \int_{[0,1]} g_j(\alpha) \; \dd [\mathfrak{m}^{+,i_h}_{k,m+1} - (1-\alpha) \; \mathfrak{m}^{+,i_h}_{k,m}] = - \lim_{i_h \to \infty} \int_{[0,1]} g_j(\alpha) \; \dd \mathfrak{p}^{+,i_h}_{k,m},
\end{equation}
where the limit in \eqref{eq:weak_conv=0_i=infty} restricted to the sub-sequence $i_h$ can be broken up in the two limits in \eqref{eq:2_limits} because the left-hand side of \eqref{eq:2_limits} exists due to the weak convergence of measures $\mathfrak{m}^{+,i_h}_{k,m}$ (note that $g_j(\alpha)$ and $g_j(\alpha)(1-\alpha)$ are continuous functions). For the functions $g_j$ that are non-negative (i.e., $g_j(\alpha) \geq 0$, $\forall
\alpha$), which we henceforth write as $g^+_j$ to help interpretation, \eqref{eq:2_limits} gives
\begin{equation} \label{eq:nonpositivity_of_integral_seq}
\lim_{i_h \to \infty} \int_{[0,1]} g^+_j(\alpha) \; \dd [\mathfrak{m}^{+,i_h}_{k,m+1} - (1-\alpha) \; \mathfrak{m}^{+,i_h}_{k,m}] \leq 0.
\end{equation}
Taking now any non-negative function $f^+$ in $\textsf{C}^0[0,1]$ and noting that $f^+$ can be arbitrarily approximated in the $\sup$ norm by a function $g^+_j$,\footnote{Note that function $f^+(\alpha)$ can be zero for
some $\alpha$, so that an approximant, however close, might as well take negative values, against the requirement that the approximant is a non-negative $g^+_j$. Nonetheless, any $\varepsilon$-close approximant of $f^+(\alpha)+\varepsilon$ is non-negative and it is also a $2\varepsilon$-close approximant of $f^+(\alpha)$.} weak convergence of $\mathfrak{m}^{+,i_h}_{k,m}$ to $\bar{\mathfrak{m}}^+_{k,m}$ used in \eqref{eq:nonpositivity_of_integral_seq} yields
$$
\int_{[0,1]} f^+(\alpha) \; \dd [\bar{\mathfrak{m}}^+_{k,m+1} - (1-\alpha) \; \bar{\mathfrak{m}}^+_{k,m}] \leq 0,
$$
from which $\bar{\mathfrak{m}}^+_{k,m+1} - (1-\alpha) \; \bar{\mathfrak{m}}^+_{k,m}$ is a negative measure (recall Footnote \ref{key_footnote}). \\
\\
If we now choose $\bar{\mathfrak{p}}^+_{k,m} = -[\bar{\mathfrak{m}}^+_{k,m+1} - (1-\alpha) \; \bar{\mathfrak{m}}^+_{k,m}]$ (which is in $\mathcal{M}^+$), then in the light of \eqref{eq:m_bar_attains_cost} and \eqref{eq:m_bar_attains_=_constr} one sees that $(\{\bar{\mathfrak{m}}^+_{k,m} \},\{\bar{\mathfrak{p}}^+_{k,m}\})$ maps into the point $(\bar{V},\{r_m =
0 \},\big\{ \mathfrak{q}_{k,m} = 0 \big\})$, which proves that this point is in $H$. Hence, it holds that $V = \bar{V}$ and equation \eqref{eq:supinf=infsup_tau} remains proven.
\end{itemize}
We next show that
\begin{equation} \label{eq:limsupinfLtau=supinfL}
\lim_{\tau \to 0} \; \sup_{\{\mathfrak{m}^+_{k,m}\}} \inf_{\substack{\{\lambda_m\} \\ \{\mu^+_{k,m}\}}} \mathfrak{L}_\tau = \sup_{\{\mathfrak{m}^+_{k,m}\}} \inf_{\substack{\{\lambda_m\} \\ \{\mu^+_{k,m}\}}} \mathfrak{L},
\end{equation}
which is the only relation in \eqref{fundamnetal relations} that is still unproven, so concluding the proof.
\begin{itemize}
\item[] Notice that, in both sides of \eqref{eq:limsupinfLtau=supinfL},
the $\inf$ operator sends the value to $-\infty$ whenever the constraints
in \eqref{eq:primal_M_equality_constr} or \eqref{eq:primal_M_inequality_constr} are not satisfied by $\{ \mathfrak{m}^+_{k,m}\}$: hence, \eqref{eq:primal_M_equality_constr} and \eqref{eq:primal_M_inequality_constr} must
be satisfied and are always assumed from now on. Under \eqref{eq:primal_M_equality_constr} and \eqref{eq:primal_M_inequality_constr}, $\inf$ is attained for $\lambda_m = 0$ and $\mu^+_{k,m} = 0$ for all $m$ and $k$,
and \eqref{eq:limsupinfLtau=supinfL} is therefore rewritten as
\begin{equation} \label{lim_sup_cost_tau=lim_sup_cost}
\lim_{\tau \to 0} \; \sup_{\{\mathfrak{m}^+_{k,m}\}}
\sum_{k=0}^N {N \choose k} \int_{[0,1]} \varphi_{k,\tau}(\alpha) \; \dd
\mathfrak{m}^+_{k,N} = \sup_{\{\mathfrak{m}^+_{k,m}\}}
\sum_{k=0}^N {N \choose k} \int_{[0,1]} \One{\alpha \in (\eps_k,1]} \; \dd \mathfrak{m}^+_{k,N}.
\end{equation}
To show the validity of \eqref{lim_sup_cost_tau=lim_sup_cost}, we discretize $\tau$ into $\tau_i$, $i=1,2,\ldots$, $\tau_i \to 0$, and consider a sequence $\{ \breve{\mathfrak{m}}^{+,i}_{k,m} \}$, $i=1,2,\ldots$ (where measures $\breve{\mathfrak{m}}^{+,i}_{k,m}$ satisfy \eqref{eq:primal_M_equality_constr} and \eqref{eq:primal_M_inequality_constr} for any $i$), such that
$$
\lim_{i\to \infty} \sum_{k=0}^N {N \choose k} \int_{[0,1]} \varphi_{k,\tau_i}(\alpha) \; \dd \breve{\mathfrak{m}}^{+,i}_{k,N}
$$
equals the left-hand side of \eqref{lim_sup_cost_tau=lim_sup_cost} (for this to hold, $\breve{\mathfrak{m}}^{+,i}_{k,m}$ must achieve a progressively closer and closer approximation of $\sup_{\{\mathfrak{m}^+_{k,m}\}}$ in the left-hand side of \eqref{lim_sup_cost_tau=lim_sup_cost} as $i$ increases); then, we construct from $\{ \breve{\mathfrak{m}}^{+,i}_{k,m} \}$ a new
sequence $\{ \tilde{\mathfrak{m}}^{+,i}_{k,m} \}$, $i=1,2,\ldots$ (still satisfying \eqref{eq:primal_M_equality_constr} and \eqref{eq:primal_M_inequality_constr}), such that
\begin{equation} \label{eq:lim_cost_phi_m<=lim_cost_One_m_tilde}
\lim_{i\to \infty} \sum_{k=0}^N {N \choose k} \int_{[0,1]} \varphi_{k,\tau_i}(\alpha) \; \dd \breve{\mathfrak{m}}^{+,i}_{k,N} \leq \lim_{i\to \infty} \sum_{k=0}^N {N \choose k} \int_{[0,1]} \One{\alpha \in (\eps_k,1]} \; \dd \tilde{\mathfrak{m}}^{+,i}_{k,N},
\end{equation}
which shows that the left-hand side of \eqref{lim_sup_cost_tau=lim_sup_cost}
is upper-bounded by a value that is no bigger than the right-hand side of
\eqref{lim_sup_cost_tau=lim_sup_cost}. Since, on the other hand, the left-hand side of \eqref{lim_sup_cost_tau=lim_sup_cost} cannot be smaller than the right-hand side of \eqref{lim_sup_cost_tau=lim_sup_cost} because $\varphi_{k,\tau}(\alpha) \geq \One{\alpha \in (\eps_k,1]}$, \eqref{lim_sup_cost_tau=lim_sup_cost} remains proven. \\
\\
The construction of $\{ \tilde{\mathfrak{m}}^{+,i}_{k,m} \}$ is in three steps:
\begin{itemize}
\item[Step 1.] [construction of $\{ \check{\mathfrak{m}}^{+,i}_{k,m} \}$] For all $k \leq N$ for which $\eps_k \neq 1$ and for all $m$, move the probabilistic mass of $\breve{\mathfrak{m}}^{+,i}_{k,m}$ contained in the interval $(\eps_k-\tau_i,\eps_k]$ into a concentrated mass in point $\eps_k+\tau_i$; let $\check{\mathfrak{m}}^{+,i}_{k,m}$ be the corresponding measures.
\item[Step 2.] [construction of $\{ \hat{\mathfrak{m}}^{+,i}_{k,m} \}$]
The mass shift in Step 1 can lead to measures $\check{\mathfrak{m}}^{+,i}_{k,m}$ that violate condition \eqref{eq:primal_M_inequality_constr} in $\eps_k+\tau_i$; the new measures $\hat{\mathfrak{m}}^{+,i}_{k,m}$ restore
the validity of this condition. For all $k > N$ and all $k \leq N$ for which $\eps_k = 1$ (so that no mass shift has been performed in Step 1), let $\hat{\mathfrak{m}}^{+,i}_{k,m} = \check{\mathfrak{m}}^{+,i}_{k,m}$, for all $m = k, \ldots, M$. For all other $k$'s, let $\hat{\mathfrak{m}}^{+,i}_{k,k} = \check{\mathfrak{m}}^{+,i}_{k,k}$; then, verify sequentially for $m = k, \ldots, M-1$ whether the condition
	$$
	\check{\mathfrak{m}}^{+,i}_{k,m+1}(\{\eps_k+\tau_i\}) - (1-(\eps_k+\tau_i)) \; \hat{\mathfrak{m}}^{+,i}_{k,m}(\{\eps_k+\tau_i\}) \leq 0
	$$
is satisfied; if yes, let $\hat{\mathfrak{m}}^{+,i}_{k,m+1} = \check{\mathfrak{m}}^{+,i}_{k,m+1}$, otherwise trim $\check{\mathfrak{m}}^{+,i}_{k,m+1}(\{\eps_k+\tau_i\})$ to the value $(1-(\eps_k+\tau_i)) \; \hat{\mathfrak{m}}^{+,i}_{k,m}(\{\eps_k+\tau_i\})$ and define $\hat{\mathfrak{m}}^{+,i}_{k,m+1}$ as the trimmed version of $\check{\mathfrak{m}}^{+,i}_{k,m+1}$.
\item[Step 3.] [construction of $\{ \tilde{\mathfrak{m}}^{+,i}_{k,m} \}$] The trimming operation in Step 2 may have unbalanced some equalities in \eqref{eq:primal_M_equality_constr}, i.e., it may be that
	$$
	\sum_{k=0}^m {m \choose k} \int_{[0,1]} \dd \hat{\mathfrak{m}}^{+,i}_{k,m} < 1
	$$
for some $m$. If so, re-gain balance by adding to measure $\hat{\mathfrak{m}}^{+,i}_{m,m}$ a suitable probabilistic mass (e.g., concentrated in $\alpha=1$), while leaving other measures $\hat{\mathfrak{m}}^{+,i}_{k,m}$, $k \neq m$, unaltered. The so-obtained measures are $\tilde{\mathfrak{m}}^{+,i}_{k,m}$. Note that this operation preserves the validity of condition $\tilde{\mathfrak{m}}^{+,i}_{m,m+1} - (1-\alpha) \; \tilde{\mathfrak{m}}^{+,i}_{m,m} \in \mathcal{M}^-$, so that $\{ \tilde{\mathfrak{m}}^{+,i}_{k,m} \}$ satisfies \eqref{eq:primal_M_inequality_constr} besides \eqref{eq:primal_M_equality_constr}.
\end{itemize}
Since $\varphi_{k,\tau_i}(\alpha)$ is non-decreasing in $\alpha$, the mass shift in Step 1 can only increase $\sum_{k=0}^N {N \choose k} \int_{[0,1]} \varphi_{k,\tau_i}(\alpha) \; \dd \breve{\mathfrak{m}}^{+,i}_{k,N}$; moreover, any trimming and re-balancing in Steps 2 and 3 involve vanishing masses as $\tau_i \to 0$. Therefore,
\begin{equation} \label{eq:lim_cost_phi_m<=lim_cost_phi_m_tilde}
\lim_{i\to \infty} \sum_{k=0}^N {N \choose k} \int_{[0,1]} \varphi_{k,\tau_i}(\alpha) \; \dd \breve{\mathfrak{m}}^{+,i}_{k,N} \leq \lim_{i\to \infty} \sum_{k=0}^N {N \choose k} \int_{[0,1]} \varphi_{k,\tau_i}(\alpha) \; \dd \tilde{\mathfrak{m}}^{+,i}_{k,N}.
\end{equation}
On the other hand, by construction, $\varphi_{k,\tau_i}(\alpha) = \One{\alpha \in (\eps_k,1]}$ if $\eps_k = 1$, while, for $\eps_k \neq 1$, $\varphi_{k,\tau_i}(\alpha) \neq \One{\alpha \in (\eps_k,1]}$ only occurs on the interval $(\eps_k-\tau_i,\eps_k]$ where $\tilde{\mathfrak{m}}^{+,i}_{k,N}$ is null. Hence,
$$
\sum_{k=0}^N {N \choose k} \int_{[0,1]} \varphi_{k,\tau_i}(\alpha) \; \dd \tilde{\mathfrak{m}}^{+,i}_{k,N} = \sum_{k=0}^N {N \choose k} \int_{[0,1]} \One{\alpha \in (\eps_k,1]} \; \dd \tilde{\mathfrak{m}}^{+,i}_{k,N},
$$
which, substituted in \eqref{eq:lim_cost_phi_m<=lim_cost_phi_m_tilde}, gives \eqref{eq:lim_cost_phi_m<=lim_cost_One_m_tilde}.
\end{itemize}
This concludes the proof. \qed

\vskip 0.2in
\bibliography{compression_theory_new} 

\end{document}